\documentclass{article}
\DeclareMathAlphabet{\pazocal}{OMS}{zplm}{m}{n}
\usepackage[T1]{fontenc}
\usepackage{float}
\usepackage{microtype} 
\usepackage[inline]{enumitem}
\usepackage{lmodern}
\usepackage{verbatim}
\usepackage{multicol}
\usepackage[margin=1in]{geometry}
\usepackage{mathtools}
\usepackage{amssymb,amsmath,amsfonts,amsthm,dsfont}
\usepackage{bm, bbm}
\usepackage[linesnumbered, vlined, ruled]{algorithm2e}
\SetAlgorithmName{Algorithm}{Algo}{List of Algos}

\usepackage[numbers]{natbib}
\usepackage[dvipsnames]{xcolor}   
\usepackage{hyperref}
\hypersetup{
    colorlinks=true,
    citecolor=ForestGreen,
    filecolor=black,
    linkcolor=VioletRed,
    urlcolor=blue
}
\usepackage[capitalise,nameinlink,noabbrev]{cleveref}

\usepackage{thmtools} 
\usepackage{thm-restate}

\usepackage{graphicx}
\usepackage{caption}
\usepackage{subfigure}
\usepackage{algorithmic}

\newtheorem{theorem}{Theorem}[section]\numberwithin{equation}{section}

\newtheorem{corollary}[theorem]{Corollary}
\newtheorem{definition}[theorem]{Definition}

\newtheorem{fact}[theorem]{Fact}

\newtheorem{lemma}[theorem]{Lemma}
\newtheorem{remark}[theorem]{Remark}

\newtheorem{problem}[theorem]{Problem}

\usepackage{tikz}
\usetikzlibrary{calc, graphs, graphs.standard, shapes, arrows, arrows.meta, positioning, decorations.pathreplacing, decorations.markings, decorations.pathmorphing, fit, matrix, patterns, shapes.misc, tikzmark}
\usepackage{psfrag,epsf,placeins, setspace, comment}

\tikzset{
	>=stealth',
	true/.style={
		rectangle,
		draw=black, very thick,
		text width=6.5em,
		minimum height=2em,
		text centered,
		fill=gray, opacity = 0.5},
	punkt/.style={
		rectangle,
		rounded corners,
		draw=black, very thick,
		text width=6.5em,
		minimum height=2em,
		text centered},
	est/.style={
		circle,
		draw=black, very thick,
		text centered},
	shade/.style={
		circle,
		draw=black, very thick, fill=gray!50,
		text centered},
	weight/.style={
		circle,
		draw=black, very thick,
		text width=6.5em,
		minimum height=2em,
		text centered},
	pil/.style={
		->,
		thick,
		shorten <=2pt,
		shorten >=2pt,},
	double/.style={
		<->,
		thick,
		shorten <=2pt,
		shorten >=2pt,},
	dash/.style={
		dashed,
		thick,
		shorten <=2pt,
		shorten >=2pt,},
	dashdouble/.style={
		<->,
		dashed,
		thick,
		shorten <=2pt,
		shorten >=2pt}}

\usepackage{xcolor}

\usepackage{pifont}

\newcommand{\eps}{\varepsilon}
\newcommand{\wh}{\widehat}

\newcommand{\dkl}{D_{\mathrm{KL}}}

\newcommand{\R}{\mathbb{R}}
\newcommand{\N}{\mathcal{N}}
\newcommand{\abs}[1]{|#1|}

\makeatletter
\newcommand*{\centernot}{%
  \mathpalette\@centernot
}
\def\@centernot#1#2{%
  \mathrel{%
    \rlap{%
      \settowidth\dimen@{$\m@th#1{#2}$}%
      \kern.5\dimen@
      \settowidth\dimen@{$\m@th#1=$}%
      \kern-.5\dimen@
      $\m@th#1\not$%
    }%
    {#2}%
  }%
}
\makeatother

\newcommand{\indep}{\perp\mkern-9.5mu\perp}

\newcommand{\mathdash}{\relbar\mkern-9mu\relbar}
\DeclareMathOperator{\cov}{cov}

\DeclareMathOperator{\poly}{poly}
\DeclareMathOperator*{\E}{\mathbb{E}}

\DeclareMathOperator*{\argmin}{arg\,min}

\newcommand{\unif}{\pazocal{U}} 
\newcommand{\ignore}[1]{}

\usepackage{todonotes}

\DeclarePairedDelimiterX{\infdivx}[2]{(}{)}{%
  #1\;\delimsize\|\;#2%
}
\newcommand{\kl}{D_{\mathrm{KL}}\infdivx}

\allowdisplaybreaks

\usepackage[symbol]{footmisc}

\DeclareMathOperator{\pa}{{pa}}      
\DeclareMathOperator{\ch}{{ch}}      
\DeclareMathOperator{\an}{{an}}   
\DeclareMathOperator{\de}{{de}}   
\DeclareMathOperator{\nd}{{nd}}   
\DeclareMathOperator{\var}{ var}       

\newcommand{\vertiii}[1]{{\left\vert\kern-0.25ex\left\vert\kern-0.25ex\left\vert #1 
		\right\vert\kern-0.25ex\right\vert\kern-0.25ex\right\vert}}

\newcommand\numberthis{\addtocounter{equation}{1}\tag{\theequation}}

\newcommand{\given}{\,|\,}
\newcommand{\T}{^\top}
\newcommand{\prob}{\mathrm{Pr}}

\DeclareMathOperator{\sk}{sk}
\newcommand{\mec}{\overline{T}}
\newcommand{\drtr}{\mathcal{T}}
\newcommand{\pltr}{\widetilde{\mathcal{T}}}

\newcommand{\defeq}{\vcentcolon=}

\title{Optimal estimation of Gaussian (poly)trees}
\author{
Yuhao Wang\\
National University of Singapore\\
\texttt{yuhaowang@u.nus.edu}
\and
Ming Gao\\
University of Chicago\\
\texttt{minggao@chicagobooth.edu}
\and
Wai Ming Tai\\
Nanyang Technological University\\
\texttt{waiming.tai@ntu.edu.sg}
\and
Bryon Aragam\\
University of Chicago\\
\texttt{bryon@chicagobooth.edu}
\and
Arnab Bhattacharyya\\
National University of Singapore\\
\texttt{arnabb@nus.edu.sg}
}
\date{}
\begin{document}
\maketitle
\begin{abstract}
  We develop optimal algorithms for learning undirected Gaussian trees and directed Gaussian polytrees from data. We consider both problems of distribution learning (i.e. in KL distance) and structure learning (i.e. exact recovery).
  The first approach is based on the Chow-Liu algorithm, and learns an optimal tree-structured distribution efficiently.
  The second approach is a modification of the PC algorithm for polytrees that uses partial correlation as a conditional independence tester for constraint-based structure learning. We derive explicit finite-sample guarantees for both approaches, and show that both approaches are optimal by deriving matching lower bounds. Additionally, we conduct numerical experiments\footnote{\href{https://github.com/YohannaWANG/Polytree}{https://github.com/YohannaWANG/Polytree}} to compare the performance of various algorithms, providing further insights and empirical evidence.
\end{abstract}

\section{Introduction}
Graphical models are a classical statistical tool for efficiently modeling data with rich, combinatorial structure. Directed acyclic graphs (DAGs) are widely used to capture causal relationships among complex systems. Probabilistic graphical models defined on DAGs, known as Bayesian networks \citep{pearl2000models}, have found broad applications in various disciplines, from biology \citep{markowetz2007inferring, zhang2013integrated, altay2010inferring}, social science \citep{gupta2008linking}, knowledge representation \citep{van2008handbook}, data mining \citep{heckerman1997bayesian}, recommendation systems \citep{hsu2012design}, legal decision making \citep{thagard2004causal}, and more.
When this structure is known in advance, it is straightforward to exploit this structure for inference tasks, among other things \citep{wainwright2008graphical}. When this structure is unknown, it is must first be learned from data, which is the difficult problem of \emph{structure learning} in graphical models.
First, observational data only reveal the Markov equivalence class, captured by a completed partially directed acyclic graph (CPDAG) \citep{andersson1997characterization}. Classical approaches to learning a CPDAG from data include the PC algorithm \citep{spirtes1991algorithm,kalisch2007estimating} and GES \citep{chickering2002optimal,nandy2018high}.
Moreover, it is also known that the general problem of learning DAGs from observational data is an NP-complete problem \citep{chickering1996learning, chickering2004large, chen2019causal}, although a few polynomial-time algorithms have been proposed for special cases \citep{ghoshal2017learning, chen2019causal, park2020identifiability, gao2020polynomial}.

An important unresolved problem in this direction is to characterize the sample complexity of structure learning, or the minimum number of samples required to learn the graph from data. The past decade has produced a broad literature on this problem, mainly focused on \emph{undirected} graphical models (i.e. Markov random fields
\citep{wainwright2019high,wang2010information,santhanam2012information}). By comparison, much less is known about DAGs.
In this paper, we study in detail the simplest unresolved DAG model, namely, directed Gaussian trees. Perhaps surprisingly, despite its  simplicity, and unlike in the undirected case, the optimal sample complexity of learning directed Gaussian trees has remained an open problem. Suppose we are given sample access to a Gaussian distribution $P=\N(0,\Sigma)$, where the goal is to learn a DAG $G$ that represents $P$. While we defer formal definitions to \cref{sec:prelim},
we can broadly summarize three different problems to be addressed here at the outset: 
\begin{enumerate}
    \item (Non-realizable setting) When $P$ is an arbitrary Gaussian (i.e. not representable by any tree), how many samples are required to learn a tree-structured distribution $Q$ that is optimally close to $P$?
    \item (Realizable setting) When $P$ itself is tree-structured, how samples are required to learn a tree-structured distribution $Q$ that is optimally close to $P$?
    \item (Faithful setting) When $P$ is faithful to some tree $T$, how samples are required to learn $T$ itself (i.e. the tree structure) up to Markov equivalence?
\end{enumerate}
It is well-known that each of these problems is solvable---in principle---under different assumptions. For example, the celebrated Chow-Liu algorithm solves the first two problems, however, whether or not this can be improved with a more efficient algorithm is unknown. The same goes for the third setting: The famous PC and GES algorithms can find a faithful DAG (even without the tree assumption), however, their optimality remains unresolved.
One of our main contributions is to study all three problems in a single unified setting, allowing for apples-to-apples comparisons of the assumptions required, and the resulting (optimal) sample complexity for each. 

Although faithfulness can be a strong assumption in practice, we emphasize that to the best of our knowledge, no optimality results under this assumption are known. Thus, our analysis presents a possible first foray in this direction. Previous work has shown that faithfulness is notoriously challenging to analyze \citep[e.g.][]{uhler2013,gao2023optimal}.

\subsection{Our Contributions}
We are given $n$ i.i.d. samples $X=(X^{(1)},\ldots,X^{(n)})\in\R^{n\times d}$ from an unknown Gaussian $P$. 
We consider two distinct but canonical problems: \emph{Distribution learning} and \emph{structure learning}. The difference between these two problems lies in the error metric: In distribution learning, we seek to learn $P$ in KL-divergence, with no respect for underlying structure (i.e. there may be no structure at all), whereas in structure learning, we assume \emph{a priori} the existence of a tree $T$ and seek to learn $T$ exactly, with no respect for the distribution $P$. Structure learning is known to require restrictive assumptions, and thus part of our effort is to illustrate how different assumptions lead to different conclusions and sample complexities.
With this in mind, our results consider three progressively stronger assumptions on $P$: Non-realizable, realizable, and faithful.

Below, we outline our main contributions at a high-level, while deferring precise statements and problem formulations to \cref{sec:chow-liu tree} and \cref{sec:faithfultree}.

\paragraph{Non-realizable Setting}
Without making additional assumptions on $P$, we show that\footnote{$\widetilde{\Theta}$ is used to ignore potential log factors.}
\begin{align}
\label{eq:complexity:nrlz}
    n = \widetilde{\Theta}\Big( \frac{d^2}{\eps^2} \Big)
\end{align}
are necessary and sufficient to learn (with probability at least $2/3$) a tree-structured distribution that is $\eps$-close to the closest tree-structured distribution for $P$.

\paragraph{Realizable Setting}
When $P$ itself is Markov to a tree $T$ (i.e. it is \emph{tree-structured}), then 
\begin{align}
\label{eq:complexity:rlz}
    n = \widetilde{\Theta}\Big( \frac{d}{\eps}  \Big)
\end{align}
are necessary and sufficient to learn (with probability at least $2/3$) a tree-structured distribution that is $\eps$-close to $P$ itself.

\paragraph{Faithful Polytrees}
Switching our goal from learning the closest tree-structured distribution to structure learning, we additionally assume that $P$ is faithful to some \emph{polytree} $T$.
We show that the optimal sample complexity of learning $\mec$, the CPDAG of $T$, is 
\begin{align}
\label{eq:complexity:sl}
    n = \Theta\bigg(\frac{\log d}{c^2}\bigg),
\end{align}
where $c$ is a faithfulness parameter defined in \eqref{defn:strongtreefaith}.

Clearly, and unsurprisingly, realizable distribution learning is easier than the non-realizable case. A more interesting question is how to compare these to structure learning. In \cref{sec:discussion}, we conclude with a discussion and comparison of these two cases, with some intriguing directions for future work.

\subsection{Other Related Work}

\paragraph{Learning Bayesian Networks}

Structure learning of Bayesian networks
has been extensively studied, and the reader may consult one of several overviews for more details and background \citep{spirtes2000causation, pearl2000models, koller2009probabilistic, murphy2012machine,peters2017elements,maathuis2018handbook,squires2022causal}. 
Classical approaches assume faithfulness, a condition that permits learning of the Markov equivalence class, such as constraint-based methods \citep{spirtes1991algorithm,friedman2013learning} and score-based approaches \citep{chickering2002optimal,nandy2018high}.
A different strand of research has explored a range of alternative distributional assumptions that allow for effective learning such as non-gaussianity \citep{shimizu2006linear,shimizu2014lingam, wang2020high}, non-linearity \citep{hoyer2008nonlinear,zhang2009} or equal error variances \citep{peters2014identifiability,ghoshal2017learning,ghoshal2018learning,chen2019causal,gao2020polynomial}.

When it comes to the tree-structured graphical model of a distribution, the classical Chow-Liu algorithm \citep{ChowL68} can recover the skeleton of a non-degenerate polytree in the equivalence class. Furthermore, 
\citep{chow1973consistency} demonstrate that as the number of samples approaches infinity, the Chow-Liu algorithm is \emph{consistent}. One of the first papers to consider the problem of learning polytrees was \citep{rebane1987recovery}, after which \citep{dasgupta1999polytree} showed that learning polytrees is NP-hard in general. 
\citep{srebro2003maximum} has shown that the related problem of finding the maximum likelihood graphical model with bounded treewidth is also NP-hard.
Recently Tan et al. \citep{tan2010learning, tan2011learning} investigated the recovery difficulty of trees and forests, while Liu et al. \citep{liu2011forest} adopted a nonparametric approach using kernel density estimates. The Chow-Liu algorithm has also been applied for learning latent locally tree-like graphs~\citep{anandkumar2013learning}.

\paragraph{Sample Complexity of Structure Learning}
Early work to consider the sample complexity problem for Bayesian networks includes \citep{friedman1996, zuk2012}.
More recently, for distribution learning over finite alphabets,
\citep{daskalakis2020tree, daskalakis2021sample} showed that $d$-variable tree-structured Ising models can be learned computationally-efficiently to within total variation distance $\eps$ from an optimal $O(d\log d/\eps^2)$ samples.
Around the same time, \citep{bhattacharyya2021near} derived explicit sample complexity bounds for the Chow-Liu algorithm of $\widetilde{O}(d\eps^{-1})$ for trees on $d$ vertices, and $d^2\eps^{-2}$ samples for a general distribution $P$. 
\citep{choo2023learning} further extend \citep{bhattacharyya2021near} into $d$-polytree when the underlying graph skeleton is known. 

The literature on structure learning is comparatively deeper; however, it has traditionally forgone concerns about optimality and lower bounds. As this is our main focus, we focus here on prior work on optimal algorithms.
\citep{ghoshal2017information} first established lower bounds for a range of DAG models, after which \citep{gao2022optimal} showed that a variant of the algorithm from \citep{chen2019causal} achieves optimal sample complexity of $n \asymp q \log(d/q)$ for equal variance DAGs \citep{peters2014identifiability,loh2014high}, where $q$ is the maximum number of parents and $d$ is the number of nodes.
To the best of our knowledge, optimality results and lower bounds in the faithful setting are missing, one exception is the sub-problem of neighbourhood selection \citep{gao2023optimal}, and one of our main contributions is to partially fill this gap. 
We mention prior work that considers consistency and upper bounds under faithfulness \citep{kalisch2007estimating,nandy2018high,rothenhausler2018causal},
relaxation and improvement on classical methods \citep{chickering2020statistically,marx2021weaker,lam2022greedy},
and recent progress on learning polytrees \citep{gao2021efficient,azadkia2021fast,tramontano2022learning,jakobsen2022structure}.

Learning polytrees is among the easiest tasks in learning DAGs and has received attention in \citep{ChowL68, karger2001learning, rebane2013recovery, nie2014}. 
The crucial advantage of such networks is that they allow for a more efficient solution of the inference task \citep{pearl1988probabilistic, guo2002survey}.
The complexity of polytree learning has been studied in several
works \citep{safaei2013learning, gaspers2015finding, gruttemeier2021parameterized}. 
A recent work in \citep{gao2021efficient} shows that learning polytrees is more manageable than general DAG models, for which they establish clear conditions for the identifiability and learnability of nonparametric polytrees in polynomial time. Some other earlier works such as reconstruction of evolutionary trees can be found in \citep{buneman1971recovery, chang1991reconstruction, chang1996full, daskalakis2011evolutionary}. 
Besides, latent tree model is a class of latent variable models in which the graph may be a forest has received considerable attention \citep{choi2011learning, tan2011learning, song2011kernel, anandkumar2011spectral, parikh2011spectral, mossel2013robust, song2014nonparametric, drton2017marginal}.
Specifically, \citep{anandkumar2011spectral} shows that the structure of multivariate latent tree models can be learned with a sample complexity depends solely on intrinsic spectral properties of the distribution. (also see survey paper \citep{mourad2013survey} for more details).
\citep{anandkumar2012learning} proved that a $\poly(d, r)$ sample and computational requirements serves as a good approximation of a $r$-component mixture of $d$-variate graphical models.

Furthermore, developing a (conditional) independence tester with respect to mutual information with $o(1/\eps^2)$ sample complexity was posed as
an open problem in \citep{canonne2018testing}. In \citep{canonne2018testing}, they have shown that both Ising model \emph{Goodness-of-fit Testing} and Ising model \emph{Independece Tesing} can be solved from $\poly(d, 1/\eps)$ samples in polynomial time. More details related to the \emph{distribution property testing} can be found in \citep{rubinfeld2012taming, canonne2020survey, goldreich2017introduction,bhattacharyya2022property}.

\section{Preliminaries and Tools}\label{sec:prelim}

\paragraph{Preliminary Notions} We employ standard asymptotic notation ${O}(\cdot), \Omega(\cdot), \Theta(\cdot)$; and as usual, $\widetilde{\cdot}$ indicates up to log factors. For example, if $f=\widetilde{\Theta}(g)$ then $f=O(g (\log g)^{c_1})$ and $f=\Omega(g/(\log g)^{c_2})$ for some constants $c_1$ and $c_2$.
We say $f \lesssim g$ and $f\gtrsim g$ if $f\le Cg$ and $f\ge cg$ for some positive constants $C$ and $c$.

\paragraph{Graphical Definitions}
For a directed acyclic graph (DAG) $G=(V,E)$,
for each node $k \in V$, $\pa(k) = \{j:(j,k)\in E\}$ denotes its parent nodes, descendants $\de(k)$ denotes the nodes that can be reached by $k$ and $\nd(k)=V\setminus \de(k)$ denotes the nondescendants.
The skeleton of $G$, $\sk(G)$, is the undirected graph formed by removing directions of all the edges in $G$.
For any $j,\ell,k \in V$, a triple $(j,\ell,k)$ is called unshielded if both $j,k$ are adjacent to $\ell$ but not adjacent to each other, graphically $j-\ell-k$; and is called a $v$-structure if additionally $j,k$ are parents of $\ell$, i.e. $j\rightarrow \ell \leftarrow k$.
The in-degree of $G$ is $\max_k |\pa(k)|$.
A tree is an undirected graph in which any two nodes are connected by exactly one path. 
A directed tree is a directed graph in which, for some root node $u$, and any other node $v$, there is exactly one directed path from $u$ to $v$.
A polytree is a directed graph whose skeleton to be a tree.
Denote the set of directed trees (resp. polytrees) over $d$ nodes to be $\drtr$ (resp. $\pltr$).
Note that a directed tree is a polytree with in-degree equal to one except the root node who has no parent and $\drtr\subseteq \pltr$. 

\paragraph{Gaussian Bayesian Networks}
Given a random vector $X = (X_1, \dots, X_d)$ drawn from a distribution $P$, a DAG $G$ is a Bayesian network for $X$ (or precisely, its joint distribution $P$) if the following factorization holds:
\begin{align}\label{eq:markov}
P(X) = \prod_{k=1}^d P(X_k\given X_{\pa(k)})\,.
\end{align}
Here, we use $X=V=[d]$ interchangeably with some abuse of notation.
From now on, we assume that $P=\mathcal{N}(0,\Sigma)$ throughout.
Since $P$ is Gaussian, we can always express $X$ as the following linear structural equation model (SEM):
\begin{align}\label{eq:sem}
    X_{k} = \beta_k\T X + \eta_{k}\,, \quad \eta_k\sim\N(0,\sigma_k^2),
\end{align}
where $\beta_k\in\mathbb{R}^d$ is supported on $\pa(k)$ and the $\{\eta_k\}_{k=1}^d$ are mutually independent. A Gaussian distribution is said to be $T$-structured for some directed tree $T\in \drtr$ (or simply tree-structured when the specific $T$ is not important in the context) if it satisfies \eqref{eq:markov} with respect to some tree $T$.
For a distribution $P$ and a directed tree $T$, let
\begin{align*}
    P_T \defeq \argmin_{\text{$T$-structured distribution }Q}\kl{P}{Q},
\end{align*}
where $\kl{\cdot}{\cdot}$ denotes the KL-divergence. In this paper, we consider both general Gaussians (non-realizable case) as well as tree-structured distributions (realizable and faithful cases), i.e. \eqref{eq:sem} holds for some directed (poly)tree $T$.

\paragraph{Faithfulness and Markov Equivalence Class}
For the purpose of structure learning, a common assumption is faithfulness, under which the DAG is identified up to its Markov equivalence class (MEC). We assume the reader is familiar with standard graphical concepts such as $d$-separation; 
see \citep{koller2009probabilistic} for more background.
\begin{definition}[Faithfulness]\label{defn:faith}
We say a distribution $P$ is faithful to a DAG $G$ if for any $j,k\in V$ and $S\subseteq V\setminus \{j,k\}$,
\begin{align*}
    X_j\indep X_k\given X_S \Rightarrow j \text{ and } k \text{ are d-separated by } S \,.
\end{align*}
\end{definition}
Equivalently, for any two nodes $j$ and $k$ not $d$-separated by set $S$, faithfulness requires $X_j\not\indep X_k\given X_S$.
The MEC of a DAG $G$ is the set of DAGs that encode the same set of conditional independencies as $G$, which is usually represented by a CPDAG, denoted by $\overline{G}$.
A standard approach to learning a CPDAG under faithfulness is the PC algorithm \citep{spirtes1991algorithm}, which relies on conditional independence testing to recover the skeleton and orient the edges.
While faithfulness can be a strong assumption~\citep{uhler2013}, it is known that weaker assumptions suffice. For example:
\begin{definition}[Restricted faithfulness]\label{defn:resfaith}
We say a distribution $P$ is restricted faithful to a DAG $G$ if
\begin{enumerate}
    \item For any $(j,k)\in E$, $S\subseteq V\setminus \{j,k\}$, $X_j\not\indep X_k\given X_S$;
    \item For any unshielded triple $j-\ell-k$, if this is a v-structure, then $X_j\not\indep X_k \given S$ for any $S\subseteq  V\setminus\{j,k\}$ with $\ell\in S$; if not, then $X_j\not\indep X_k \given X_S$ for any $S\subseteq  V\setminus\{j,k,\ell\}$.
\end{enumerate}
\end{definition}
 
Under general faithfulness, all conditional independence relationships imply $d$-separations in a DAG. In other words, all instances of $d$-connections lead to conditional dependence. On the contrary, restricted faithfulness requires only a subset of $d$-connections to imply conditional dependence.
Conventionally, the first part of \cref{defn:resfaith} is also named \textit{adjacency-faithfulness} and the second part is named \textit{orientation-faithfulness}. With our focus on the setup where the underlying DAG is a polytree, restricted faithfulness can be further relaxed as we will discuss in \cref{sec:faithfultree}.

\section{Learning Tree-structured Gaussians}\label{sec:chow-liu tree}

We begin by studying the sample complexity for learning tree-structured Gaussian distributions. For any $\eps>0$, we would like to devise an algorithm taking samples drawn from a Gaussian $P$ that returns a directed tree $\wh{T}\in\drtr$ and a distribution $P_{\wh{T}}$ that is Markov to $\wh{T}$ such that
\begin{align*}
    \kl{P}{P_{\wh{T}}} \leq \min_{T\in\drtr} \kl{P}{P_{T}} + \eps\,,
\end{align*}
We seek to achieve this goal with a minimal number of samples.
Notably, for any $T\in\drtr$, $\kl{P}{P_{T}}$ can be expressed as 
\begin{align*}
    -\sum_{i=1}^d I(X_i; X_{\pa(i)}) - H(X) + \sum_{i=1}^d H(X_i), \numberthis\label{eq:kl_p_pt}
\end{align*}
where $H$ is the entropy function and $I$ is the mutual information.

\subsection{Distribution Learning Upper Bounds}

The classical Chow-Liu algorithm \citep{ChowL68} builds the maximum weight spanning tree where the weight of the ``potential'' edge between nodes $j$ and $k$ is the estimated mutual information $\wh{I}(X_j, X_k)$ from data. 
Although its return is an undirected graph, we modify the output to be any directed tree whose skeleton matches the undirected graph with light abuse of notation.
This is because any $T\in\drtr$ with the same skeleton will share the same $P_T$, which is the target of distribution learning analyzed in the sequel.

\begin{algorithm}[ht]
\caption{Modified \texttt{Chow-Liu} algorithm}
\label{algo:chow-liu}
\textbf{Input:} $n$ i.i.d. samples $(X^{(i)}_1,\ldots,X^{(i)}_d)$\\
\begin{enumerate}
    \item For each $j=1,\dots,d$:
    \begin{enumerate}
        \item $\wh{\sigma}_j^2 \gets \frac{1}{n}\sum_{i=1}^n (X_j^{(i)})^2$
    \end{enumerate}
    \item For each pair $(j,k), 1\leq j<k\leq d$:
    \begin{enumerate}
        \item $\wh{\rho}_{jk} \gets \frac{1}{n}\sum_{i=1}^n X_j^{(i)}X_k^{(i)}$
    \end{enumerate}
    \item For each pair $(j,k), 1\leq j<k\leq d$:
    \begin{enumerate}
        \item $\wh{I}(X_j;X_k)\gets -\frac{1}{2}\log\big({1-\frac{\wh{\rho}_{jk}^2}{\wh{\sigma}_j^2\wh{\sigma}_k^2}}\big)$ which is same as $\frac{1}{2}\log(1+\frac{\wh{\beta}_{jk}^2\wh{\sigma}_j^2}{\wh{\sigma}_{k\mid j}})$ defined in Section \ref{sec:cmit}
    \end{enumerate}
    \item $G \gets$ the weighted complete undirected graph on $[d]$ whose edge weight for $(j,k)$ is $\wh{I}(X_j;X_k)$
    \item $\wh{S} \gets$ the maximum weighted spanning tree of $G$
    \item $\wh{T} \gets$ any directed tree with skeleton to be $\wh{S}$ 
\end{enumerate}
\textbf{Output:} A directed tree $\wh{T}$
\end{algorithm}

Our first result gives an upper bound on the sample complexity for distribution learning in the non-realizable setting:
\begin{restatable}{theorem}{thmupperboundnonrealizable}
\label{tm:nonrealizable}

Let $P$ be a Gaussian distribution. Given $n$ i.i.d. samples from $P$, for any $\eps, \delta>0$, if $n \gtrsim\frac{d^2}{\eps^2}\log \frac{d}{\delta}$, then $\wh{T}$ returned by \cref{algo:chow-liu} satisfies
\begin{align*}
    \kl{P}{P_{\wh{T}}} \leq \min_{T\in\drtr} \kl{P}{P_{T}} + \eps,
\end{align*}
with probability at least $1-\delta$.   
\end{restatable}

When $P$ is Markov to a tree (i.e. it is tree-structured), then the sample complexity improves:
\begin{restatable}{theorem}{thmupperrealizabe}
\label{tm:realizable}
Let $T^*$ be a directed tree and $P_{T^*}$ be a $T^*$-structured Gaussian.
Given $n$ i.i.d. samples from $P_{T^*}$,
for any $\eps, \delta>0$, if $n \gtrsim \frac{d}{\eps}\log\frac{d}{\delta}$, then $\wh{T}$ returned by \cref{algo:chow-liu}
satisfies
\begin{align*}
    \kl{P_{T^*}}{P_{\wh{T}}} \leq\eps,
\end{align*}
with probability at least $1-\delta$.
\end{restatable}

\textbf{Remark}: We can also obtain a sample-efficient algorithm for bounded-degree Gaussian {\em polytrees}, using the guarantees of the estimator $\wh{I}$, assuming that the skeleton is known. We defer the description of this result to \cref{sec:polytree_distlearning}.

\subsection{Distribution Learning Lower Bounds}

The main idea of our proof is to reduce a distribution testing problem to our problem. Intuitively, the distribution testing problem is defined as follows. Suppose $R^{(1)}$ and $R^{(2)}$ are two distributions whose $\kl{R^{(1)}}{R^{(2)}}$ is small.
We are given $n$ i.i.d. samples drawn from a distribution $P$ where $P$ is a $m$-variate distribution and each coordinate is distributed as either $R^{(1)}$ or $R^{(2)}$ uniformly and independently.
Our task is to determine which of $R^{(1)}$ or $R^{(2)}$ the samples are drawn from correctly for at least $m/2$ coordinates.
The formal definition will be presented in \cref{prob:reduction}.
When $\kl{R^{(1)}}{R^{(2)}}$ is sufficiently small, one should expect that $n$ needs to be large enough to solve this problem with probability $2/3$.
Hence, we construct the $(R^{(1)},R^{(2)})$ pairs for the non-realizable and realizable case accordingly.

\begin{restatable}{theorem}{thmnonrealizablelowerbound}
Suppose $P$ is an unknown Gaussian distribution.
Given $n$ i.i.d. samples drawn from $P$.
For any small $\eps>0$, if $n=o(d^2/\eps^2)$, no algorithm returns a directed tree $\wh{T}$ such that 
\begin{align*}
   \kl{P}{P_{\wh{T}}} \leq \min_{T\in\drtr} \kl{P}{P_{T}} + \eps
\end{align*}
with probability at least $2/3$.
\end{restatable}

\begin{restatable}{theorem}{thmrealizablelowerbound}
Suppose $P$ is an unknown Gaussian distribution such that there exists a directed tree $T^*$ that $P$ is $T^*$-structured, i.e. $P=P_{T^*}$.
Given $n$ i.i.d. samples drawn from $P$.
For any small $\eps>0$, if $n=o(d/\eps)$, no algorithm returns a directed tree $\wh{T}$ such that 
\begin{align*}
   \kl{P}{P_{\wh{T}}} \leq  \eps
\end{align*}
with probability at least $2/3$.
\end{restatable}

\section{Optimal Faithful Tree Learning}
\label{sec:faithfultree}
In the preceding section, we learned a tree-structured distribution under the KL distance, without concern for the learned tree structure. This viewpoint primarily pertains to \emph{distribution learning}. This section adopts an different approach, emphasizing the aspect of \emph{structure learning}.
Specifically, we assume the underlying graph structure is indeed a tree, more generally, a polytree. We introduce an estimator based on the classic PC algorithm  \citep{spirtes1991algorithm} and analyze its sample complexity under faithfulness.
Crucially, we provide a matching lower bound to conclude the minimax optimality of the algorithm, which offers insights into the difficulty of structure learning under faithfulness.

\subsection{Tree-Faithfulness}\label{sec:faithfultree:prelim}
As alluded to in \cref{sec:prelim}, the tree structure allows us to relax the usual notion of faithfulness:
\begin{definition}[Tree-faithfulness]\label{defn:treefaith}
We say distribution $P$ is tree-faithful to a polytree $T$ if 
\begin{enumerate}
    \item For any two nodes connected $X_j-X_k$, we have $X_k\not\indep X_j\given X_\ell$ for all $\ell \in V\cup \{\emptyset\}\setminus \{k,j\}$;
    \item For any $v$-structure $X_k\rightarrow X_\ell\leftarrow X_j$, we have $X_k\not\indep X_j\given X_\ell$.
\end{enumerate}
\end{definition}
\noindent
Tree-faithfulness comprises two components, each corresponding to adjacency-faithfulness and orientation-faithfulness respectively in restricted faithfulness (cf. \cref{defn:resfaith}). 
In comparison to adjacency-faithfulness, tree-faithfulness solely requires conditional dependence for neighbouring nodes with conditioning sets of size at most one. Likewise, compared to orientation faithfulness, tree-faithfulness only needs conditional dependence for $v$-structures given the the collider.
Let $\rho(X_j,X_k\given X_\ell)$ be the conditional correlation coefficient between $X_k$ and $X_j$ given $X_\ell$.
As usual, in order to establish uniform, finite-sample results, we need the following concept of $c$-strong tree-faithfulness:
\begin{definition}[$c$-strong tree-faithfulness]\label{defn:strongtreefaith}
We say that $P$ is $c$-strong tree-faithful to a polytree $T$ if 
\begin{enumerate}
    \item For any two nodes connected $X_j-X_k$, we have $\rho(X_k, X_j\given X_\ell)\ge c$ for $\ell \in V\cup \{\emptyset\}\setminus \{k,j\}$;
    \item For any $v$-structure $X_k\rightarrow X_\ell\leftarrow X_j$, we have $\rho(X_k, X_j\given X_\ell)\ge c$.
\end{enumerate}
\end{definition}
Under strong tree-faithfulness, we can now establish how the sample complexity depends on both the dimension $d$ and the signal strength $c$. 

\subsection{Structure Learning Upper Bounds}\label{sec:faithfultree:ub}
We develop the \texttt{PC-Tree} algorithm for learning polytrees as a modification to the classic PC algorithm, outlined in \cref{alg:tree}, effectively identifying the polytree's skeleton. 
An important by-product is the separation set resulted from the CI testing, which is used to obtain the CPDAG by applying an \texttt{ORIENT} step (\cref{alg:orient}) as in the original PC algorithm.

\begin{algorithm}[t]
\caption{\texttt{PC-Tree} algorithm}\label{alg:tree}
\textbf{Input:} $n$ i.i.d. samples $(X^{(i)}_1,\ldots,X^{(i)}_d)$\\
\begin{enumerate}
    \item Let $\widehat{E} = \emptyset$.
    \item For each pair $(j,k)$, $0\le j < k \le d$:
    \begin{enumerate}
        \item For all $\ell \in [d]\cup\{\emptyset\}\setminus \{j,k\}$:
        \begin{enumerate}
            \item Test $H_0:X_j\indep X_k\given X_\ell$ vs. $H_1:X_j\not\indep X_k\given X_\ell$, store the results.
        \end{enumerate}
        \item If all tests reject, then $\widehat{E}\leftarrow \widehat{E}\cup \{j-k\}$.
        \item Else (if some test accepts), let $S(j,k) = \{\ell\in [d]\cup\{\emptyset\}\setminus \{j,k\}: X_j \indep X_k\given X_\ell\}$.
    \end{enumerate}
\end{enumerate}
\textbf{Output:} $\widehat{T}= ([d],\widehat{E})$, separation set $S$.
\end{algorithm}

In contrast to the original PC algorithm, \texttt{PC-Tree} distinguishes itself in two key aspects. Firstly, when assessing the presence of an edge between any two nodes, instead of exploring all potential conditioning sets,
\texttt{PC-Tree} simplifies the process by exclusively testing marginal independence and conditional independence given only one other node.
Furthermore, a notable departure from the original PC algorithm is that \texttt{PC-Tree} combines marginal independence tests and conditional independence tests, as opposed to ignoring the latter once marginal independence is established.
\texttt{PC-Tree} will rely on sample (conditional) correlation coefficient for all the (conditional) independence tests when running the algorithm, see more details in \cref{app:tree:samplecorr}.

Now we are ready to provide the sample complexity of \texttt{PC-Tree} in the following theorem, whose proof is postponed to \cref{app:tree:cit} and~\ref{app:tree:ub}. 
\begin{restatable}{theorem}{thmtreeub}
\label{thm:tree:ub}
    For any $T\in\pltr$, assuming $P$ is $c$-strong tree-faithful to $T$, applying \cref{alg:tree} with 
    sample correlation for CI testing,
    if the sample size 
    \begin{align*}
        n \gtrsim \frac{1}{c^2}\bigg(\log d+ \log(1/\delta)\bigg)\,,
    \end{align*}
    then $\prob(\widehat{T} = \sk(T)) \ge 1-\delta$, and $\prob(\textsc{Orient}(\widehat{T},S) = \mec) \ge 1-\delta$
\end{restatable}
We may compare this upper bound $(\log d)/c^2$ with some of existing results on structure learning. 
Compared to learning equal variance general DAGs \citep{gao2022optimal} with optimal rates being $q\log (d/q)$, tree structure simplifies the problem by removing the factor of in-degree $q$.
As against recovering undirected graph in MRF \citep{misra2020information}, whose optimal sample complexity is $(s\log d)/\kappa^2$, we are able to improve the rate by the maximum degree $s$. 
Moreover, considering directed trees $T\in\drtr\subset \pltr$, \cref{lem:cconstant} shows $c$ to be a constant under mild assumption on the parametrization of~\eqref{eq:sem}, which assures possible concern of dependence on $c$.

\subsection{Structure Learning Lower Bounds}\label{sec:faithfultree:lb}
Having provided the sample complexity upper bound, we continue to derive a matching lower bound:
\begin{restatable}{theorem}{thmtreelb}\label{thm:tree:lb}
    Assuming $c$-strong tree-faithfulness, and $c^2\le 1/5$, $d\ge 4$, if the sample size is bounded as
    \begin{align*}
        n \le 
        \frac{1-2\delta}{8}\times \frac{\log d}{c^2}\,,
    \end{align*}
    then for any estimator $\widehat{T}$ for $\mec$,
    \begin{align*}
        \inf_{\widehat{T}}\sup_{\substack{T\in \pltr \\ P \text{ is } c\text{-strong} \\ \text{tree-faithful to }T}}\prob(\widehat{T}\ne \mec) \ge \delta - \frac{\log 2}{\log d}\,.
    \end{align*}
\end{restatable}
The lower bound in \cref{thm:tree:lb} implies the optimal sample complexity is $\Theta(\log d/c^2)$, where the dependence on $1/c^2$ term characterizes the hardness from ``how (Tree-)faithful'' the distribution is; and $\log d$ term comes from the cardinality of all polytrees, which is much smaller compared to number of all DAGs.

To prove this lower bound, we employee Fano's inequality \citep{yu1997assouad} and consider a subclass of $\drtr$ to exploit the property that any node in directed tree has at most one parent. 
This subclass of directed trees has large enough cardinality by Cayley's formula of undirected trees. 
With the parametrization of edge weights appropriately calibrated, we show the KL divergence between the distributions consistent with any two instances from the subclass is well controlled, which leads to the final lower bound.
The detailed proof can be found in \cref{app:tree:lb}.

\begin{remark}
    The optimality results in this section also extend to directed tree, polyforest and Markov chain. 
    Since the lower bound is constructed using directed trees, the optimality applies.
    For polyforest, which is essentially polytree but allows for disconnected component, \texttt{PC-Tree} algorithm is able to identify the correct skeleton. On the other hand, polytree is a subclass of polyforest, thus the lower bound in \cref{thm:tree:lb} applies. 
    For Markov chain, the algorithm is modified to dismiss marginal independence test, and the lower bound construction considers all Markov chains with the same way of parametrization as in \cref{thm:tree:lb}.
    All these graphical models share the optimal sample complexity $\Theta(\log d/c^2)$.
\end{remark}

\section{Experiments}
\label{sec:experiments}

\begin{figure*}[hbt!]
\centering 
\subfigure[SHD comparison]{\label{fig:gauss_100_shd}\includegraphics[width=0.49\linewidth]{\detokenize{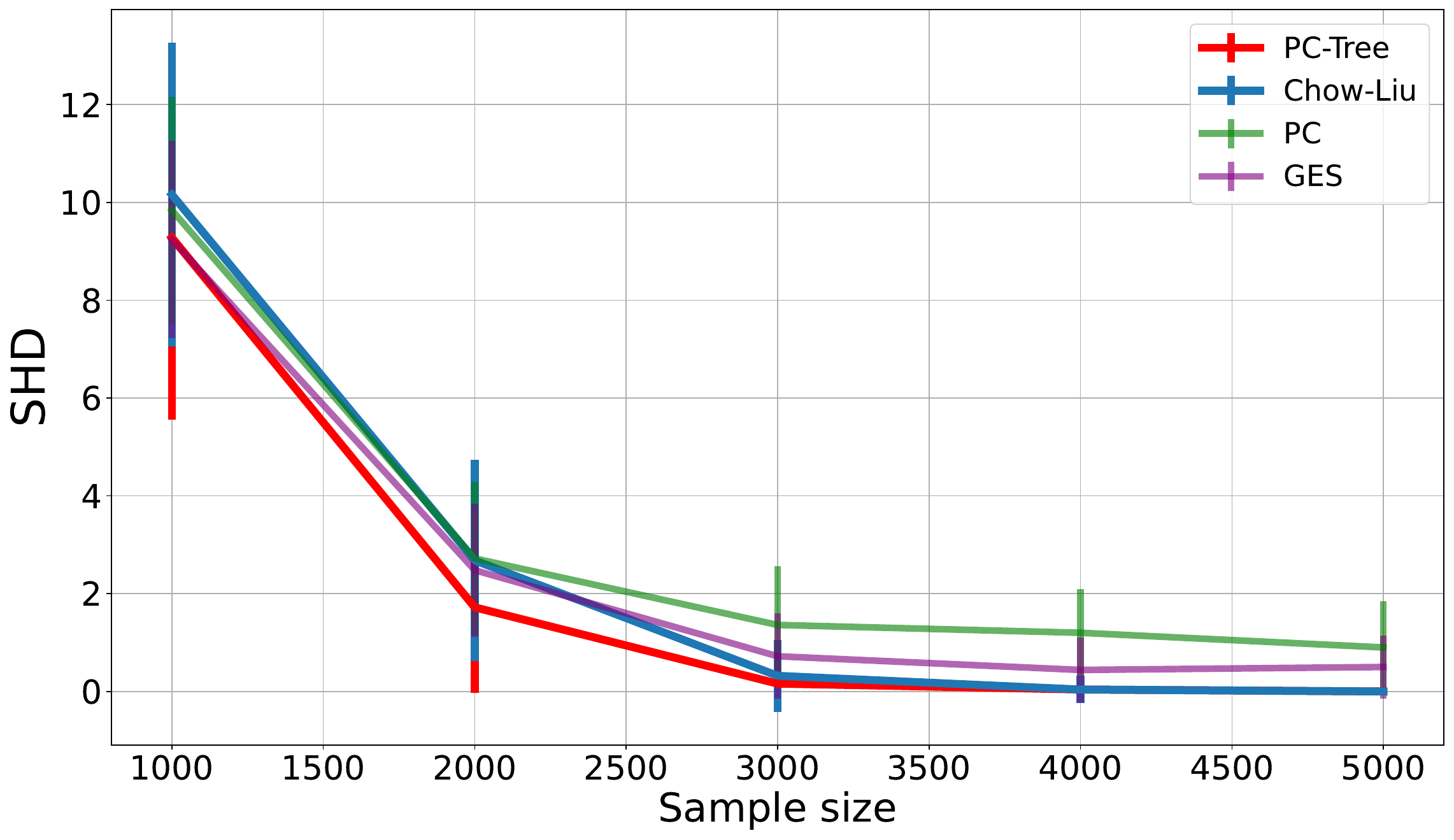}}}
\hspace{-0.1cm}
\subfigure[PRR comparison]{\label{fig:gauss_100_prr}\includegraphics[width=0.49\linewidth]{\detokenize{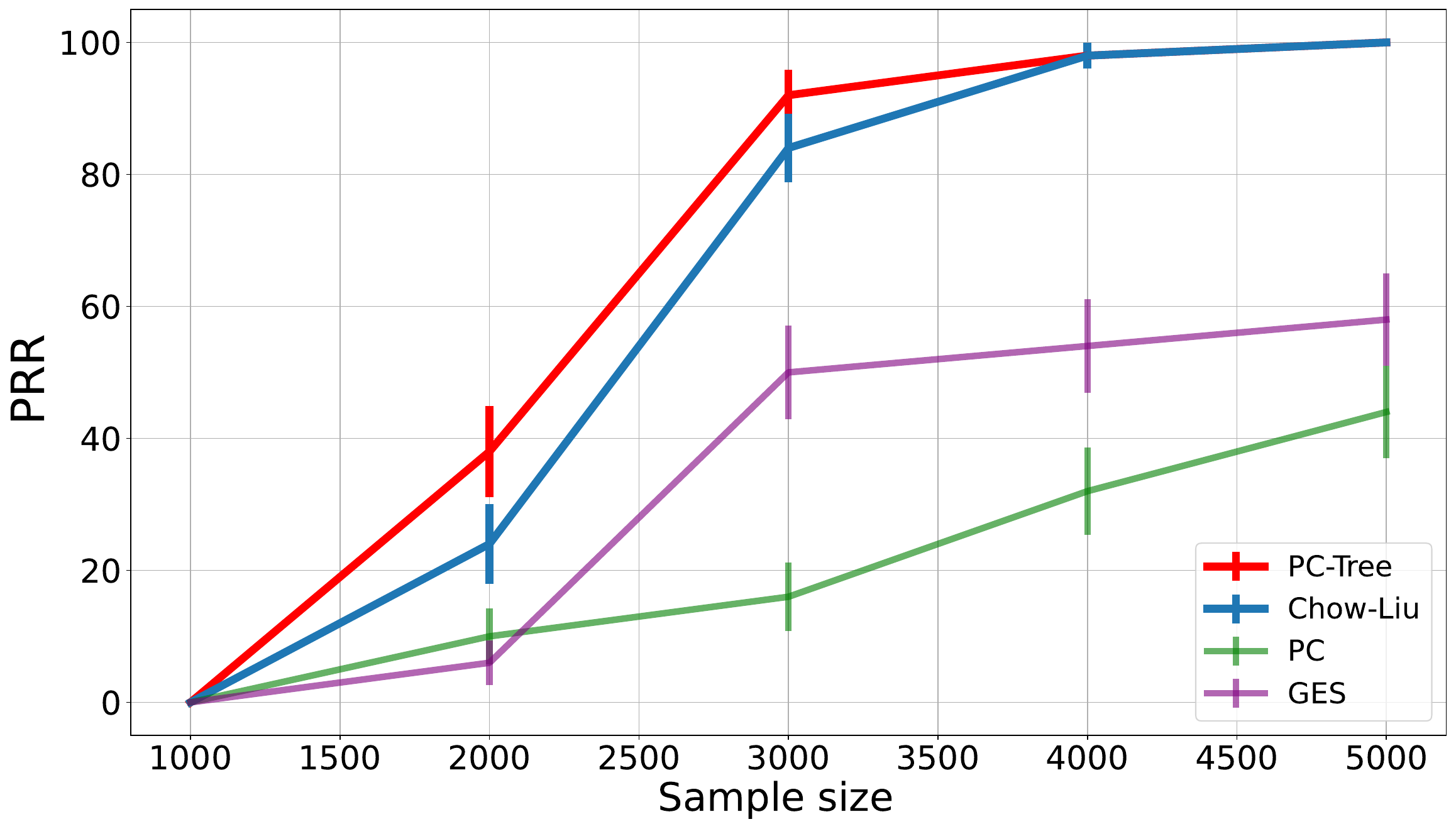}}}
\caption{Performance comparison for PC-Tree, Chow-Liu, PC and GES algorithm evaluated on SHD and PRR. The red, blue, green, purple lines are for PC-Tree, Chow-Liu, PC and GES respectively.}
\label{fig:gauss_shd_prr}
\end{figure*}

We conduct experiments to verify our findings in structure learning. For brevity, we report here only the most difficult setting with $d=100$ nodes; full details on the experiments and additional setups, e.g. when noise $\eta_k$ is not Gaussian, can be found in \cref{app:expt}.
We simulated random directed trees and synthetic data via~\eqref{eq:sem}. 
We compare the performance of PC-Tree, Chow-Liu to PC and GES as classical baselines when only faithfulness assumed.
Though Chow-Liu algorithm aims for distribution learning, it also estimates the skeleton as a byproduct. Therefore, to make fair comparison, we evaluate them by the accuracy of skeleton of the outputs (of PC-Tree, PC and GES).
The results on average Structural Hamming Distance (SHD) and the Precise Recovery Rate (PRR) are reported in \cref{fig:gauss_shd_prr}, where PRR measures the relative frequency of exact recovery.
From the figure, we can see PC-Tree algorithm does perform the best, especially the significantly better result on PRR over the baselines, which is the main metric we are concerned with and have established optimality for.
The competitive performance of Chow-Liu is also noticeable, for which we have not analyzed under the goal of structure learning, and we conjecture a similar sample complexity is shared with PC-Tree.

\section{Comparison and Discussion}
\label{sec:discussion}
The literature on distribution learning and structure learning have largely evolved separate from one another. An interesting aspect of our results is that they consider both problems in a unified setting, allowing for an explicit comparison of these problems.

First, it is clear that the non-realizable setting should not be compared to structure learning, since in the former setting there is no structure to speak of. In the realizable setting, however, it is reasonable to ask for a comparison. Comparing \eqref{eq:complexity:rlz} and \eqref{eq:complexity:sl}, it is easy to see that there is a phase transition when $\eps \asymp dc^2$.
Focusing on directed trees for an apple-to-apple comparison, if the SEM parameters, e.g. $\beta_k,\sigma^2_k$ in~\eqref{eq:sem} are bounded, then strong tree-faithfulness holds with $c\asymp 1$, see \cref{lem:cconstant}. 
In this case, the optimal sample complexity for structure learning is $\log d$ and $(d\log d)/\eps$ for distribution learning, which has an additional factor of $d/\eps$. Thus, as long as $\eps = o(d)$, which is typical, structure learning is easier than distribution learning.

Another interesting scenario arises when $\eps\ll dc^2$: Here, distribution learning is harder, however, we might hope to learn the structure of $T$ ``for free'' by first learning the distribution to within KL accuracy $\eps$. 
This is because, as $\eps$ goes to zero, $\widehat{P}$ converges to $P$, which implies we can use $\widehat{P}$ directly to estimate partial correlations for structure learning. Then the question boils down to whether there exists a good estimator of the structure that exploits $\widehat{P}$ when $\eps \ll dc^2$.
\cref{lem:distlnvsstrucln} shows that as long as the estimator is agnostic to $\widehat{P}$ (in the sense that it treats $\widehat{P}$ as a black-box), then we must have at least $\eps\ll c^2$. Thus, there is a regime $c^2\ll\eps\ll dc^2$ where distribution learning does not automatically imply structure learning, at least in general. 
It remains as an interesting open question how small $\eps$ must be for $\widehat{P}$ to be efficiently used for structure learning, or whether or not there exist \emph{specific} estimators $\widehat{P}$ that can be used for structure learning when $c^2\ll\eps\ll dc^2$.

Extending these results beyond the Gaussians we consider here (and finite alphabets as in previous work) is a promising direction for future research. Especially interesting would be bounds in a non-parametric setting.

\bibliographystyle{alpha}
\bibliography{references}  

\newpage 
\appendix
\section{Comparing Structure Learning and Distribution Learning}

\begin{lemma}\label{lem:cconstant}
    Suppose $T\in\drtr$ and $P$ is parameterized using $\{\beta_k,\sigma^2_k\}_{k=1}^d$ as~\eqref{eq:sem} according to $T$. 
    If there exists a constant $M>1$ such that for any $k\in[d]$,
    \begin{align*}
        |\beta_{kj}| & \in [M^{-1},M],\qquad \forall \beta_{kj}\ne 0\\
        \sigma_k^2 & \in [M^{-1},M] \,,
    \end{align*}
    then $P$ is $c$-strong Tree-faithful to $T$ for some $c\asymp 1$.
\end{lemma}

\begin{proof}[Proof of \cref{lem:cconstant}]

\begin{figure}[!hb]
    \centering
    \begin{tikzpicture}[->,>=stealth',node distance=1cm, auto]
    \node[est] (Xj) {$X_j$};
    \node[est, above = of Xj] (d1) {$\cdots$};
    \node[est, above = of d1] (Xl2) {$X_\ell^{(2)}$};
    \node[est, below left = of Xj] (Xk) {$X_k$};
    \node[est, below left = of Xk] (d2) {$\cdots$};
    \node[est, below left = of d2] (Xl4) {$X_\ell^{(4)}$};
    \node[est, below right = of Xj] (d3) {$\cdots$};
    \node[est, below right = of d3] (Xl3) {$X_\ell^{(3)}$};
    \path[pil] (Xl2) edgenode {} (d1);
    \path[pil] (d1) edgenode {} (Xj);
    \path[pil] (Xj) edgenode {} (Xk);
    \path[pil] (Xk) edgenode {} (d2);
    \path[pil] (d2) edgenode {} (Xl4);
    \path[pil] (Xj) edgenode {} (d3);
    \path[pil] (d3) edgenode {} (Xl3);
    \end{tikzpicture}  
    \caption{Four cases of $\ell$ to verify for $c$-strong Tree-faithfulness, indicated by the superscript of $X_\ell$. The first case is when $\ell=\emptyset$. The second, third and fourth are when $\ell$ is the ancestor of $j$, descendant of $j$ and descendant of $k$.}
    \label{fig:cconstant}
\end{figure}
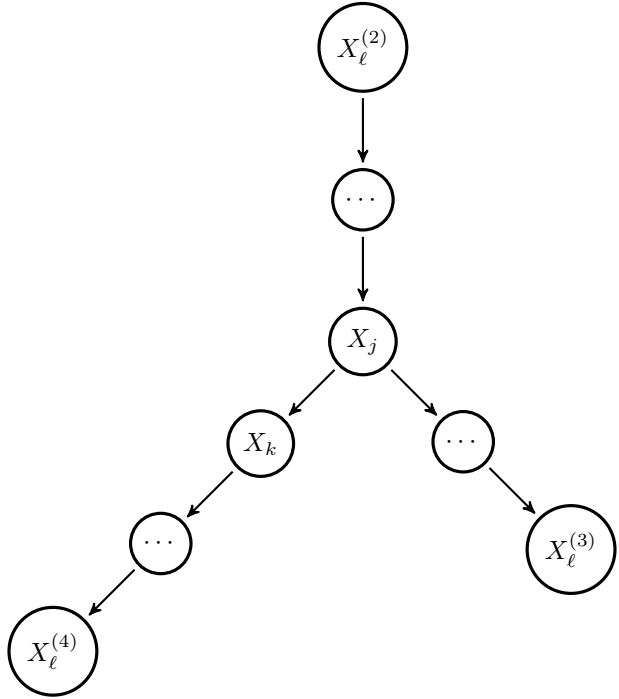

Since a directed tree $T$ does not have any $v$-structures, we only need to verify adjacency faithfulness in \cref{defn:strongtreefaith}. 
For any two nodes connected as $j\to k$, we want to check whether $\rho(X_j,X_k\given X_\ell)$ is lower bounded by some constant for $\ell\in V\cup \{\emptyset\}\setminus\{j,k\}$. There are four cases of $\ell$ to consider, see \cref{fig:cconstant}:
\begin{itemize}
    \item $\ell=\emptyset$: To simplify the notation, we write 
    \begin{align*}
        X_k = \beta_k \times X_j + \eta_k
    \end{align*}
    with $\beta_k\in \mathbb{R}$ and $|\beta_k|\in [M^{-1},M]$, $\var(\eta_k) = \sigma^2_k$. We also write $V_j^2:=\var(X_j)\ge \sigma_j^2$. Hence,
    \begin{align*}
        \rho(X_j,X_k) = \frac{\beta_kV_j^2}{\sqrt{V_j^2}\sqrt{\beta_k^2V_j^2 + \sigma_k^2}} = \frac{1}{\sqrt{1 + \sigma^2_k/\beta_k^2V_j^2}} \ge \frac{1}{\sqrt{1 + \sigma^2_k/\beta_k^2\sigma_j^2}} \gtrsim 1\,.
    \end{align*}
    \item $\ell\in \an(j)$: Write $V^2_{j\given \ell}=\var(X_j\given X_\ell) \ge \sigma_j^2$, hence
    \begin{align*}
        \rho(X_j,X_k\given X_\ell) = \frac{\beta_kV_{j\given \ell}^2}{\sqrt{V_{j\given \ell}^2}\sqrt{\beta_k^2V_{j\given \ell}^2 + \sigma_k^2}} \ge \frac{1}{\sqrt{1 + \sigma^2_k/\beta_k^2\sigma_j^2}} \gtrsim 1\,.
    \end{align*}
    \item $\ell\in \de(j)$: Suppose the directed path from $j$ to $\ell$ is
    $j\to h_1\to h_2\to \ldots \to h_q \to \ell$, $q$ can be $0$, then we can write
    \begin{align*}
        X_\ell = b_1 X_j + u_1\,,
    \end{align*}
    with 
    \begin{align*}
        b_1 = \beta_\ell\prod_{i=1}^q\beta_{h_i} , \qquad u_1 = \eta_\ell + \beta_\ell\sum_{i=1}^q \eta_{h_i}\prod_{t=i+1}^q\beta_{h_t} \,,
    \end{align*}
    and
    \begin{align*}
        \nu_1^2:=\var(u_1) = \sigma^2_\ell + \beta_\ell^2\sum_{i=1}^q\sigma_{h_i}^2\prod_{t=i+1}^q\beta^2_{h_t} \ge \beta_\ell^2 \sigma^2_{h_1}\prod_{t=2}^q\beta^2_{h_t}\,.
    \end{align*}
    So we have $b_1^2 / \nu_1^2 \le \beta_{h_1}^2 / \sigma^2_{h_1} \asymp 1$. The covariance among $X_{j}, X_k,X_\ell$ is
    \begin{align*}
        \cov(X_j,X_k,X_\ell) = \begin{pmatrix}
            V_j^2 & \beta_kV_j^2 & b_1V_j^2 \\
            * & \beta_k^2V_j^2+\sigma^2_k & b_1\beta_kV_j^2 \\
            * & * & b^2_1V_j^2 + \nu_1^2
        \end{pmatrix}\,.
    \end{align*}
    Then the conditional covariance is 
    \begin{align*}
        \cov(X_j,X_k\given X_\ell) \propto \begin{pmatrix}
            \nu_1^2 & \beta_k\nu^2_1 \\
            *& \beta_k^2\nu_1^2 + \sigma_k^2b_1^2+\sigma^2_k\nu_1^2 / V_j^2
        \end{pmatrix} \,.
    \end{align*}
    Therefore,
    \begin{align*}
        \rho(X_j,X_k\given X_\ell) = \frac{1}{\sqrt{1 + \frac{\sigma_k^2}{\beta_k^2}\times \frac{b_1^2}{\nu_1^2} + \frac{\sigma_k^2}{V_j^2\beta_k^2}}} \gtrsim 1\,.
    \end{align*}
    \item $\ell\in \de(k)$: Similarly, we can write 
    \begin{align*}
        X_\ell = b_2 X_k + u_2, \qquad \var(u_2) = \nu^2_2\,,
    \end{align*}
    with $b_2^2/\nu_2^2\lesssim 1$.
    The covariance among $X_{j}, X_k,X_\ell$ is
    \begin{align*}
        \cov(X_j,X_k,X_\ell) = \begin{pmatrix}
            V_j^2 & \beta_kV_j^2 & b_2\beta_kV_j^2 \\
            * & V_k^2 & b_2V_k^2 \\
            * & * & b^2_2V_k^2 + \nu_2^2
          \end{pmatrix}\,.
    \end{align*}
    Then the conditional covariance is 
    \begin{align*}
        \cov(X_j,X_k\given X_\ell) \propto \begin{pmatrix}
            b_2^2\sigma_k^2V_j^2 + \nu_2^2V_j^2 & \beta_kV_j^2\nu_2^2 \\
            *& \nu^2_2V_k^2
        \end{pmatrix} \,.
    \end{align*}
    Therefore,
    \begin{align*}
        \rho(X_j,X_k\given X_\ell) = \frac{1}{\sqrt{(1 + \frac{\sigma^2_k}{\beta_k^2V_j^2})(1 + \frac{b_2^2}{\nu_2^2}\sigma_k^2)}}\gtrsim 1\,.
    \end{align*}
\end{itemize}
In all four cases, $\rho(X_j,X_k\given X_\ell) \gtrsim 1$, thus $c$-strong Tree-faithfulness is satisfied with some $c\asymp 1$.
\end{proof}

\begin{lemma}\label{lem:distlnvsstrucln}
    Let $\mathcal{A}$ denote some distribution learning algorithm such that given a tree-structured distribution $P$,
    $\mathcal{A}$ takes data from $P$ and outputs $\widehat{P}$ with $\dkl(P\|\widehat{P})\le \eps$. 
    If $\eps\gtrsim c^2$, then for any estimator $\widehat{T}(\widehat{P})$ for $\mec$ using solely $\widehat{P}$,
    \begin{align*}
        \inf_{\widehat{T}(\widehat{P})}
        \sup_{\substack{T\in \drtr \\ P\text{ is }c\text{-strong} \\ \text{Tree-faithful to }T}}
        \sup_{\mathcal{A}}
        \prob(\widehat{T}(\widehat{P})\ne \mec) = 1\,.
    \end{align*}
\end{lemma}

\begin{proof}

We construct $T,T'\in \drtr$ with different skeletons, and $P,P'$ Markov and strongly faithful to $T,T'$ respectively such that $\dkl(P\|P')\asymp c^2$. 
In this way, consider the ground truth to be $T$ and $P$, and supppose $\mathcal{A}$ outputs $\widehat{P}=P'$. Then we have $\dkl(P\|\widehat{P})\le \eps$ with $\eps\asymp c^2$. While $P$ and $\widehat{P}=P'$ correspond to different structures, thus any estimator using solely $\widehat{P}$ cannot uniformly find the true structure.

It remains to show the construction: Consider $T$ and $T'$ as follows:

\begin{figure}[h]
\centering 
\subfigure[Tree $T$]{
    \begin{tikzpicture}[->,>=stealth',node distance=1cm, auto]
    \node[est] (X2) {$X_2$};
    \node[est, right = of X2] (X3) {$X_3$};
    \node[est, left = of X2] (X1) {$X_1$};
    \node[est, below left = of X2, xshift=-1.5cm] (X4) {$X_4$};
    \node[est, right = of X4] (X5) {$X_5$};
    \node[est, right = of X5] (dot1) {$\cdots$};
    \node[est, right = of dot1] (Xd) {$X_d$};
    \path[pil] (X1) edgenode {} (X2);
    \path[pil] (X2) edgenode {} (X3);
    \path[pil] (X2) edgenode {} (X4);
    \path[pil] (X2) edgenode {} (X5);
    \path[pil] (X2) edgenode {} (dot1);
    \path[pil] (X2) edgenode {} (Xd);
    \end{tikzpicture}
}
\hspace{0.8cm}
\subfigure[Tree $T'$]{
    \begin{tikzpicture}[->,>=stealth',node distance=1cm, auto]
    \node[est] (X2) {$X_2$};
    \node[est, right = of X2] (X3) {$X_3$};
    \node[est, right = of X3] (X1) {$X_1$};
    \node[est, below left = of X2, xshift=-1.5cm] (X4) {$X_4$};
    \node[est, right = of X4] (X5) {$X_5$};
    \node[est, right = of X5] (dot1) {$\cdots$};
    \node[est, right = of dot1] (Xd) {$X_d$};
    \path[pil] (X3) edgenode {} (X1);
    \path[pil] (X2) edgenode {} (X3);
    \path[pil] (X2) edgenode {} (X4);
    \path[pil] (X2) edgenode {} (X5);
    \path[pil] (X2) edgenode {} (dot1);
    \path[pil] (X2) edgenode {} (Xd);
    \end{tikzpicture}
}
\caption{Construction for \cref{lem:distlnvsstrucln}.}
\label{fig:lemA2}
\end{figure}

We parameterize $P,P'$ as the lower bound construction in \cref{app:tree:lb}:
\begin{align*}
    X_k = \beta X_{\pa(k)} + \eta_k\,,
\end{align*}
where $\beta=\sqrt{2}c$, $\eta_k\sim\mathcal{N}(0,1)$ and \cref{lem:tree:beta} makes sure they are $c$-strong tree faithful. 
Now we only need to compute the KL divergence:
\begin{align*}
    \dkl(P\|P') & = \E_P \log\frac{\prod_k P(X_k\given \pa(k))}{\prod_j P'(X_j\given \pa(j))} \\
    & = \E_P \log\frac{P(X_3\given X_2)P(P_2\given X_1)P(X_1)}{P(X_1\given X_3)P(X_3\given X_2)P(X_2)}\\
    & = \E_P \frac{1}{2} \bigg( X_2^2 + (X_1-\beta X_3)^2 - X_1^2 - (X_3-\beta X_2)^2 \bigg)\\
    & = \frac{1}{2} \bigg( -\beta^4 + \beta^6 + 2(\beta^2 + \beta^4 - \beta^3) \bigg) \\
    & \le 2\beta^2 = 4c^2\,,
\end{align*}
which completes the proof.
\end{proof}

\section{Proofs of \cref{sec:chow-liu tree}}
\label{sec:deffer_proof}

\subsection{Preliminaries}

We first state some useful lemmas. They are well-known results for the concentration bound on variances and covariances. For completeness, we provide the proof below.

\begin{lemma}[Guarantees of variance recovery]
\label{lm:deltasigma}
    Suppose $X$ is the random variable of $\N(0,\sigma^2)$ for some $\sigma>0$.
    Let $X^{(1)},\dots,X^{(n)}$ be the i.i.d. samples of $X$ and $\wh{\sigma}^2$ be $\frac{1}{n}\sum_{i=1}^n (X^{(i)})^2$.
    Then, for any $t\in(0,1)$, we have
    \begin{align*}
        \abs{\wh{\sigma}^2 - \sigma^2} < t \sigma^2
    \end{align*}
    with probability $1-O(e^{-\Omega(nt^2)})$.
\end{lemma}

\begin{proof}

We first show that the probability of $\wh{\sigma}^2 > (1+t) \sigma^2$ is bounded by $e^{-\Omega(nt^2)}$ and the other inequality $\wh{\sigma}^2 < (1-t) \sigma^2$ follows similarly.

Note that 
\begin{align*}
    \wh{\sigma}^2 > (1+t) \sigma^2
    & \Leftrightarrow
    e^{\lambda \frac{1}{n}\sum_{i=1}^n (X^{(i)})^2} > e^{\lambda (1+t)\sigma^2} & \text{for any $\lambda>0$.}
\end{align*}
By Markov inequality, the probability of $\wh{\sigma}^2 > (1+t)\sigma^2$ is bounded by
\begin{align*}
    \E(e^{\lambda \frac{1}{n}\sum_{i=1}^n (X^{(i)})^2}) / e^{\lambda (1+t) \sigma^2}
    & =
    \underbrace{\E(e^{\lambda \frac{1}{n}X^2})^n}_{\text{by i.i.d. assumption}} / e^{\lambda (1+t) \sigma^2}. \numberthis \label{eq:markov_var}
\end{align*}
Hence, we need to bound the term $\E(e^{\lambda \frac{1}{n}X^2})$.
\begin{align*}
    \E(e^{\lambda \frac{1}{n}X^2})
    & =
    \int_{-\infty}^\infty \frac{1}{\sqrt{2\pi \sigma^2}} e^{\lambda\frac{1}{n}x^2} e^{-\frac{1}{2\sigma^2}x^2} d x 
    =
    \frac{1}{\sqrt{1-\frac{2\sigma^2\lambda}{n}}} & \text{as long as $\frac{1}{2\sigma^2} - \frac{\lambda}{n} > 0$}
\end{align*}
Moreover, using the inequality $\frac{1}{\sqrt{1-x}} \leq e^{\frac{1}{2}x+x^2}$ for $x < \frac{1}{2}$, we have 
\begin{align*}
    \E(e^{\lambda \frac{1}{n}X^2})
    & \leq
    e^{\frac{\sigma^2\lambda}{n} + \frac{4\sigma^4\lambda^2}{n^2}} & \text{as long as $\frac{2\sigma^2\lambda}{n}<\frac{1}{2}$} \numberthis\label{eq:exp_var}
\end{align*}
Plugging \eqref{eq:exp_var} into \eqref{eq:markov_var}, the probability of $\wh{\sigma}^2 > (1+t) \sigma^2$ is bounded by
\begin{align*}
    (e^{\frac{\sigma^2\lambda}{n} + \frac{4\sigma^4\lambda^2}{n^2}})^n/e^{\lambda(1+t)\sigma^2}
    & =
    e^{-\frac{4\sigma^4\lambda^2}{n} + \lambda t \sigma^2}
    =
    e^{-\frac{4\sigma^4}{n}(\lambda - \frac{nt}{8\sigma^2})^2 + \frac{nt^2}{16}}
\end{align*}
and, by taking $\lambda = \frac{nt}{8\sigma^2}$, it becomes $e^{-\frac{nt^2}{16}}$.
    
\end{proof}

\begin{lemma}[Guarantees of correlation coefficient recovery]
\label{lm:deltarho}
Suppose $(X,Y)$ is the random variable of $\N(0,\Sigma)$ for some positive definite $\Sigma = \begin{bmatrix}
    \sigma_x^2 & \rho_{xy} \\
    \rho_{xy} & \sigma_y^2
\end{bmatrix}$.
Let $(X^{(1)},Y^{(1)}),\dots,(X^{(n)},Y^{(n)})$ be the i.i.d. samples of $(X,Y)$ and $\wh{\rho}_{xy}$ be $\frac{1}{n}\sum_{i=1}^n {X^{(i)}}{Y^{(i)}}$.
Then, for any $t\in(0,1)$, we have
\begin{align*}
    \abs{\wh{\rho}_{xy} - \rho_{xy}}
    & < 
    t \sigma_x \sigma_y
\end{align*}
with probability $1-O(e^{-\Omega(nt^2)})$.
\end{lemma}

\begin{proof}

We first show that the probability of $\wh{\rho}_{xy} > \rho_{xy} + t \sigma_x\sigma_y$ is bounded by $e^{-\Omega(nt^2)}$ and the other inequality $\wh{\rho}_{xy} < \rho - t \sigma_x\sigma_y$ follows similarly.

Note that
\begin{align*}
    \wh{\rho}_{xy} > \rho_{xy} + t\sigma_x\sigma_y
    & \Leftrightarrow
    e^{\lambda \frac{1}{n} \sum_{i=1}^n {X^{(i)}}{Y^{(i)}}} > e^{\lambda(\rho_{xy} + t\sigma_x\sigma_y)} & \text{for any $\lambda>0$.}
\end{align*}
By Markov inequality, the probability of $\wh{\rho}_{xy} > \rho_{xy} + t \sigma_x\sigma_y$ is bounded by 
\begin{align*}
    \E(e^{\lambda\frac{1}{n}\sum_{i=1}^n{X^{(i)}}{Y^{(i)}}}) / e^{\lambda(\rho_{xy} + t\sigma_x\sigma_y)}
    & =
    \underbrace{\E(e^{\lambda\frac{1}{n}XY})^n}_{\text{by i.i.d. assumption}} / e^{\lambda(\rho_{xy} + t\sigma_x\sigma_y)}. \numberthis\label{eq:markov_covar}
\end{align*}
Hence, we need to bound the term $\E(e^{\lambda\frac{1}{n}XY})$.
\begin{align*}
    \E(e^{\lambda\frac{1}{n}XY})
    & =
    \int_{-\infty}^{\infty}\int_{-\infty}^{\infty}\frac{1}{\sqrt{(2\pi)^2(\sigma_x^2\sigma_y^2 - \rho_{xy}^2)}} e^{\lambda\frac{1}{n}xy} e^{-\frac{1}{2(\sigma_x^2\sigma_y^2 - \rho_{xy}^2)} (\sigma_y^2x^2 - 2\rho_{xy}xy + \sigma_x^2y^2)} dxdy \\
    & =
    \frac{1}{\sqrt{1-\frac{2\rho_{xy}\lambda}{n} - \frac{\lambda^2\Delta}{n^2}}} \qquad \text{as long as $\sigma_x^2\sigma_y^2 > (\rho_{xy} + \frac{\lambda\Delta}{n})^2$ where $\Delta = \sigma_x^2\sigma_y^2 - \rho_{xy}^2$} 
\end{align*}
Moreover, using the inequality $\frac{1}{\sqrt{1-x}} \leq e^{\frac{1}{2}x + x^2}$ for $x<\frac{1}{2}$, we have
\begin{align*}
    \E(e^{\lambda\frac{1}{n}XY})
    & \leq
    e^{\frac{1}{2}(\frac{2\rho_{xy}\lambda}{n} +\frac{\lambda^2\Delta}{n^2})+(\frac{2\rho_{xy}\lambda}{n} +\frac{\lambda^2\Delta}{n^2})^2} \qquad \text{as long as $\frac{2\rho_{xy}\lambda}{n} +\frac{\lambda^2\Delta}{n^2} < \frac{1}{2}$} \\
    & \leq
    e^{\frac{\rho_{xy}\lambda}{n} + \frac{\lambda^2\sigma_x^2\sigma_y^2}{2n^2} + (\frac{2\sigma_x\sigma_y\lambda}{n} + \frac{\lambda^2\sigma_x^2\sigma_y^2}{n^2})^2} \qquad \text{using $\rho_{xy}\leq \sigma_x\sigma_y$ and $\Delta\leq \sigma_x^2\sigma_y^2$} \\
    & \leq
    e^{\frac{\rho_{xy}\lambda}{n} + \frac{19\lambda^2\sigma_x^2\sigma_y^2}{2n^2}} \qquad \text{as long as $\frac{\lambda\sigma_x\sigma_y}{n}<1$}\numberthis\label{eq:exp_covar}
\end{align*}

Plugging \eqref{eq:exp_covar} into \eqref{eq:markov_covar}, the probability of $\wh{\rho}_{xy} > \rho_{xy} + t\sigma_x\sigma_y$ is bounded by 
\begin{align*}
    (e^{\frac{\rho_{xy}\lambda}{n} + \frac{19\lambda^2\sigma_x^2\sigma_y^2}{2n^2}})^n / e^{\lambda(\rho_{xy} + t\sigma_x\sigma_y)}
    & =
    e^{-\frac{19\sigma_x^2\sigma_y^2}{2n}\lambda^2 + t\sigma_x\sigma_y\lambda}
    =
    e^{-\frac{19\sigma_x^2\sigma_y^2}{2n}(\lambda - \frac{tn}{19\sigma_x\sigma_y})^2 + \frac{t^2n}{38}}
\end{align*}
and, by taking $\lambda = \frac{tn}{19\sigma_x\sigma_y}$, it becomes $e^{-\frac{t^2n}{38}}$.
 
\end{proof}

\begin{corollary}
\label{cor:rhosigma}
    Suppose $(X_1,\dots,X_d)$ is the random variable of $\N(0,\Sigma)$ for some positive definite $\Sigma$ where $\rho_{ij} := \Sigma_{ij}$ and $\sigma_i^2 := \Sigma_{ii}$ for $i,j=1,\dots,d$.
    Let $(X_1^{(1)},\dots,X_d^{(1)}),\dots,(X_1^{(n)},\dots,X_d^{(n)})$ be the i.i.d. samples of $(X_1,\dots,X_d)$ and 
    \begin{align*}
        \wh{\rho}_{jk} = \frac{1}{n}\sum_{i=1}^n X_j^{(i)}X_k^{(i)} \qquad \text{and} \qquad\wh{\sigma}_j^2 = \frac{1}{n}\sum_{i=1}^n (X_j^{(i)})^2.
    \end{align*}
    Then, when $n = \Theta(\frac{1}{t^2}\log \frac{d}{\delta})$, we have, for all $j,k=1,\dots,d$,
    \begin{align*}
        \abs{\wh{\rho}_{jk} - \rho_{jk}} \leq t \sigma_j\sigma_k \qquad \text{and} \qquad\abs{\wh{\sigma}_j^2 - \sigma_j^2} \leq t \sigma_j^2
    \end{align*}
    with probability $1-\delta$.
\end{corollary}

\subsection{Conditional Mutual Information Tester}\label{sec:cmit}

In this subsection, we define the conditional mutual information tester used in our main algorithm.

Suppose $(X,Y,Z)$ is the random variable of $\N(0,\Sigma)$ for some positive definite $\Sigma =
\begin{bmatrix}
    \sigma_x^2 & \rho_{xy} & \rho_{xz} \\
    \rho_{xy} & \sigma_y^2 & \rho_{yz} \\
    \rho_{xz} & \rho_{xy} & \sigma_z^2
\end{bmatrix}$.
WLOG, we can express $(X,Y,Z)$ as
\begin{align*}
    Y & = \beta_{xy} X + \eta_y \\
    Z & = \gamma_{xz} X + \gamma_{yz}Y + \eta_z
\end{align*}
for some random variables $\eta_y,\eta_z$ where 
\begin{align*}
    \beta_{xy} = \frac{\rho_{xy}}{\sigma_x^2} \qquad & \text{and} \qquad
    \begin{bmatrix}
        \gamma_{xz} \\
        \gamma_{yz}
    \end{bmatrix} 
    =
    \begin{bmatrix}
        \sigma_x^2 & \rho_{xy} \\
        \rho_{xy} & \sigma_y^2
    \end{bmatrix}^{-1}
    \begin{bmatrix}
        \rho_{xz} \\
        \rho_{yz}
    \end{bmatrix}.
\end{align*}
Let $\sigma_{y\mid x}^2$ be $\E(\eta_y^2)$ and $\sigma_{z\mid x,y}^2$ be $\E(\eta_z^2)$.
Recall that the mutual information $I(X;Y)$ and the conditional mutual information $I(Y;Z\mid X)$ are defined (equivalently) as
\begin{align*}
    I(X;Y)
    & :=
    \frac{1}{2}\log(1+\frac{\beta_{xy}^2\sigma_x^2}{\sigma_{y\mid x}^2}) \qquad \text{and} \qquad
    I(Y;Z \mid X)
    :=
    \frac{1}{2}\log(1+\frac{\gamma_{yz}^2\sigma_{y\mid x}^2}{\sigma_{z\mid x,y}^2})
\end{align*}
Let $(X^{(1)},Y^{(1)},Z^{(1)}),\dots,(X^{(n)},Y^{(n)},Z^{(n)})$ be the i.i.d. samples of $(X,Y,Z)$.
Then we define the empirical mutual information $\wh{I}(X;Y)$ and the empirical mutual information $\wh{I}(Y;Z\mid X)$ to be 
\begin{align*}
    \wh{I}(X;Y)
    & :=
    \frac{1}{2}\log(1+\frac{\wh{\beta}_{xy}^2 \wh{\sigma}_x^2}{\wh{\sigma}_{y\mid x}^2}) \qquad \text{and} \qquad 
    \wh{I}(Y;Z\mid X)
    :=
    \frac{1}{2}\log(1+\frac{\wh{\gamma}_{yz}^2 \wh{\sigma}_{y\mid x}^2}{\wh{\sigma}_{z\mid x,y}^2}) \numberthis\label{eq:emp_mi}
\end{align*}
where the $\wh{\cdot}$ mark indicates the empirical version of the quantity.
Namely,
\begin{align*}
\begin{cases}
    \wh{\sigma}_x^2
    & :=
    \frac{1}{n}\sum_{i=1}^n (X^{(i)})^2, \qquad
    \wh{\sigma}_y^2
    :=
    \frac{1}{n}\sum_{i=1}^n (Y^{(i)})^2, \qquad
    \wh{\sigma}_z^2
    :=
    \frac{1}{n}\sum_{i=1}^n (Z^{(i)})^2, \\
    \wh{\rho}_{xy}
    & :=
    \frac{1}{n}\sum_{i=1}^n X^{(i)}Y^{(i)}, \qquad
    \wh{\rho}_{xz}
    :=
    \frac{1}{n}\sum_{i=1}^n X^{(i)}Z^{(i)}, \qquad
    \wh{\rho}_{yz}
    :=
    \frac{1}{n}\sum_{i=1}^n Y^{(i)}Z^{(i)}, \\
    \wh{\beta}_{xy}
    & :=
    \frac{\wh{\rho}_{xy}}{\wh{\sigma}_x^2},\qquad
    \begin{bmatrix}
        \wh{\gamma}_{xz} \\
        \wh{\gamma}_{yz}
    \end{bmatrix}
    :=
    \begin{bmatrix}
        \wh{\sigma}_x^2 & \wh{\rho}_{xy} \\
        \wh{\rho}_{xy} & \wh{\sigma}_y^2
    \end{bmatrix}^{-1}
    \begin{bmatrix}
        \wh{\rho}_{xz} \\
        \wh{\rho}_{yz}
    \end{bmatrix}, \\
    \wh{\sigma}_{y\mid x}^2
    & :=
    \wh{\sigma}_y^2 - \wh{\beta}_{xy}^2\wh{\sigma}_x^2   \qquad \text{and} \qquad
    \wh{\sigma}_{z\mid x,y}^2
    :=
    \wh{\sigma}_z^2 - \wh{\gamma}_{xz}^2\wh{\sigma}_x^2 - \wh{\gamma}_{yz}^2\wh{\sigma}_{y\mid x}^2.
\end{cases} \numberthis\label{eq:emp_def}
\end{align*}
Note that the above quantities depend on the samples but we will not emphasize it if the set of samples is clear in the context.
Also, it is known that, by the chain rule of mutual information,
\begin{align*}
    I(X;Y) - I(X;Z)
    & =
    I(X;Y\mid Z) - I(X;Z \mid Y)\numberthis\label{eq:conmui_chain}\\
    \wh{I}(X;Y) - \wh{I}(X;Z)
    & =
    \wh{I}(X;Y\mid Z) - \wh{I}(X;Z \mid Y). \numberthis\label{eq:conmui_chain_emp}
\end{align*}

From now on, when we have a $d$-dimensional random variable $(X_1,\dots,X_d)$, we abuse the notations defined in \eqref{eq:emp_def} by replacing $x,y,z$ with $i,j,k$ for $i,j,k=1,\dots,d$.

\begin{lemma} \label{lem:betasigmabound}
Suppose $(X_1,\dots,X_d)$ is the random variable of $\N(0,\Sigma)$ for some positive definite $\Sigma$ where $\rho_{ij} := \Sigma_{ij}$ and $\sigma_i^2 := \Sigma_{ii}$ for $i,j=1,\dots,d$.
Let $(X_1^{(1)},\dots,X_d^{(1)}),\dots,(X_1^{(n)},\dots,X_d^{(n)})$ be the i.i.d. samples of $(X_1,\dots,X_d)$ and $\wh{\gamma}_{ij}, \wh{\sigma}_{i\mid j}, \wh{\sigma}_{i\mid j,k}$ be the quantities defined in \eqref{eq:emp_def} for $i,j,k=1,\dots,d$.
Then, when $n=\Theta(\frac{1}{t^2}\log\frac{d}{\delta})$, we have, for all $i,j,k=1,\dots,d$, 
\begin{align*}
    \abs{\wh{\gamma}_{ij} - \gamma_{ij}} < t \frac{\sigma_{j\mid i,k}}{\sigma_{i\mid k}}, \qquad 
    \abs{\wh{\sigma}_{i\mid j}^2 - \sigma_{i\mid j}^2} < t \sigma_{i\mid j}^2 \qquad \text{and} \qquad
    \abs{\wh{\sigma}_{i\mid j,k}^2 - \sigma_{i\mid j,k}^2} < t \sigma_{i\mid j,k}^2
\end{align*}
with probability $1-\delta$.
\end{lemma}

\begin{proof}

By using Corollary \ref{cor:rhosigma} and the definition in \eqref{eq:emp_def}, it can be done by a straightforward calculation.
    
\end{proof}

\begin{restatable}[Conditional Mutual Information Tester]{theorem}{thmcondintest}
\label{thm:conditionalindependence}
Suppose $(X_1,\dots,X_d)$ is the random variable of $\N(0,\Sigma)$ for some positive definite $\Sigma$.
Let $(X_1^{(1)},\dots,X_d^{(1)}),\dots,(X_1^{(n)},\dots,X_d^{(n)})$ be the i.i.d. samples of $(X_1,\dots,X_d)$
For any sufficiently small $\eps,\delta >0$, if 
  \[
    n= \Theta(\frac{1}{\eps} \log\frac{d}{\delta}),
  \]
  the following results hold for all $i,j,k= 1,\dots,d$ with probability $1-\delta$:
  \begin{enumerate}
  \item If $I(X_i; X_j \mid X_k ) = 0$, then $\wh{I}(X_i;X_j \mid X_k) \leq \frac{\eps}{100}$.
  \item If $I(X_i;X_j \mid X_k) \ge \eps$, then $\wh{I}(X_i;X_j \mid X_k) > \frac{1}{20}I(X_i;X_j \mid X_k) - \frac{\eps}{40}$.
  \end{enumerate}
Combining these two cases, we have
\begin{align*}
    \wh{I}(X_i;X_j \mid X_k) > \frac{1}{20}I(X_i;X_j \mid X_k) - \frac{\eps}{40}.
\end{align*}

\end{restatable}

\begin{proof}

By Lemma \ref{lem:betasigmabound}, with $\Theta(\frac{1}{\eps}\log\frac{d}{\delta})$, we have the following properties for all $i,j,k=1,\dots,d$ with probability $1-\delta$:
\begin{align*}
    \abs{\wh{\gamma}_{ij} - \gamma_{ij}} < \frac{\sqrt{\eps}}{100}\frac{\sigma_{j\mid i,k}}{\sigma_{i\mid k}}, \qquad 
    \abs{\wh{\sigma}_{i\mid j}^2 - \sigma_{i\mid j}^2} < \frac{\sqrt{\eps}}{100}  \sigma_{i\mid j}^2 \qquad \text{and} \qquad
    \abs{\wh{\sigma}_{i\mid j,k}^2 - \sigma_{i\mid j,k}^2} < \frac{\sqrt{\eps}}{100}  \sigma_{i\mid j,k}^2 \numberthis\label{eq:conmui_betasigma}
\end{align*}

We express 
\begin{align*}
    \wh{I}(X_i;X_j \mid X_k)
    & =
    \frac{1}{2} \log\left(1 + \wh{\gamma}_{ij}^2 \frac{\wh{\sigma}_{i\mid k}^2}{\wh{\sigma}_{j\mid i,k}^2} \right)
    =
    \frac{1}{2} \log\left(1 + \wh{\gamma}_{ij}^2\frac{\sigma_{i\mid k}^2}{\sigma_{j\mid i,k}^2}\cdot \frac{\wh{\sigma}_{i\mid k}^2}{\sigma_{i\mid k}^2} \cdot \frac{\sigma_{j\mid i,k}^2}{\wh{\sigma}_{j\mid i,k}^2} \right) \numberthis\label{eq:condmui}
\end{align*}
We bound each term $\wh{\gamma}_{ij}^2\frac{\sigma_{i\mid k}^2}{\sigma_{j\mid i,k}^2}$, $\frac{\wh{\sigma}_{i\mid k}^2}{\sigma_{i\mid k}^2}$ and $\frac{\sigma_{j\mid i,k}^2}{\wh{\sigma}_{j\mid i,k}^2}$ for the cases of $I(X_i;X_j \mid X_k) = 0$ and $I(X_i;X_j \mid X_k)\geq \eps$.

We first prove if $I(X_i;X_j \mid X_k) = 0$ then $\wh{I}(X_i;X_j \mid X_k)\leq \frac{\eps}{100}$.
Since $I(X_i;X_j \mid X_k) = 0$, it means that $X_i$ and $X_j$ are independent conditioned on $X_k$ and hence $\gamma_{ij}=0$.
We have $\wh{\gamma}_{ij}^2 \frac{\sigma_{i\mid k}^2}{\sigma_{j\mid i,k}^2} \leq \frac{\eps}{100}$.
For the term $\frac{\wh{\sigma}_{i\mid k}^2}{\sigma_{i\mid k}^2}$, we have $\frac{\wh{\sigma}_{i\mid k}^2}{\sigma_{i\mid k}^2} \leq 1+\frac{\sqrt{\eps}}{100}$ by \eqref{eq:conmui_betasigma}.
For the term $\frac{\sigma_{j\mid i,k}^2}{\wh{\sigma}_{j\mid i,k}^2}$, we have $\frac{\sigma_{j\mid i,k}^2}{\wh{\sigma}_{j\mid i,k}^2} \leq \frac{1}{1-\frac{\sqrt{\eps}}{100}}$ by \eqref{eq:conmui_betasigma}.
Plugging these three inequalities into \eqref{eq:condmui}, we have
\begin{align*}
    \wh{I}(X_i;X_j \mid X_k)
    &= 
    \frac{1}{2} \log\left(1 + \wh{\gamma}_{ij}^2\frac{\sigma_{i\mid k}^2}{\sigma_{j\mid i,k}^2}\cdot \frac{\wh{\sigma}_{i\mid k}^2}{\sigma_{i\mid k}^2} \cdot \frac{\sigma_{j\mid i,k}^2}{\wh{\sigma}_{j\mid i,k}^2} \right) 
    \leq 
    \frac{1}{2} \log \left( 1 + \frac{\eps}{100}\cdot \frac{1 + \frac{\sqrt{\eps}}{100}}{ 1 -\frac{\sqrt{\eps}}{100}} \right) 
    \leq 
     \frac{\eps}{100}
\end{align*}
for any sufficiently small $\eps>0$.

We now prove if $I(X_i;X_j \mid X_k) \ge \eps$, then $\wh{I}(X_i;X_j \mid X_k) > \frac{1}{20}I(X_i;X_j \mid X_k) - \frac{\eps}{40}$.
Since $I(X_i;X_j \mid X_k) \ge \eps$, it means that $I(X_i;X_j \mid X_k)  = \frac{1}{2}\log(1 + \gamma_{ij}^2\frac{\sigma_{i\mid k}^2}{\sigma_{j\mid i,k}^2}) \ge \eps$ and hence $\gamma_{ij}^2\frac{\sigma_{i\mid k}^2}{\sigma_{j\mid i,k}^2} \ge e^{2\eps}-1 \ge 2\eps$.
We have $\wh{\gamma}_{ij}^2\frac{\sigma_{i\mid k}^2}{\sigma_{j\mid i,k}^2}  \ge \gamma_{ij}^2\frac{\sigma_{i\mid k}^2}{\sigma_{j\mid i,k}^2} - \sqrt{\frac{\eps}{100}} \ge  0$.
For the term $\frac{\wh{\sigma}_{i\mid k}^2}{\sigma_{i\mid k}^2}$, we have $\frac{\wh{\sigma}_{i\mid k}^2}{\sigma_{i\mid k}^2} \ge 1-\frac{\sqrt{\eps}}{100}$ by \eqref{eq:conmui_betasigma}.
For the term $\frac{\sigma_{j\mid i,k}^2}{\wh{\sigma}_{j\mid i,k}^2}$, we have $\frac{\sigma_{j\mid i,k}^2}{\wh{\sigma}_{j\mid i,k}^2} \ge \frac{1}{1+\frac{\sqrt{\eps}}{100}}$ by \eqref{eq:conmui_betasigma}.
Plugging these three inequalities into \eqref{eq:condmui}, we have
\begin{align*}
    \wh{I}(X_i;X_j \mid X_k)
    & =
    \frac{1}{2} \log\left(1 + \wh{\gamma}_{ij}^2\frac{\sigma_{i\mid k}^2}{\sigma_{j\mid i,k}^2}\cdot \frac{\wh{\sigma}_{i\mid k}^2}{\sigma_{i\mid k}^2} \cdot \frac{\sigma_{j\mid i,k}^2}{\wh{\sigma}_{j\mid i,k}^2} \right) 
    \geq
    \frac{1}{2} \log\left(1 + \bigg(\gamma_{ij}^2\frac{\sigma_{i\mid k}^2}{\sigma_{j\mid i,k}^2} - \sqrt{\frac{\eps}{100}}\bigg)^2 \cdot \frac{1-\frac{\sqrt{\eps}}{100}}{1+\frac{\sqrt{\eps}}{100}}\right).
\end{align*}
Note that, for any $a,b$, we have $(a-b)^2 \geq \frac{1}{2}a^2 - b^2$ which implies the term $\bigg(\gamma_{ij}^2\frac{\sigma_{i\mid k}^2}{\sigma_{j\mid i,k}^2} - \sqrt{\frac{\eps}{100}}\bigg)^2$ is larger than $\frac{1}{2}\gamma_{ij}^2\frac{\sigma_{i\mid k}^2}{\sigma_{j\mid i,k}^2} - \frac{\eps}{100}$.
Namely, we have
\begin{align*}
    \wh{I}(X_i;X_j \mid X_k)
    & \geq
    \frac{1}{2} \log\left(1 + \bigg(\frac{1}{2}\gamma_{ij}^2\frac{\sigma_{i\mid k}^2}{\sigma_{j\mid i,k}^2}  - \frac{\eps}{100}\bigg) \cdot \frac{1-\frac{\sqrt{\eps}}{100}}{1+\frac{\sqrt{\eps}}{100}}\right) \\
    &\geq 
    \frac{1}{2} \log\left(1 + \frac{1}{3}\gamma_{ij}^2\frac{\sigma_{i\mid k}^2}{\sigma_{j\mid i,k}^2}  - \frac{\eps}{100}\right) \\
    & \geq 
    \frac{1}{2} \log\left(1 + \frac{1}{3}\gamma_{ij}^2\frac{\sigma_{i\mid k}^2}{\sigma_{j\mid i,k}^2}\right) - \frac{\eps}{40}
\end{align*}
for any sufficiently small $\eps>0$.
Note that, for any $a>0$, $\log(1+\frac{1}{3}a) \geq \frac{1}{10}\log(1+a)$.
Namely, we have
\begin{align*}
    \wh{I}(X_i;X_j \mid X_k)
    & \geq
    \frac{1}{2} \log\left(1 + \frac{1}{3}\gamma_{ij}^2\frac{\sigma_{i\mid k}^2}{\sigma_{j\mid i,k}^2}\right) - \frac{\eps}{40}
    \geq 
    \frac{1}{20} \log\left(1 + \gamma_{ij}^2\frac{\sigma_{i\mid k}^2}{\sigma_{j\mid i,k}^2}\right) - \frac{\eps}{40}
    =
    \frac{1}{20}I(X_i;X_j \mid X_k) - \frac{\eps}{40}.\qedhere
\end{align*}
\end{proof}

\subsection{Distribution Learning Upper Bounds}
In this subsection, we give the formal proof of the upper bounds on the sample complexity for distribution learning in the non-realizable setting \cref{tm:nonrealizable} and realizable setting \cref{tm:realizable}:

\subsubsection{Non-realizable Case}
\thmupperboundnonrealizable*

\begin{proof}

Let $T^*$ be $\arg\min_{T\in\mathcal{T}} \kl{P}{P_T}$.
By \eqref{eq:kl_p_pt}, we express $\kl{P}{P_{\wh{T}}} - \kl{P}{P_{T^*}}$ as
\begin{align*}
    \kl{P}{P_{\wh{T}}} - \kl{P}{P_{T^*}} 
    & =
    -\sum_{(W,Z)\in \wh{T}} I(W;Z) + \sum_{(X,Y)\in T^*} I(X;Y)
\end{align*}
Since $\wh{T}$ is the output of \cref{algo:chow-liu}, we have
\begin{align*}
    \sum_{(X,Y)\in T^*} \wh{I}(X;Y) - \sum_{(W,Z)\in \wh{T}} \wh{I}(W;Z)
    & \leq
    0.
\end{align*}
Hence, we have
\begin{align*}
    \MoveEqLeft \kl{P}{P_{\wh{T}}} - \kl{P}{P_{T^*}}  \\
    & \leq
    \sum_{(W,Z)\in \wh{T}} \wh{I}(W;Z)-\sum_{(W,Z)\in \wh{T}} I(W;Z) + \sum_{(X,Y)\in T^*} I(X;Y) - \sum_{(X,Y)\in T^*} \wh{I}(X;Y)
\end{align*}
By the definition in \eqref{eq:emp_mi} and Corollary \ref{cor:rhosigma}, we can show that each $\abs{\wh{I}(X,Y) - I(X,Y)} < \frac{\eps}{d}$ for all $(X,Y)$ using  ${O}(\frac{d^2}{\eps^2} \log \frac{d}{\delta})$ samples.
Therefore, we have
\begin{align*}
    \kl{P}{P_{\wh{T}}} - \kl{P}{P_{T^*}} 
    <
    \eps.
\end{align*}

\end{proof}

\subsubsection{Realizable Case}

\begin{fact}[\citep{bhattacharyya2021near}]
\label{fact:spanningtree}
Let $T_1$ and $T_2$ be two spanning trees on $d$ vertices such that their symmetric difference consists of the edges $E = \{e_1, e_2, \dots, e_l\} \in T_1 \setminus T_2$ and $F = \{f_1, f_2, \dots, f_l\} \in T_2 \setminus T_1$. Then $E$ and $F$ can be paired up, say $
\langle e_i, f_i\rangle$, such that for all $i$, $T_1 \cup\{f_i\} \setminus \{e_i\} $ is a spanning tree.
\end{fact}

\thmupperrealizabe*

\begin{proof}

We first consider the edge difference between $\wh{T}$ and $T^*$.
By \cref{fact:spanningtree}, we can pair up the edges in $\wh{T}\backslash T^*$ with the edges in $T^* \backslash \wh{T}$ such that $T^*\cup \{(W,Z)\} \backslash \{(X,Y)\}$ is also a spanning tree for any $(W,Z) \in \wh{T}\backslash T^*$ and $(X,Y) \in T^* \backslash \wh{T}$.
Let $\wh{T}\backslash T^*$ be $\{(W_1,Z_1),\dots,(W_k,Z_k)\}$ and $T^*\backslash\wh{T}$ be $\{(X_1,Y_1),\dots,(X_k,Y_k)\}$ such that $(W_i,Z_i)$ pairs up with $(X_i,Y_i)$ for $i=1,\dots,k$.
Because of that, there exists a path in $T^*$ from $W_i$ to $Z_i$ containing $X_i$ and $Y_i$.
Without loss of generality, we assume that the order of them is $W_i\leadsto X_i \mathdash Y_i \leadsto  Z_i$ in $T^*$.

Since $\wh{T}$ is the output of \cref{algo:chow-liu}, we have 
\begin{align*}
    \sum_{i=1}^k \wh{I}(X_i;Y_i) - \sum_{i=1}^k \wh{I}(W_i;Z_i)
    & \leq
    0
\end{align*}
by the definition of the maximal spanning tree.
We first expand the LHS as
\begin{align*}
    \MoveEqLeft \sum_{i=1}^k \wh{I}(X_i,Y_i) - \sum_{i=1}^k \wh{I}(W_i,Z_i)
    =
    \sum_{i=1}^k \left( \wh{I}(X_i,Y_i) - \wh{I}(X_i;Z_i) + \wh{I}(X_i;Z_i) - \wh{I}(W_i;Z_i) \right)\\
    & =
    \sum_{i=1}^k \left(\wh{I}(X_i; Y_i \mid Z_i) -  \wh{I}(X_i; Z_i \mid Y_i) +  \wh{I}(X_i; Z_i \mid W_i) - \wh{I}(W_i; Z_i \mid X_i) \right)  \qquad \text{by \eqref{eq:conmui_chain_emp}}\\ 
    & =  
    \underbrace{\sum_{i=1}^k \left(\wh{I}(X_i; Y_i \mid Z_i) +  \wh{I}(X_i; Z_i \mid W_i) \right)}_{:=A} - \underbrace{\sum_{i=1}^k \left( \wh{I}(X_i; Z_i \mid Y_i) + \wh{I}(W_i; Z_i \mid X_i)\right)}_{:=B}.
\end{align*}
In other words, we have $A \leq B$.

Recall that there exists a path $W_i\leadsto X_i \mathdash Y_i \leadsto  Z_i$ in $T^*$ and hence $(X_i,Z_i)\notin T^*$ which further implies $I(X_i; Z_i \mid Y_i) = 0$.
Similarly, we have $I(W_i; Z_i \mid X_i) = 0$.
By \cref{thm:conditionalindependence} with $\Theta(\frac{1}{\eps'}\log\frac{d}{\delta})$ samples, we have
\begin{align*}
    \wh{I}(X_i; Z_i \mid Y_i) \leq \eps'/100
    \qquad \text{and} \qquad
    \wh{I}(W_i;Z_i \mid X_i) \leq \eps'/100 \qquad \text{for all $i=1,\dots,k$.}
\end{align*}
Plugging them into each term in $B$, we can bound $B$ by $2k\cdot \eps'/100 \leq d\eps'/50$.
Namely, we have
\begin{align*}
    A
    & =
    \sum_{i=1}^k \left(\wh{I}(X_i; Y_i \mid Z_i) +  \wh{I}(X_i; Z_i \mid W_i) \right)
    \leq d\eps'/50.
\end{align*}
By \cref{thm:conditionalindependence} with $\Theta(\frac{1}{\eps'}\log\frac{d}{\delta})$ samples, we have 
\begin{align*}
    \frac{1}{20}I(X_i;Y_i \mid Z_i) - \frac{\eps'}{40} \leq \wh{I}(X_i; Y_i \mid Z_i)
    \qquad \text{and} \qquad
    \frac{1}{20}I(X_i;Z_i \mid W_i) - \frac{\eps'}{40} \leq \wh{I}(X_i; Z_i \mid W_i)
\end{align*}
for all $i=1,\dots,k$.
In other words, 
\begin{align*}
    A
    & \geq 
    \frac{1}{20}\sum_{i=1}^k \left(I(X_i; Y_i \mid Z_i) +  I(X_i; Z_i \mid W_i) \right) - \frac{d\eps'}{40}
\end{align*}
or 
\begin{align*}
    \sum_{i=1}^k \left(I(X_i; Y_i \mid Z_i) +  I(X_i; Z_i \mid W_i) \right)
    & \leq
    \frac{9d\eps'}{10} \numberthis\label{eq:muti_sum}
\end{align*}

Now, we can bound $\kl{P_{T^*}}{P_{\wh{T}}}$.
We express it as
\begin{align*}
    \kl{P_{T^*}}{P_{\wh{T}}} 
    & = 
    \sum_{i=1}^k I(X_i;Y_i) - \sum_{i=1}^k I(W_i;Z_i)
    =
    \sum_{i=1}^k \left(I(X_i; Y_i) - I(X_i; Z_i) + I(X_i; Z_i) - I(W_i; Z_i) \right)\\
    &= 
    \sum_{i=1}^k \left(I(X_i; Y_i \mid Z_i) -  I(X_i; Z_i \mid Y_i) +  I(X_i; Z_i \mid W_i) - I(W_i; Z_i \mid X_i)\right) \qquad \text{by \eqref{eq:conmui_chain}}
\end{align*}
Recall that we have $I(X_i; Z_i \mid Y_i) = 0$ and $I(W_i; Z_i \mid X_i) = 0$.
Combining with \eqref{eq:muti_sum}, we have
\begin{align*}
    \kl{P_{T^*}}{P_{\wh{T}}} 
    & \leq 
    \frac{9d\eps'}{10}
\end{align*}
with probability at least $1-\delta$.
By picking $\eps' = \frac{10\eps}{9d}$, we conclude our result.
\end{proof}

\subsection{Distribution Learning Lower Bounds}

To show the lower bounds, our main idea is to reduce \cref{prob:reduction} defined below to our problem.

\begin{problem}\label{prob:reduction}
Suppose $R^{(1)}$ and $R^{(2)}$ are two distributions such that $\kl{R^{(1)}}{R^{(2)}} \leq \delta$.
Let $P$ be a distribution on $m$ variables where each variable is distributed as either $R^{(1)}$ or $R^{(2)}$ uniformly and independently.
We are given $n$ i.i.d. samples drawn from a distribution $P$.
Our task is to determine which distribution the samples are drawn from correctly for at least $51m/100$ variables.
Formally, we define 
\begin{align*}
    \mathcal{R}
    :=
    \{(R_1,\dots,R_m)\mid R_i\in\{R^{(1)},R^{(2)}\}\}.
\end{align*}
We pick a distribution uniformly from $\mathcal{R}$ and let $P=(R^*_1,\dots,R^*_m)$ be this distribution.
Then, our goal is to design an algorithm that takes $n$ i.i.d. samples drawn from $P$ as input and returns $(\wh{R}_1,\dots,\wh{R}_m)$ such that $\wh{R}_i = R^*_i$ for at least $51m/100$ of $\{1,\dots,m\}$.

\end{problem}

\begin{fact}\label{fact:standard_lb}
    By the standard information-theoretic lower bounds, if $n = o(\frac{1}{\delta})$, then no algorithm can solve \cref{prob:reduction} with probability $2/3$.
\end{fact}

\subsubsection{Non-realizable Case}

We define two distributions $Q^{(1)},Q^{(2)}$ as follows.
\begin{align*}
    Q^{(1)}
    =
    \begin{cases}
        H \sim \N(0,1) \\
        X \sim (1+\eps)H + \N(0,1) \\
        Y \sim H + \N(0,1) \\
        Z \sim H + \N(0,1)
    \end{cases}
    \qquad \text{and} \qquad
    Q^{(2)}
    =
    \begin{cases}
        H \sim \N(0,1) \\
        X \sim H + \N(0,1) \\
        Y \sim (1+\eps)H + \N(0,1) \\
        Z \sim H + \N(0,1)
    \end{cases} 
    \numberthis\label{eq:nonreal_threenodes}
\end{align*}
Also, we define $R^{(1)},R^{(2)}$ to be the corresponding marginal distributions on $(X,Y,Z)$.

\begin{figure}  
\centering
\subfigure[$R^{(1)}$]
{  
    \begin{tikzpicture}[->,>=stealth',node distance=1cm, auto]
    \node[est, label = {$H + \N(0,1)$}] (Y) {$Y$};
    \node[est, right = of Y, dashed, label = {[xshift=1.2cm, yshift=-0.7cm]$\N(0,1)$}, fill=lightgray] (H) {$H$};
    \node[est, above = of H,  label = {$H + \N(0,1)$}] (Z) {$Z$};
    \node[est, below = of H, label = {[xshift=-2.0cm, yshift=-0.7cm]$(1 + \eps)H + \N(0, 1)$}] (X) {$X$};
    \path[pil] (H) edgenode {} (Z);
    \path[pil] (H) edge node {} (Y);
    \path[pil] (H) edgenode {} (X);
    \path[pil, densely dashed] (X) edge [bend left=40] node [left]  {} (Y);
    \path[pil, densely dashed] (X) edge [bend right=80] node [right]  {} (Z);
    \end{tikzpicture}  
} 
\hspace{1cm}
\subfigure[$R^{(2)}$]  
{  
    \begin{tikzpicture}[->,>=stealth',node distance=1cm, auto]
    \node[est, label = {[xshift=-2.0cm, yshift=-0.7cm]$(1+\eps)H + \N(0,1)$}] (Y) {$Y$};
    \node[est, right = of Y, dashed, label = {[xshift=1.2cm, yshift=-0.7cm]$\N(0,1)$}, fill=lightgray] (H) {$H$};
    \node[est, above = of H,  label = {$H + \N(0,1)$}] (Z) {$Z$};
    \node[est, below = of H, label = {[xshift=-1.7cm, yshift=-0.7cm]$H + \N(0,1)$}] (X) {$X$};
    \path[pil] (H) edgenode {} (Z);
    \path[pil] (H) edge node {} (Y);
    \path[pil] (H) edgenode {} (X);
    \path[pil, densely dashed] (Y) edge [bend left=40] node [left]  {} (Z);
    \path[pil, densely dashed] (Y) edge [bend right=40] node [right]  {} (X);
    \end{tikzpicture}  
} 

\caption{The $\Omega(1/\eps^2)$ bound in the non-realizable setting. The underlying graph is represented with solid lines, while the best estimated tree structure is depicted with dashed lines.}
\label{fig:nonrealizable}
\end{figure}
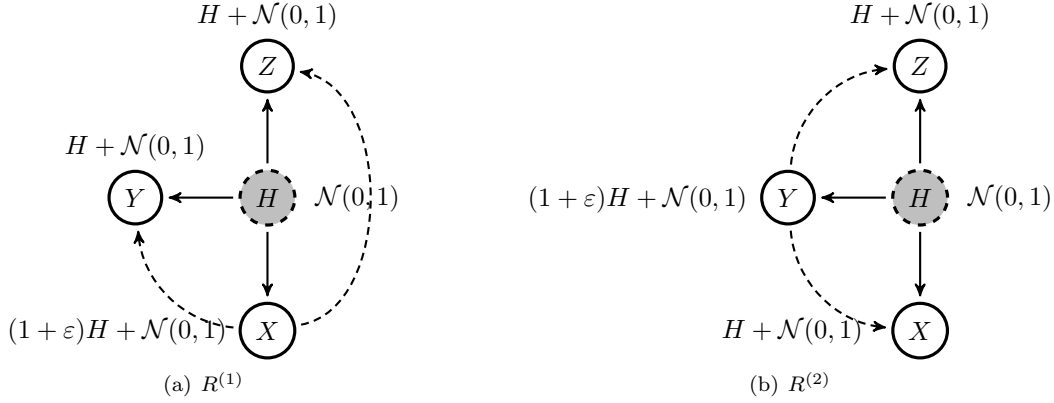

\begin{lemma}\label{lem:nonreal_three}

Suppose $R^*$ is one of $R^{(1)}$ and $R^{(2)}$ defined in \eqref{eq:nonreal_threenodes}.
For any small $\eps>0$, if a direct tree $\wh{T}$ satisfies
\begin{align*}
    \kl{R^*}{R^*_{\wh{T}}}
    \leq 
    \min_T \kl{R^*}{R^*_T} + \frac{\eps}{100} 
    \numberthis\label{eq:nonreal_c}
\end{align*}
and 
$\wh{R}=\arg\min_{R\in\{R^{(1)},R^{(2)}\}} \kl{R}{R_{\wh{T}}}$, then $\wh{R} = R^*$.

\end{lemma}

\begin{proof}

Since there are three variables, there are only three possible tree structures: $T_1=Y-X-Z$, $T_2=X-Y-Z$ and $T_3=X-Z-Y$.
Recall that, by \eqref{eq:kl_p_pt}, we have
\begin{align*}
    \kl{R^{(1)}}{R^{(1)}_{T_2}} - \kl{R^{(1)}}{R^{(1)}_{T_1}}
    & =
    I(X;Z) - I(Y;Z)
    \geq \frac{\eps}{50} \qquad \text{by a straightforward calculation} \numberthis\label{eq:nonreal_21}
\end{align*}
and, similarly, we also have
\begin{align*}
    \kl{R^{(1)}}{R^{(1)}_{T_3}} - \kl{R^{(1)}}{R^{(1)}_{T_1}}
    \geq 
    \frac{\eps}{50} \numberthis\label{eq:nonreal_31}\\
    \kl{R^{(2)}}{R^{(2)}_{T_1}} - \kl{R^{(2)}}{R^{(2)}_{T_2}}
    \geq 
    \frac{\eps}{50} \numberthis\label{eq:nonreal_12}\\
    \kl{R^{(2)}}{R^{(2)}_{T_3}} - \kl{R^{(2)}}{R^{(2)}_{T_2}}
    \geq 
    \frac{\eps}{50}\numberthis\label{eq:nonreal_32}
\end{align*}

By \eqref{eq:nonreal_c}, \eqref{eq:nonreal_31} and \eqref{eq:nonreal_32}, we have $\wh{T}\neq T_3$.
Namely, $\wh{T}$ is either $T_1$ or $T_2$ (WLOG, say $T_1$).
By \eqref{eq:nonreal_c} and \eqref{eq:nonreal_12}, we have $R^* = R^{(1)}$.
By \eqref{eq:nonreal_21}, we have
\begin{align*}
    \kl{R^{(1)}}{R^{(1)}_{T_1}}
    & \leq
    \kl{R^{(1)}}{R^{(1)}_{T_2}} - \frac{\eps}{50}
    <
    \underbrace{\kl{R^{(1)}}{R^{(1)}_{T_2}}
    =
    \kl{R^{(2)}}{R^{(2)}_{T_1}}}_{\text{by symmetry}}.
\end{align*}
Hence, $\wh{R} = R^{(1)} = R^*$ by the definition of $\wh{R}$.

\end{proof}

\thmnonrealizablelowerbound*

\begin{proof}

We will prove the statement by reducing \cref{prob:reduction} to our problem.
We first split the $d$ variables into $m=d/3$ groups of $3$ variables and for each group we select $R^{(1)}$ or $R^{(2)}$ defined in \eqref{eq:nonreal_threenodes} (replacing $\eps$ with $\eps/d$) uniformly and independently and notice that $\kl{R^{(1)}}{R^{(2)}} = O(\eps^2/d^2)$ by a straightforward calculation.
By \cref{fact:standard_lb}, it implies that if $n = o(\frac{d^2}{\eps^2})$ then no algorithm can determine which distribution the samples are drawn from correctly for at least $51m/100$ groups with probability $\frac{2}{3}$.

Suppose there is an algorithm that takes these $n$ i.i.d. samples as input and returns a directed tree $\wh{T}$ such that
\begin{align*}
   \kl{P}{P_{\wh{T}}} \leq \min_{T\in\drtr} \kl{P}{P_{T}} + \eps \numberthis \label{eq:nonreal_d_c}
\end{align*}
with probability $\frac{2}{3}$.
If we manage to show that we can use $\wh{T}$ to determine which distribution the samples are drawn from correctly for $51m/100$ groups then it implies $n=\Omega(\frac{d^2}{\eps^2})$.

We construct the reduction as follows.
For the $i$-th group of variables, we consider its subtree $\wh{T}_i$ of $\wh{T}$ and declare $\wh{R}_i$ to be the distribution for this group where $\wh{R}_i$ is defined to be $\arg\min_{R\in\{R^{(1)},R^{(2)}\}} \kl{R}{R_{\wh{T}_i}}$.
To see the correctness, we have the following.
Since each group is independent, \eqref{eq:nonreal_d_c} can be decomposed into 
\begin{align*}
    \sum_{i=1}^m \kl{P_i}{(P_i)_{\wh{T}_i}}
    \leq
    \sum_{i=1}^m \min_{T_i}\kl{P_i}{(P_i)_{T_i}} + \eps
\end{align*}
where $P_i$ is the random pick of $R^{(1)}$ or $R^{(2)}$ for the $i$-th group.
Therefore, at least $51m/100$ of the terms $\kl{P_i}{(P_i)_{\wh{T}_i}} - \min_{T_i}\kl{P_i}{(P_i)_{T_i}} \leq \frac{10\eps}{m}$.
By Lemma \ref{lem:nonreal_three}, for these $51m/100$ groups, $\wh{R}_i$ is correctly determined, i.e. $\wh{R}_i = P_i$ and hence the reduction is completed.
\end{proof}

\subsubsection{Realizable Case}

We define two distributions $R^{(1)}, R^{(2)}$ as follows.
\begin{align*}
    R^{(1)}
    =
    \begin{cases}
        X \sim \N(0,1) \\
        Y \sim (1-\sqrt{\eps})X+\sqrt{\eps}\N(0,1) \\
        Z \sim \frac{1}{2}X+\frac{1}{2}\N(0,1)
    \end{cases}
    \qquad \text{and}\qquad
    R^{(2)}
    =
    \begin{cases}
        X \sim \N(0,1) \\
        Y \sim (1-\sqrt{\eps})X+\sqrt{\eps}\N(0,1) \\
        Z \sim \frac{1}{2}Y+\frac{1}{2}\N(0,1)
    \end{cases}
    \numberthis\label{eq:real_threenodes}
\end{align*}
Namely, the underlying graph for $R^{(1)}$ is $Y <- X -> Z$ and the underlying graph for $R^{(2)}$ is $X-> Y->Z$.
Both have $X->Y$ and the only difference is $Z$.
\begin{figure*} 
\centering
\subfigure[$R^{(1)}$]  
{ 
\centering
    \begin{tikzpicture}[->,>=stealth',node distance=1cm, auto,]
    \node[est, label = {[xshift=0cm, yshift=0cm]$\N(0, 1)$}] (X) {$X$};
    \node[est, below right = of X, label = {[xshift=1cm, yshift=-1.4cm]$(1-\sqrt{\eps})X+\sqrt{\eps}\N(0, 1)$}] (Y) {$Y$};
    \node[est, below left = of X,  label = {[xshift=0cm, yshift=-1.4cm]$\frac{1}{2}X+\frac{1}{2}\N(0, 1)$}] (Z) {$Z$};
    \path[pil] (X) edgenode {} (Z);
    \path[pil] (X) edgenode {} (Y);
    \end{tikzpicture}  
}
\hspace{1cm}
\subfigure[$R^{(2)}$]  
{ 
\centering
    \begin{tikzpicture}[->,>=stealth',node distance=1cm, auto,]
    \node[est, label = {[xshift=0cm, yshift=0cm]$\N(0, 1)$}] (X) {$X$};
    \node[est, below right = of X, label = {[xshift=1cm, yshift=-1.4cm]$(1-\sqrt{\eps})X+\sqrt{\eps}\N(0,1)$}] (Y) {$Y$};
    \node[est, below left = of X,  label = {[xshift=0cm, yshift=-1.4cm]$\frac{1}{2}Y+\frac{1}{2}\N(0, 1)$}] (Z) {$Z$};
    \path[pil] (Y) edgenode {} (Z);
    \path[pil] (X) edgenode {} (Y);
    \end{tikzpicture}  
}
\caption{Realizable setting} 
\label{fig:realizable}
\end{figure*}
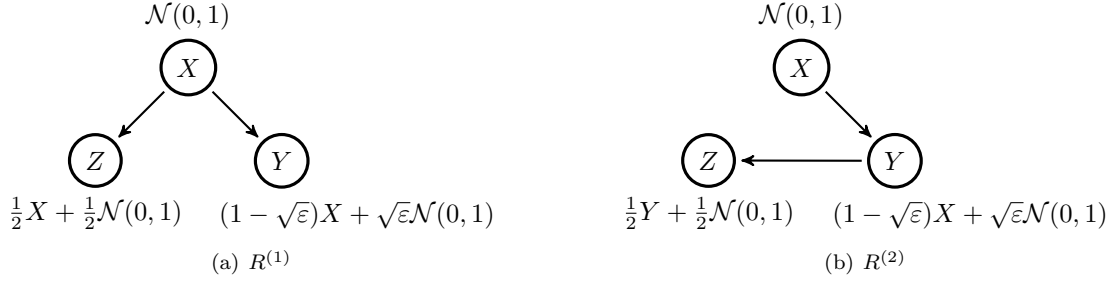

\begin{lemma}\label{lem:real_three}
Suppose $R^*$ is one of $R^{(1)}$ and $R^{(2)}$ defined in \eqref{eq:real_threenodes}.
For any small $\eps>0$, if a direct tree $\wh{T}$ satisfies
\begin{align*}
    \kl{R^*}{R^*_{\wh{T}}}
    \leq 
    \frac{\eps}{100} \numberthis\label{eq:real_c}
\end{align*}
and $\wh{R} = \arg\min_{R\in\{R^{(1)},R^{(2)}\}}\kl{R}{R_{\wh{T}}}$, then $\wh{R}=R^*$.
\end{lemma}

\begin{proof}

Since there are three variables, there are only three possible tree structures: $T_1=Y-X-Z$, $T_2=X-Y-Z$ and $T_3=X-Z-Y$.
Recall that, by \eqref{eq:kl_p_pt}, we have
\begin{align*}
    \kl{R^{(1)}}{R^{(1)}_{T_2}} - \kl{R^{(1)}}{R^{(1)}_{T_1}}
    & =
    I(X;Z) - I(Y;Z)
    \geq \frac{\eps}{50} \qquad \text{by a straightforward calculation.}
\end{align*}
Note that $\kl{R^{(1)}}{R^{(1)}_{T_1}}=0$ and hence
\begin{align*}
    \kl{R^{(1)}}{R^{(1)}_{T_2}}
    \geq 
    \frac{\eps}{50} \numberthis\label{eq:real_21}
\end{align*}
Similarly, we also have
\begin{align*}
    \kl{R^{(1)}}{R^{(1)}_{T_3}}
    &\geq 
    \Omega(1)
    \geq
    \frac{\eps}{50} \numberthis\label{eq:real_31}\\
    \kl{R^{(2)}}{R^{(2)}_{T_1}} 
    &\geq 
    \frac{\eps}{50} \numberthis\label{eq:real_12}\\
    \kl{R^{(2)}}{R^{(2)}_{T_3}}
    &\geq
    \Omega(1)
    \geq 
    \frac{\eps}{50}\numberthis\label{eq:real_32}
\end{align*}

By \eqref{eq:real_c}, \eqref{eq:real_31} and \eqref{eq:real_32}, we have $\wh{T}\neq T_3$.
Namely, $\wh{T}$ is either $T_1$ or $T_2$.
If $\wh{T}=T_1$, by \eqref{eq:real_c} and \eqref{eq:real_12}, we have
\begin{align*}
    \kl{R^{(2)}}{R^{(2)}_{\wh{T}}} > \kl{R^*}{R^*_{\wh{T}}}
\end{align*}
and hence $R^* = R^{(1)}$.
If $\wh{T} = T_2$, by \eqref{eq:real_c} and \eqref{eq:real_21}, we have
\begin{align*}
    \kl{R^{(1)}}{R^{(1)}_{\wh{T}}} > \kl{R^*}{R^*_{\wh{T}}}
\end{align*}
and hence $R^*=R^{(2)}$.
By the definition of $\wh{R}$, both cases imply $\wh{R} = R^*$.

\end{proof}

\thmrealizablelowerbound*

\begin{proof}

We will prove the statement by reducing \cref{prob:reduction} to our problem.
We first split the $d$ variables into $m=d/3$ groups of $3$ variables and for each group we select $R^{(1)}$ or $R^{(2)}$ defined in \eqref{eq:real_threenodes} (replacing $\eps$ with $\eps/d$) uniformly and independently and notice that $\kl{R^{(1)}}{R^{(2)}} = O(\eps/d)$ by a straightforward calculation.
By \cref{fact:standard_lb}, it implies that if $n = o(\frac{d}{\eps})$ then no algorithm can determine which distribution the samples are drawn from correctly for at least $51m/100$ groups with probability $\frac{2}{3}$.

Suppose there is an algorithm that takes these $n$ i.i.d. samples as input and returns a directed tree $\wh{T}$ such that
\begin{align*}
   \kl{P}{P_{\wh{T}}} \leq \eps \numberthis \label{eq:real_d_c}
\end{align*}
with probability $\frac{2}{3}$.
If we manage to show that we can use $\wh{T}$ to determine which distribution the samples are drawn from correctly for $51m/100$ groups then it implies $n=\Omega(\frac{d}{\eps})$.

We construct the reduction as follows.
For the $i$-th group of variables, we consider its subtree $\wh{T}_i$ of $\wh{T}$ and declare $\wh{R}_i$ to be the distribution for this group where $\wh{R}_i$ is defined to be $\arg\min_{R\in\{R^{(1)},R^{(2)}\}} \kl{R}{R_{\wh{T}_i}}$.
To see the correctness, we have the following.
Since each group is independent, \eqref{eq:real_d_c} can be decomposed into 
\begin{align*}
    \sum_{i=1}^m \kl{P_i}{(P_i)_{\wh{T}_i}}
    \leq
    \eps
\end{align*}
where $P_i$ is the random pick of $R^{(1)}$ or $R^{(2)}$ for the $i$-th group.
Therefore, at least $51m/100$ of the terms $\kl{P_i}{(P_i)_{\wh{T}_i}}\leq \frac{10\eps}{m}$.
By Lemma \ref{lem:real_three}, for these $51m/100$ groups, $\wh{R}_i$ is correctly determined, i.e. $\wh{R}_i = P_i$ and hence the reduction is completed.
    
\end{proof}

\subsection{Learning Polytrees given Skeleton}
\label{sec:polytree_distlearning}
In this section, we sketch how to obtain a sample-efficient algorithm for learning bounded-degree gaussian {\em polytrees} 
by adapting the recent results from \citep{choo2023learning}, using the guarantees of the estimator $\wh{I}$, assuming that the skeleton is known. Let a {\em $m$-polytree} denote a polytree with maximum in-degree $m$. Our main result in this section is the following:
\begin{theorem}\label{thm:distlearn_polytrees}
There exists an algorithm which, given $n$ samples from a gaussian $m$-polytree $P$ over $\R^d$, accuracy parameter $\eps > 0$, failure probability $\delta$, maximum in-degree $m$, and the explicit description of the ground truth skeleton of $P$, outputs a $m$-polytree $\wh{P}$ such that $\dkl(P \| \wh{P}) \leq \eps$ with success probability at least $1-\delta$, as long as:
\[
n \geq \widetilde{O}\left(\frac{d}{\eps}\log\frac1\delta \right).
\]
Moreover, the algorithm runs in time polynomial in $n$ and $d$.
\end{theorem}
\noindent Note that the guarantee in \cref{thm:distlearn_polytrees} is entirely independent of any faithfulness parameter, in contrast to \cref{thm:tree:ub}.
The algorithm and its analysis is exactly the same as in \citep{choo2023learning}, with the only change being that we use \eqref{eq:emp_mi}
for the estimator $\wh{I}$.

\section{Proofs of \cref{sec:faithfultree}}

\begin{algorithm}[t]
\caption{\textsc{Orient} algorithm}\label{alg:orient}
\textbf{Input:} Skeleton $\widehat{T}$, separation sets $S$\\
\textbf{Output:} CPDAG $\widehat{\mec}$.
\begin{enumerate}
    \item For all pairs of nonadjacent nodes $j,k$ with common neighbour $\ell$:
    \begin{enumerate}
        \item If $\ell\not\in S(j,k)$, then directize $j-\ell-k$ in $\widehat{T}$ by $j\rightarrow\ell\leftarrow k$
    \end{enumerate}
    \item In the resulting PDAG $\widehat{T}$, orient as many as possible undirected edges by applying following rules:
    \begin{itemize}
        \item \textbf{R1} Orient $k-\ell$ into $k\rightarrow \ell$ whenever there is an arrow $j\rightarrow k$ such that $j$ and $\ell$ are not adjacent
        \item \textbf{R2} Orient $j-k$ into $j\rightarrow k$ whenever there is a chain $j\rightarrow \ell \rightarrow k$
        \item \textbf{R3} Orient $j-k$ into $j\rightarrow k$ whenever there are two chains $j- \ell\rightarrow k$ and $j - i \rightarrow k$ such that $\ell$ and $i$ are not adjacent
        \item \textbf{R4} Orient $j-k$ into $j\rightarrow k$ whenever there are two chains $j- \ell\rightarrow i$ and $\ell - i \rightarrow k$ such that $\ell$ and $i$ are not adjacent
    \end{itemize}
    \item Return $\widehat{T}$ as $\widehat{\mec}$.
\end{enumerate}
\end{algorithm}

\subsection{Sample Conditional Correlation Coefficient as CI Tester}\label{app:tree:samplecorr}
\texttt{PC-Tree} relies on sample (conditional) correlation coefficient as (conditional) independence tester. Specifically, denote the sample covariance matrix to be $\widehat{\Sigma} = \frac{1}{n}\sum_{i=1}^n X^{(i)}{X^{(i)}}\T$, for any two nodes $j,k\in V$ and any subset $S\subseteq V\setminus \{j,k\}$, which could be $\emptyset$, the sample correlation coefficient is defined by
\begin{align*}
    \widehat{\rho}_{jk\given S} : = \frac{\widehat{\Sigma}_{jk} - \widehat{\Sigma}_{jS}\widehat{\Sigma}_{SS}^{-1}\widehat{\Sigma}_{Sk}}{\sqrt{(\widehat{\Sigma}_{jj}-\widehat{\Sigma}_{jS}\widehat{\Sigma}_{SS}^{-1}\widehat{\Sigma}_{Sj})(\widehat{\Sigma}_{kk} - \widehat{\Sigma}_{kS}\widehat{\Sigma}_{SS}^{-1}\widehat{\Sigma}_{Sk})}}
\end{align*}
Then the conditional independence tester for hypothesis $H_0:X_j\indep X_k\given X_S$ is given by a cutoff on the sample correlation coefficient:
\begin{align}\label{eq:cit}
    \text{Output} = 
    \begin{cases}
        \text{accept }H_0 & \text{if }|\widehat{\rho}_{jk\given S}| \ge c/2 \\
        \text{reject }H_0 & \text{if }|\widehat{\rho}_{jk\given S}| < c/2 
    \end{cases} \,.
\end{align}
Here the choice of $c/2$ is for theoretical purpose.
Since correlation coefficient is normalized between $[-1,1]$, in practice, the tester can be implemented by choosing a cutoff that is small enough, e.g. 0.05.
The analysis of \texttt{PC-Tree} crucially relies on the following lemma on the estimation error of the sample (conditional) correlation coefficients:
\begin{restatable}{lemma}{lemtreecit}
\label{lem:cit}
Let $X \in\mathbb{R}^d \sim \mathcal{N}(0,\Sigma)$, for any $j,k\in V$ and any subset $S\subseteq V\setminus \{j,k\}$ with $|S|\le q$, $\delta\in(0,1)$, if $n\gtrsim q + 1/\delta^2$, then
\begin{align*}
    \prob(|\widehat{\rho}_{jk\given S} - \rho_{jk\given S}|\ge \delta) \le \exp(-C_0(n-q)\delta^2)\, ,
\end{align*}
for some universal constant $C_0>0$.
\end{restatable}
It is clear to see that as long as the (conditional) correlation coefficients are estimated accurately enough, the CI tests are correct due to $c$-strong Tree-faithfulness. \cref{lem:cit} is more general than needed to analyze \texttt{PC-Tree} algorithm. Since \cref{lem:cit} reveals the dependence on the size of conditioning set $S$, while \texttt{PC-Tree} only requires $|S|\le 1$.

\subsection{Proof of \cref{lem:cit}}\label{app:tree:cit}
\lemtreecit*
\begin{proof}
    The proof is a combination of the following lemmas.
    We start with analyzing sample marginal correlation of bivariate normal distribution, then extend to conditional correlation.
    \begin{lemma}\label{lem:pc:testFor2}
    Let $W = (X,Y)\sim \mathcal{N}(0,\Sigma)$ where $\Sigma = \begin{pmatrix}
    \sigma_X^2 & \sigma_{XY} \\ \sigma_{XY} & \sigma_Y^2
    \end{pmatrix}\in\mathbb{R}^{2\times 2}$, and $\rho = \frac{\sigma_{XY}}{\sigma_X\sigma_Y}$. Let the sample covariance matrix and correlation be
    \begin{align*}
    \frac{1}{n}\sum_{\ell=1}^n w^{(\ell)}{w^{(\ell)}}\T = \begin{pmatrix}
    \widehat{\sigma}_X^2 & \widehat{\sigma}_{XY} \\ \widehat{\sigma}_{XY} & \widehat{\sigma}^2_Y
    \end{pmatrix} \,, \quad \text{ and } \quad \widehat{\rho} = \frac{\widehat{\sigma}_{XY}}{\widehat{\sigma}_X\widehat{\sigma}_Y} \,.
    \end{align*}
    For $\delta \in (0,1)$, if $n\gtrsim 1/\delta^2$, then
    \begin{align*}
        \prob(|\widehat{\rho} - \rho| \ge \delta) \le \exp(-C_0n\delta^2)\, ,
    \end{align*}
    for some constant $C_0>0$.
    \end{lemma}
    Now look at sample conditional correlation, suppose we want to estimate $\rho_{jk\given S}$ with $|S|=q'\le q$. Recall the sample covariance matrix is $\widehat{\Sigma} = \frac{1}{n}\sum_{i=1}^n X^{(i)}{X^{(i)}}\T$. Denote $I=\{j,k\}$, then the estimator is given by $2\times 2$ matrix 
    \begin{align*}
        \widehat{\Sigma}_{II\given S} := \widehat{\Sigma}_{II} - \widehat{\Sigma}_{II,S}\widehat{\Sigma}_{SS}^{-1}\widehat{\Sigma}_{S,II}\,.
    \end{align*} 
    We borrow a classic result regarding the distribution of $\widehat{\Sigma}_{II\given S}$:
    \begin{lemma}[\citep{anderson1958introduction}, Theorem 4.3.4]\label{lem:pc:classic}
    The sample covariance matrix $\widehat{\Sigma}_{II\given S}$ is distributed as $\frac{1}{n}\sum_{\ell=1}^{n-q'}u^{(\ell)}{u^{(\ell)}}\T$, where $\{u^{(\ell)}\}_{\ell=1}^{n-q'}$ are independently distributed according to $\mathcal{N}(0,\Sigma_{II\given S})$.
    \end{lemma}
    Then applying the bivariate result from \cref{lem:pc:testFor2} with covariance matrix $\Sigma_{II\given S}$ and sample size $n-q'\le n-q$ completes the proof.
\end{proof}

It remains to prove the lemma used in proof above.
\begin{proof}[Proof of \cref{lem:pc:testFor2}]
    Let $Z_X = X / \sigma_X$, $Z_Y = Y/\sigma_Y$, then $Z_X,Z_Y\sim\mathcal{N}(0,1)$ and $\rho_{Z_X,Z_Y}=\rho=\cov(Z_X,Z_Y)\in [-1,1]$. Denote the corresponding samples to be $z_X=(z_X^{(1)},\ldots,z_X^{(n)})$ and $z_Y=(z_Y^{(1)},\ldots,z_Y^{(n)})$, therefore
    \begin{align*}
        \widehat{\rho} = & \frac{\widehat{\sigma}_{XY}}{\widehat{\sigma}_X\widehat{\sigma}_Y} 
        = \frac{\widehat{\sigma}_{XY} / (\sigma_X\sigma_Y)}{(\widehat{\sigma}_X/\sigma_X) \times (\widehat{\sigma}_Y/\sigma_Y)} 
        =  \frac{\langle z_X, z_Y\rangle}{\|z_X\|\|z_Y\|}\,.
    \end{align*}
    Then the deviation
    \begin{align*}
        |\widehat{\rho} - \rho| & = \bigg| \frac{\langle z_X, z_Y\rangle}{\|z_X\|\|z_Y\|} - \cov(Z_X,Z_Y) \bigg| \\
        & \le  \bigg| \frac{\langle z_X, z_Y\rangle / n}{\|z_X\|\|z_Y\|/n} - \frac{\cov(Z_X,Z_Y)}{\|z_X\|\|z_Y\|/n} + \frac{\cov(Z_X,Z_Y)}{\|z_X\|\|z_Y\|/n}- \cov(Z_X,Z_Y) \bigg|\\
        & \le \bigg|\frac{1}{\|z_X\|\|z_Y\|/n} - 1\bigg|\bigg|\langle z_X,z_Y\rangle/n - \cov(Z_X,Z_Y)\bigg| + \bigg|\langle z_X,z_Y\rangle/n - \cov(Z_X,Z_Y)\bigg| \\
        & \quad + \bigg|\cov(Z_X,Z_Y)\bigg|\bigg|\frac{1}{\|z_X\|\|z_Y\|/n} - 1\bigg| \,.
    \end{align*}
    We apply the following lemma to bound the errors:
    \begin{lemma}\label{lem:pc:varcovbound}
    If $(X,Y)\sim \mathcal{N}\Big(0,\begin{pmatrix}
    1  & r \\ r & 1
    \end{pmatrix}\Big)$ for $|r| \le 1$, then the sample variance $\widehat{\sigma}^2_X = \frac{1}{n}\sum_{i=1}^n {X^{(i)}}^2$, $\widehat{\sigma}^2_Y = \frac{1}{n}\sum_{i=1}^n {Y^{(i)}}^2$ and sample covariance $\widehat{\sigma}_{XY} = \frac{1}{n}\sum_{i=1}^n {X^{(i)}}{Y^{(i)}}$ have the following bounds: for $\zeta<1$, if $n \ge \frac{2048\log 7}{\zeta^2}$, then
    \begin{align*}
        & \prob(|\widehat{\sigma}^2_X - 1| \ge \zeta ) \le \exp(-n\zeta^2 / 16) \\
        & \prob(|\widehat{\sigma}^2_Y - 1| \ge \zeta ) \le \exp(-n\zeta^2 / 16) \\
        & \prob(|\widehat{\sigma}_{XY} - r| \ge \zeta ) \le \exp(-n\zeta^2 / 2048) \,.
    \end{align*}
    \end{lemma}
    Using \cref{lem:pc:varcovbound}, with probability at least $1-3\exp(-n\zeta^2 / 2048)$, we have $|\|z_X\|^2/n-1| \le \zeta$,  $|\|z_Y\|^2/n-1| \le \zeta$,  $|\langle z_X, z_Y \rangle/n-\cov(Z_X,Z_Y)| \le \zeta$. Then
    \begin{align*}
        \bigg|\frac{1}{\|z_X\|\|z_Y\|/n} - 1\bigg| & = \frac{|\|z_X\|\|z_Y\|/n - 1 |  }{\|z_X\|\|z_Y\|/n} \le \frac{\zeta}{1-\zeta}\,.
    \end{align*}
    Choose $\zeta = \frac{\delta}{3+\delta}$, then $\bigg|\frac{1}{\|z_X\|\|z_Y\|/n} - 1\bigg| \le \delta/3$, $\bigg|\langle z_X,z_Y\rangle/n - \cov(Z_X,Z_Y)\bigg| \le \delta / (3+\delta) \le \delta/3$. 
    Lastly,
    \begin{align*}
        |\widehat{\rho} - \rho| & \le \frac{\delta}{3}\times\frac{\delta}{3} +     \frac{\delta}{3}+\frac{\delta}{3} \le \delta\,,
    \end{align*}
    with probability at least
    \begin{align*}
        1 - 3\exp(-n\zeta^2/2048) & = 1-\exp\Big(-n\times \frac{\delta^2}{(3+\delta)^2} / 2048 + \log 3\Big) \\
        & \ge 1-\exp\Big(-n\times \frac{\delta^2}{16 \times 2048}+ \log 3\Big) \\
        & \ge 1 - \exp(-C_0n\delta^2)\,,
    \end{align*}
    for some constant $C_0>0$ as long as $n\gtrsim 1/\delta^2$.
\end{proof}

\begin{proof}[Proof of \cref{lem:pc:varcovbound}]
We only show variance bound for $X$. Since $\widehat{\sigma}_X^2 \sim \chi^2_n / n$, using the concentration of $\chi^2$ distribution, we have
\begin{align*}
    \prob(|\widehat{\sigma}^2_X - 1| \ge \zeta ) = \prob(|\chi^2_n - n| / n \ge \zeta ) \le \exp(-n\zeta^2 / 16)\,.
\end{align*}
Now we show bound for covariance. Since bivariate Gaussian $(X,Y)$ can be reparameterized by
\begin{align*}
    X & = U + W \\
    Y & = V + W 
\end{align*}
where $U$,$V$,$W$ are mutually independent with $\var(U)=\var(V)=1-r$, $ \var(W)=r$. Therefore,
\begin{align*}
    \widehat{\sigma}_{XY} & = \frac{1}{n}\sum_{i=1}^n{U^{(i)}}{V^{(i)}} + \frac{1}{n}\sum_{i=1}^n{U^{(i)}}{W^{(i)}} + \frac{1}{n}\sum_{i=1}^n{V^{(i)}}{W^{(i)}} + \frac{1}{n}\sum_{i=1}^n{W^{(i)}}^2 \\
    & = \frac{1-r}{2n}\Big[\sum_{i=1}^n\Big(\frac{{U'}^{(i)} + {V'}^{(i)}}{\sqrt{2}}\Big)^2 - \sum_{i=1}^n\Big(\frac{{U'}^{(i)} - {V'}^{(i)}}{\sqrt{2}}\Big)^2\Big] \\
    & \quad + \frac{\sqrt{r(1-r)}}{2n}\Big[\sum_{i=1}^n\Big(\frac{{U'}^{(i)} + {W'}^{(i)}}{\sqrt{2}}\Big)^2 - \sum_{i=1}^n\Big(\frac{{U'}^{(i)} - {W'}^{(i)}}{\sqrt{2}}\Big)^2\Big] \\
    & \quad + \frac{\sqrt{r(1-r)}}{2n}\Big[\sum_{i=1}^n\Big(\frac{{V'}^{(i)} + {W'}^{(i)}}{\sqrt{2}}\Big)^2 - \sum_{i=1}^n\Big(\frac{{V'}^{(i)} - {W'}^{(i)}}{\sqrt{2}}\Big)^2\Big] + \frac{r}{n}\sum_{i=1}^n {W^{'(i)}}^2 \\
    & \overset{\mathcal{D}}{\sim} \frac{1-r}{2n}({\chi^2_n}_{11} - {\chi^2_n}_{12}) + \frac{\sqrt{r(1-r)}}{2n}({\chi^2_n}_{21}- {\chi^2_n}_{22}) + \frac{\sqrt{r(1-r)}}{2n}({\chi^2_n}_{31}-{\chi^2_n}_{32}) + \frac{r}{n}{\chi^2_n}_{4}
\end{align*}
where $U',V',W'$ are standard normal random variables, thus $\sum_{i=1}^n({U'}^{(i)} \pm {V'}^{(i)})^2/2$, 
$\sum_{i=1}^n({U'}^{(i)} \pm {W'}^{(i)})^2/2$, 
$\sum_{i=1}^n({V'}^{(i)} \pm {W'}^{(i)})^2/2$ are $\chi^2_n$ random variables. Since $r\le 1$,
\begin{align*}
    \prob(|\widehat{\sigma}_{XY} - r| \ge \zeta) & \le \prob\Big(\frac{1-r}{2}\times \frac{1}{n}|{\chi^2_n}_{11} - {\chi^2_n}_{12}| \ge \zeta/4 \Big) \\
    & \quad +\prob\Big(\frac{\sqrt{r(1-r)}}{2}\times \frac{1}{n}|{\chi^2_n}_{21} - {\chi^2_n}_{22}| \ge \zeta/4 \Big)\\
    & \quad +\prob\Big(\frac{\sqrt{r(1-r)}}{2}\times \frac{1}{n}|{\chi^2_n}_{31} - {\chi^2_n}_{32}| \ge \zeta/4 \Big)\\
    & \quad +\prob\Big(r\times |{\chi^2_n}_{41} / n -1| \ge \zeta/4 \Big) \\
    & \le \prob\Big( |{\chi^2_n}_{11} / n -1 | \ge \zeta/8 \Big) + \prob\Big( |{\chi^2_n}_{12} / n -1 | \ge \zeta/8 \Big)\\
    & \quad +\prob\Big( |{\chi^2_n}_{21} / n -1 | \ge \zeta/8 \Big) + \prob\Big( |{\chi^2_n}_{22} / n -1 | \ge \zeta/8 \Big)\\
    & \quad +\prob\Big( |{\chi^2_n}_{31} / n -1 | \ge \zeta/8 \Big)+\prob\Big( |{\chi^2_n}_{32} / n -1 | \ge \zeta/8 \Big)\\
    & \quad +\prob\Big(|{\chi^2_n}_{41} / n -1| \ge \zeta/4 \Big) \\
    & \le 7\exp(-n\zeta^2 / 32^2) \le \exp(-n\zeta^2/2048)\,.
\end{align*}
The last inequality holds when $n \ge 2048\log 7 / \zeta^2$.
\end{proof}

\subsection{Proof of \cref{thm:tree:ub}}\label{app:tree:ub}
\thmtreeub*
\begin{proof}
We firstly show the correctness of \cref{alg:tree}. We make following notation of sets of nodes:
\begin{itemize}
    \item $W=\{(j,k): 1\le j<k\le d\}$ is the set of all pairs of nodes in $[d]$;
    \item $E$ is the true edge set;
    \item $A=\{(j,k): j \text{ and } k \text{ are d-separated by }\emptyset\}$;
    \item $B=\{(j,k): \exists \ell\in [d]\setminus \{j,k\}, j\text{ and }k \text{ are d-separated by }\ell\}$
    \item $C = \{(j,k):\exists \ell\in[d]\setminus\{j,k\}, j\rightarrow \ell\leftarrow k\text{ is a $v$-structure}\}$
    \item $D = \{(j,k):\exists \ell\in[d]\setminus\{j,k\}, j- \ell- k\text{ is a unshielded triple but not a $v$-structure}\}$
\end{itemize}
We claim that 
\begin{enumerate}
    \item $E$ and $A\cup B$ are disjoint;
    \item $W = E\cup A\cup B$;
    \item $C\subseteq A$;
    \item $D\subseteq B$.
\end{enumerate}
It is easy to see the first claim, since for any pair of nodes connected by an edge, they cannot be $d$-separated by any set, and vice versa. 

For the second claim, it suffices to show that for any pair of nodes not adjacent, it is in either $A$ or $B$. First of all, for any two nodes $j$ and $k$ not adjacent, there will be one and only one path, denoted as $\phi$, with length at least two between them. By property of polytree:
\begin{itemize}
    \item If there is a collider on $\phi$, then the path is blocked by $\emptyset$, so $(j,k)\in A$;
    \item If there is no collider on $\phi$, then any node on $\phi$ will block the path, thus there exists $\ell\in [d]\setminus \{j,k\}$ such that $i$ and $j$ are d-separated by $\ell$, so $(j,k)\in B$. 
\end{itemize}
For the third claim, since $j\rightarrow \ell\leftarrow k$ is the only path between $(j,k)$, which is blocked by $\emptyset$, thus $C\subseteq A$. For the forth claim, since $j- \ell- k$ is the only path between $(j,k)$, either one of $j\rightarrow \ell \rightarrow k$ and $j\leftarrow \ell \leftarrow k$ and $j\leftarrow \ell \rightarrow k$ will be blocked by $\ell$, thus $D\subseteq B$.

We now claim if the CI tests in Step 2 of \cref{alg:tree} are correct for
\begin{itemize}
    \item all pairs $(j,k)\in E$ with $\ell\in [d]\cup\{\emptyset\}\setminus \{j,k\}$;
    \item all pairs $(j,k)\in A$ with $\ell=\emptyset$;
    \item all pairs $(j,k)\in C$ with $\ell$ being the collider;
    \item all pairs $(j,k)\in B$ with $\ell$ being the corresponding separation node(s)\,,
\end{itemize}
then 
\begin{enumerate}
    \item the returned $\widehat{T}$ has the correct edge set $E$ thus is the correct skeleton;
    \item for any $(j,k)\in C$, $\ell\not\in S(j,k)$;
    \item for any $(j,k)\in D$, $\ell\in S(j,k)$.
\end{enumerate}
For the first claim, if the CI tests conducted in Step 2 are correct for $E$, then pairs in $E$ will pass all the CI tests and be included into $\widehat{E}$ (which is ensured by adjacency-faithfulness in Tree-faithfulness). 
But pairs in $A$ will not pass marginal independence tests, and pairs in $B$ will not pass some CI tests with corresponding $\ell$ (which is ensured by Markov property). 
Therefore, the returned $\widehat{T}$ is the correct skeleton. 
The second claim is ensured by orientation-faithfulness in Tree-faithfulness, and the third claim is ensured by Markov property and $D\subseteq B$. 

Once the returned $\widehat{T}$ is the correct skeleton, \cref{alg:orient} will use the returned separation sets to determine $v$-structure for each possible unshielded triple. Note that $\{\text{All unshielded triples}\} = C\cup D$. For any $(j,k)\in C$, $\ell\not\in S(j,k)$, thus it will be oriented as a $v$-structure; For any $(j,k)\in D$, $\ell\in S(j,k)$; thus it will remain as non-$v$-structure. Then \textsc{Orient} step is correct, which leads to correct CPDAG.

Finally we show the sample complexity of \cref{alg:tree} with CI tester~\eqref{eq:cit}. Note that correct CI tests implies correct estimation. Therefore,
\begin{align*}
    &\quad \prob(\widehat{T}\ne \sk(T)) \\
    & \le \prob\bigg(\cup_{\substack{(j,k)\in E\\ \ell\in  [d]\cup\{\emptyset\}\setminus \{j,k\}} \text{ or } \substack{(j,k)\in A \\ \ell=\emptyset} \text{ or } \substack{(j,k)\in C \\ \ell=\text{collider}} \text{ or } \substack{(j,k)\in B \\ \ell\text{ d-separates  }(j,k)} }|\widehat{\rho}_{ij\given \ell} - \rho_{ij\given \ell}|\ge c/2 \bigg) \\
    & \le \binom{d}{2}\times (1 + (d-2)) \times \sup_{\substack{(j,k)\in E\\ \ell\in  [d]\cup\{\emptyset\}\setminus \{j,k\}} \text{ or } \substack{(j,k)\in A \\ \ell=\emptyset} \text{ or } \substack{(j,k)\in C \\ \ell=\text{collider}} \text{ or } \substack{(j,k)\in B \\ \ell\text{ d-separates  }(j,k)} }\prob(|\widehat{\rho}_{ij\given \ell} - \rho_{ij\given \ell}|\ge c/2) \\
    & \le \exp(3\log d) \times \sup_{\substack{(j,k)\in E\\ \ell\in  [d]\cup\{\emptyset\}\setminus \{j,k\}} \text{ or } \substack{(j,k)\in A \\ \ell=\emptyset} \text{ or } \substack{(j,k)\in C \\ \ell=\text{collider}} \text{ or } \substack{(j,k)\in B \\ \ell\text{ d-separates  }(j,k)} }\prob(|\widehat{\rho}_{ij\given \ell} - \rho_{ij\given \ell}|\ge c/2) \\
    & \le \exp\bigg(-C_0(n-1)c^2 + 3\log d\bigg)\,.
\end{align*}
The first inequality is because it suffices to have $|\widehat{\rho}_{ij\given\ell}-\rho_{ij\given \ell}| \le c/2$ for correct CI test. By $c$-strong Tree-faithfulness, $|\rho_{ij\given S}| \ge c$ for $\rho_{ij\given S}\ne 0$. Therefore,
\[
\begin{cases}
\widehat{\rho}_{ij\given S} > c/2 & \text{ if }\rho_{ij\given S}\ne 0  \\
\widehat{\rho}_{ij\given S} \le c/2 & \text{ if }\rho_{ij\given S}= 0
\end{cases}
\]
Thus the cutoff $=c/2$ implies correct CI tests.
The last inequality is by \cref{lem:cit} where $q=1$ and the sample size requirement is satisfied by the stated sample complexity. Set RHS to be smaller than $\delta$, we need sample complexity
\[
n\gtrsim \frac{1}{c^2}\bigg( \log d +  \log \frac{1}{\delta}\bigg)\,,
\]
which completes the proof.
\end{proof}

\subsection{Proof of \cref{thm:tree:lb}}\label{app:tree:lb}
\thmtreelb*
\begin{proof}
We construct a hard ensemble to show the lower bound. The construction is as follows:
consider a subset $\drtr'\subset \drtr\subset \pltr$, where $\drtr'$ is all the directed trees rooted at the first node $k=1$. $\drtr'$ has the same cardinality as all undirected trees with $d$ nodes, and the elements in it have different skeletons and no $v$-structures.
Since our target is MEC, which is determined by its skeleton and $v$-structures, we have at least as many MECs as undirected trees, which leads to cardinality $|\drtr'| = d^{d-2}$ using Cayley's formula. Thus the size of the ensemble is lower bounded as
\begin{align*}
    \log |\drtr'| = (d-2)\log d \ge \frac{1}{2}d\log d
\end{align*}
The inequality holds when $d$ is large enough, e.g. $d\ge 4$. 
Any directed tree has an important property: each node has at most one parent. Then we parameterize $\drtr'$ as follows
\begin{align}\label{eq:tree:lbmodel}
    X_{k} = \beta  X_{\pa(k)} + \eta_k\,, \ \ \forall k\in [d]
\end{align}
where $\eta_k\sim\mathcal{N}(0,1)$ for all $k\in[d]$. 
Now we determine $\beta>0$ to make sure the parametrization satisfies $c$-strong Tree-faithfulness.

In the subsequent lemma, we assert that the condition $\beta^2 = 2c^2 \asymp c^2$ is adequate for the validity of $c$-strong Tree-faithfulness, provided that $c$ is sufficiently small:
\begin{restatable}{lemma}{lemmatreebeta}
\label{lem:tree:beta}
    If $\beta = \sqrt{2}c$ and $c^2\le 1/5$, then for any $T\in\drtr'$, the distribution defined in~\eqref{eq:tree:lbmodel}:
    \begin{enumerate}
        \item is $c$-strong Tree-faithful to $T$;
        \item for all $k\in[d]$, $\var(X_k) \le 1 + \frac{\beta^2}{1-\beta^2}$.
    \end{enumerate}
\end{restatable}

It remains to bound the KL divergence between any two instances in this ensemble. 
Before that, we claim that for any instance, we have $\cov(X_k,X_j)> 0$ for all distinct $j,k\in [d]$. This is because for any pair of distinct nodes $(j,k)$, there can be 3 possible paths between them:
\begin{itemize}
    \item There is a directed path $j\to \phi_1 \to \cdots\to \phi_h \to k$ with length $h+1$, then $\cov(X_j,X_k) =\E[X_jX_k] =\beta^{h+1}\E[X_j^2] > 0$;
    \item There is a directed path $k\to \phi_1 \to \cdots\to \phi_h \to j$ with length $h+1$, then $\cov(X_j,X_k) =\E[X_jX_k] =\beta^{h+1}\E[X_k^2] > 0$;
    \item $j,k$ share a common ancestor $\ell$ and there is a path $j\leftarrow \phi_1 \leftarrow \cdots \leftarrow \phi_h \leftarrow \ell \to \varphi_1\to\cdots\to\varphi_g \to k$, then $\cov(X_j,X_k) =\E[X_jX_k] =\beta^{h+g+2}\E[X_\ell^2] > 0$.
\end{itemize}
To compute the KL divergence between distributions $P_0$ and $P_1$ induced by any two $T_0,T_1\in \drtr'$, let's first look at the covariance matrices of them $\Sigma_0,\Sigma_1$. Under our parametrization, they share the same determinant. To see this, let covariance matrix of $\eta$ be $\Sigma_\eta=I_d$, for $\ell\in \{0,1\}$, $\det(\Sigma_\ell) = \det(\Sigma_\eta) = \det(I_d) = 1$. Then the KL divergence is:
\begin{align*}
    \dkl(P_0\| P_1) & = \E_{P_0} \log \frac{P_0}{P_1} \\
    & = \E_{P_0} \log \frac{\exp\bigg(-\frac{1}{2}\sum_{k=1}^d(X_k - \beta  \pa_{T_0}(k))^2 \bigg) / \sqrt{\det(\Sigma_0)}}{\exp\bigg(-\frac{1}{2}\sum_{k=1}^d(X_k - \beta   \pa_{T_1}(k))^2 \bigg) / \sqrt{\det(\Sigma_1)}} \\
    & = \frac{1}{2}\bigg[\E_{P_0}\sum_{k=1}^d(X_k - \beta   \pa_{T_1}(k))^2 - d\bigg] \,.
\end{align*}
For all $k\in [d]$, let $\pa_{T_1}(k)=j$, then
\begin{align*}
    \E_{P_0}(X_{k} - \beta  \pa_{T_1}(k))^2
    & = \E_{P_0}[X_{k}^2] + \beta^2 \E_{P_0}[X_{j}^2] - 2\beta \E_{P_0}[X_{k} X_{j}] \\
    & \le \E_{P_0}[X_{k}^2] + \beta^2 \E_{P_0}[X_{j}^2]  \\
    & \le (1+\beta^2)\bigg(1+ \frac{\beta^2}{1-\beta^2}\bigg) \\
    & = 1 + \frac{2\beta^2}{1-\beta^2}\,.
\end{align*}
The first inequality is because all covariances are positive; the second one is due to the upper bound for all variances. Thus, we have
\begin{align*}
    \dkl(P_0\| P_1) & \le \frac{1}{2}\bigg(d + \frac{2d\beta^2}{1-\beta^2} - d\bigg)  = d\beta^2 \times \frac{1}{1-\beta^2} \le 2d\beta^2 = 4dc^2
\end{align*}
The last inequality holds when $\beta^2$ is small enough, e.g. $\beta^2\le 1/2$. 
The proof follows from applying Fano's inequality with KL divergence upper bound $4dc^2$ and cardinality of ensemble lower bound $\frac{1}{2}d\log d$.
\end{proof}

We end by proving the lemma used in the lower bound proof.
\begin{proof}[Proof of \cref{lem:tree:beta}]
Since for any $T\in \drtr'$, there is no $v$-structure because each node has at most one parent, thus it suffices to show the first part of \cref{defn:strongtreefaith}.

We first show all marginal variances are bounded, i.e. $1\le \var(X_k) \le 1 + \beta^2 / (1-\beta^2)$ for all $k\in[d]$. Starting from the root node $r$, whose variance is $\var(X_r)=\var(\eta_r)=1$, we can compute the variances of its children, they are all $\var(X_\ell)=\var(\eta_\ell)+\beta^2\var(X_r) = 1+\beta^2$ for all $\ell\in\ch(r)$. Proceed the calculation, $\var(X_j)=\var(\eta_j)+\beta^2\var(X_\ell)=1+\beta^2+\beta^4$ for all $j\in \ch(\ell)$ and $\ell\in\ch(r)$. Therefore, because the longest path has length at most $d-1$,
\begin{align*}
    1\le \var(X_k) \le 1+\beta^2 + \beta^4 + \cdots + \beta^{2d} = 1 + \frac{\beta^2}{1-\beta^2} \times (1 - \beta^{2(d-1)}) \le 1 + \frac{\beta^2}{1-\beta^2}\,, \ \ \forall k\in[d]
\end{align*}
Now we can show the marginal correlation is lower bounded for any adjacent nodes $(j,k)$. Without loss of generality, let $j = \pa(k)$, then $X_k=\beta X_j +\eta_k$, and the correlation
\begin{align*}
    \rho(X_j,X_k) = \frac{\E[X_jX_k]}{\sqrt{\var(X_k)\var(X_j)}} = \beta \sqrt{\frac{\E[X_j^2]}{1 + \beta^2\E[X_j^2]}}
\end{align*}
Thus $\rho(X_j,X_k)\ge c \Leftrightarrow \beta^2\E[X_j^2] \ge \frac{c^2}{1-c^2}$. Since $\E(X^2_{j})\ge 1$, then $\beta^2\E(X_{j}^2) \ge \beta^2 = 2c^2 \ge \frac{c^2}{1-c^2}$ when $c^2 \le 1/2$. Now consider any pair of adjacent nodes $(j,k)$, assuming $j=\pa(k)$, and any other node $\ell\in [d]\setminus \{j,k\}$, there are 4 cases on the relation between $\ell$ and $(j,k)$:
\begin{enumerate}
    \item $\ell$ is ancestor of $j$, i.e. a directed path $\phi$: $\ell \to \phi_1 \to\cdots\to \phi_h \to j$;
    \item $j$ and $\ell$ share the same ancestor $w$, i.e. a directed path $\phi$: $w \to \phi_1 \to\cdots\to \phi_h \to j$ and a directed path $\varphi$: $w\to\varphi_1\to\cdots\to \varphi_g\to \ell$;
    \item $\ell$ is a descendant of $k$, i.e. a directed path $\phi$: $j\to k\to \phi_1\to\cdots\to \phi_h\to \ell$;
    \item $\ell$ is a descendant of $j$ but not $k$, i.e. a directed path $\phi$: $j\to \phi_1\to\cdots\to \phi_h\to \ell$ not going through $k$;
\end{enumerate}
where $h\ge 0$ in either case. We deal with them separately:
\begin{itemize}
    \item For the first and second case, because $X_\ell\indep \eta_k$, the conditional correlation is
    \begin{align*}
        \rho(X_k,X_j \given X_{\ell}) & = \frac{\E[X_kX_{j}\given X_{\ell}]}{\sqrt{\E(X_k^2\given X_{\ell})\E(X^2_{j}\given X_{\ell})}} \\
        & = \frac{\beta\E[X_j^2\given X_{\ell}]}{\sqrt{\E(X_j^2\given X_{\ell})(1+\beta^2\E(X^2_{j}\given X_{\ell}))}} \\
        & = \sqrt{\frac{\beta^2\E[X_j^2\given X_{\ell}]}{1+\beta^2\E(X^2_{j}\given X_{\ell})}} 
    \end{align*}
    Thus $\rho(X_k,X_{j} \given X_{\ell}) \ge c \Leftrightarrow \beta^2\E(X^2_{j}\given X_{\ell}) \ge \frac{c^2}{1-c^2}$. Since $X_{\phi_h}\indep \eta_j\given X_\ell$, we have $\E(X^2_{j}\given X_{\ell}) = 1 + \beta^2\E(X_{\phi_h}^2\given X_{\ell}) \ge 1$, then $\beta^2\E(X_{j}^2\given X_{\ell}) \ge \beta^2 = 2c^2 \ge \frac{c^2}{1-c^2}$ when $c^2 \le 1/2$.
    \item For the third case, denote $v=\E[X_j^2]$, let's compute the covariance matrix of $(X_k,X_{k},X_{\ell})$:
    \begin{align*}
        \begin{pmatrix}
         v & \beta v & \beta^{h+2}v\\
         \beta v & \beta^2 v + 1 & \beta^{h+1}(\beta^2 v + 1)\\
         \beta^{h+2}v & \beta^{h+1}(\beta^2 v + 1)& \beta^{2(h+2)} v + \beta^{2(h+1)}+\cdots + \beta^2 + 1
        \end{pmatrix}
    \end{align*}
    Denote $V(v,h)=\beta^{2(h+2)} v + \beta^{2(h+1)}+\cdots + \beta^2 + 1$. The covariance matrix of $(X_j,X_{k})$ given $X_{\ell}$ is
    \begin{align*}
        &\begin{pmatrix}
         v & \beta v \\ \beta v & \beta^2 v + 1
        \end{pmatrix}  - \frac{1}{V(v,h)}\begin{pmatrix}
         \beta^{2(h+2)}v^2 & \beta^{2h+3}v(\beta^2v + 1) \\
         \beta^{2h+3}v(\beta^2v + 1) & \beta^{2(h+1)} (\beta^2 v+1)^2
        \end{pmatrix} \\
        =& \frac{1}{V(v,h)}\bigg[
        \begin{pmatrix}
        \beta^{2(h+2)}v^2 + \beta^{2(h+1)}v + \cdots + \beta^2v+ v 
        & 
        \beta^{2(h+2)+1}v^2 + \beta^{2(h+1)+1}v + \cdots + \beta^3v + \beta v
        \\
        \beta^{2(h+2)+1}v^2 + \beta^{2(h+1)+1}v + \cdots + \beta^3v+\beta v
        & 
        (\beta^{2(h+2)+2}v^2 + \beta^{2(h+1)+2}v + \cdots +\beta^4v +  \beta^2 v \\
        & +\beta^{2(h+2)}v + \beta^{2(h+1)} + \cdots + \beta^2 + 1)
        \end{pmatrix} \\
        & - 
        \begin{pmatrix}
         \beta^{2(h+2)}v^2 & \beta^{2h+3}v(\beta^2v + 1) \\
         \beta^{2h+3}v(\beta^2v + 1) & \beta^{2(h+1)} (\beta^2 v+1)^2
        \end{pmatrix}  
        \bigg] \\
        = & \frac{1}{V(v,h)}\begin{pmatrix}
          (\beta^{2(h+1)}+\beta^{2h} + \cdots +\beta^2+ 1)v 
          & 
          (\beta^{2h}+ \beta^{2(h-1)}\cdots +\beta^2 +1)\beta v 
          \\
          (\beta^{2h}+ \beta^{2(h-1)}\cdots +\beta^2 +1)\beta v
          & 
          (\beta^{2h} + \beta^{2(h-1)} + \cdots + \beta^2 + 1)(\beta^2 v+ 1)
        \end{pmatrix} 
    \end{align*}
    Thus the conditional correlation is
    \begin{align*}
        \rho(X_j,X_{k}\given X_{\ell}) & = \frac{\beta v \times \frac{1-\beta^{2(h+1)}}{1-\beta^2}}{\sqrt{v\times \frac{1-\beta^{2(h+2)}}{1-\beta^2} \times (1+\beta^2 v) \times \frac{1-\beta^{2(h+1)}}{1-\beta^2}}} \\
        & = \sqrt{\frac{\beta^2 v}{1 +\beta^2 v} \times \frac{1-\beta^{2(h+1)}}{1-\beta^{2(h+2)}}}
    \end{align*}
    Denote $f(h) = \frac{1-\beta^{2(h+1)}}{1-\beta^{2(h+2)}} = 1 - \frac{(1 - \beta^2)\beta^{2(h+1)}}{1-\beta^{2(h+2)}}$, which is increasing in $h$ with minimum value being $f(0) = \frac{1}{1+\beta^2}$. Therefore, $\rho(X_j,X_{k}\given X_{\ell}) \ge c  \Leftrightarrow \beta^2 v \ge \frac{c^2}{f(h)-c^2}$. Since $v=\E[X_j^2]\ge 1$ for all $j\in[d]$, then $\beta^2 v \ge \beta^2 = 2c^2 \ge \frac{c^2}{1/(1+2c^2) - c^2} \ge \frac{c^2}{f(h) - c^2}$ when $c^2\le 1/5$, which yields the bound.
    \item For the forth case, analogously, denote $v=\E[X_j^2]$, let's compute the covariance matrix of $(X_k,X_{k},X_{\ell})$:
    \begin{align*}
        \begin{pmatrix}
         v & \beta v & \beta^{h+1}v\\
         \beta v & \beta^2 v + 1 & \beta^{h+2}v \\
         \beta^{h+1}v & \beta^{h+2}v& \beta^{2(h+1)}+\beta^{2h}+\cdots + \beta^2 + 1
        \end{pmatrix}
    \end{align*}
    Denote $W(v,h)=\beta^{2(h+1)}+\beta^{2h} v + +\cdots + \beta^2 + 1$. The covariance matrix of $(X_j,X_{k})$ given $X_{\ell}$ is
    \begin{align*}
        &\begin{pmatrix}
         v & \beta v \\ \beta v & \beta^2 v + 1
        \end{pmatrix}  - \frac{1}{W(v,h)}\begin{pmatrix}
         \beta^{2(h+1)}v^2 & \beta^{2h+3}v^2 \\
         \beta^{2h+3}v^2 & \beta^{2(h+2)}v^2
        \end{pmatrix} \\
        =& \frac{1}{W(v,h)}\bigg[
        \begin{pmatrix}
        \beta^{2(h+1)}v^2 + \beta^{2h}v + \cdots + \beta^2v+ v 
        & 
        \beta^{2(h+1)+1}v^2 + \beta^{2h+1}v + \cdots + \beta^3v + \beta v
        \\
        \beta^{2(h+1)+1}v^2 + \beta^{2h+1}v + \cdots + \beta^3v+\beta v
        & 
        (\beta^{2(h+1)+2}v^2 + \beta^{2h+2}v + \cdots +\beta^4v +  \beta^2 v \\
        & +\beta^{2(h+1)}v + \beta^{2h} + \cdots + \beta^2 + 1)
        \end{pmatrix} \\
        & - 
        \begin{pmatrix}
         \beta^{2(h+1)}v^2 & \beta^{2h+3}v^2 \\
         \beta^{2h+3}v^2 & \beta^{2(h+2)}v^2
        \end{pmatrix} 
        \bigg] \\
        = & \frac{1}{W(v,h)}\begin{pmatrix}
          (\beta^{2h} +\beta^{2(h-1)}+ \cdots +\beta^2+ 1)v 
          & 
          (\beta^{2h}+ \beta^{2(h-1)}\cdots +\beta^2 +1)\beta v 
          \\
          (\beta^{2h}+ \beta^{2(h-1)}\cdots +\beta^2 +1)\beta v
          & 
          (\beta^{2h} + \beta^{2(h-1)} + \cdots + \beta^2 + 1)(\beta^2 v+ 1) + \beta^{2h+2}v
        \end{pmatrix} 
    \end{align*}
    Thus the conditional correlation is
    \begin{align*}
        \rho(X_j,X_{k}\given X_{\ell}) & = \frac{\beta v \times \frac{1-\beta^{2(h+1)}}{1-\beta^2}}{\sqrt{v\times \frac{1-\beta^{2(h+1)}}{1-\beta^2} \times\bigg[ (1+\beta^2 v) \times \frac{1-\beta^{2(h+1)}}{1-\beta^2} + \beta^{2h} \times \beta^2v\bigg]}} \\
        & = \sqrt{\frac{\beta^2 v}{1 +\big(1+\frac{\beta^{2h}}{g(h)} \big) \beta^2v}}
    \end{align*}
    where $ g(h) = \frac{1-\beta^{2(h+1)}}{1-\beta^2}\ge 1$ for $h\ge 0$. Since $\beta^2=2c^2\le 1$, then $1 + \beta^{2h}/g(h) \le 2$. Since $\rho(X_j,X_k\given X_\ell)\ge c \Leftrightarrow \beta^2v \ge \frac{c^2}{1 - \big(1 + \frac{\beta^{2h}}{g(h)}\big)c^2}$, and $v=\E[X_j^2]\ge 1$ for all $j\in[d]$, then $\beta^2 v \ge \beta^2 = 2c^2 \ge \frac{c^2}{1 - 2c^2} \ge \frac{c^2}{1 - \big(1 + \frac{\beta^{2h}}{g(h)}\big)c^2}$ when $c^2\le 1/5$, which completes the proof.
\end{itemize}

\end{proof}

\section{Experiments}\label{app:expt}

\paragraph{Synthetic Data Generation} 
We generate trees using package \texttt{networkx}, then randomly pick a node as root and orient it into a directed tree. 
We consider number of nodes $d \in \{10, 50, 100\}$. To generate the data as in~\eqref{eq:sem}, we uniformly sample $\beta_k$ from the interval $(-0.5, 0.1] \cup [0.1, 0.5)$ as our coefficient weight. For sample size $n = \{1000, 2000, 3000, 4000, 5000\}$, we generate our i.i.d. samples $X \in \R^{n\times d}$ according to~\eqref{eq:sem}, where $\eta \sim \N(\bold{0}, I_{d\times d})$. Besides, we also present experiments on agnostic setting where $\eta \sim \unif(-1, 1)$ is uniform distribution, or $\eta \sim \text{Laplace}(0, 1)$ is Laplace distribution.

\paragraph{Baselines}
We have employed two baseline algorithms: the PC algorithm has been executed using the Python package \texttt{Causal-learn}, while the GES algorithm has been implemented with \texttt{py-tetrad}.

\paragraph{Evaluation}
For each experiment setup, we report the average (over 50 random instantiations) Structural Hamming Distance (SHD) between the ground truth and our estimated graph skeleton, and the Precise Recovery Rate (PRR), which is the frequency of exact recovery of the tree skeleton. Results are reported in \cref{fig:gaussian_10}-\ref{fig:laplace_100}.
All experiments were conduced on an Intel Core i7-12800H 2.40GHz CPU.

\begin{figure*}
\centering 
\subfigure[SHD]{\label{fig:gaussian_10_shd}\includegraphics[width=82mm]{\detokenize{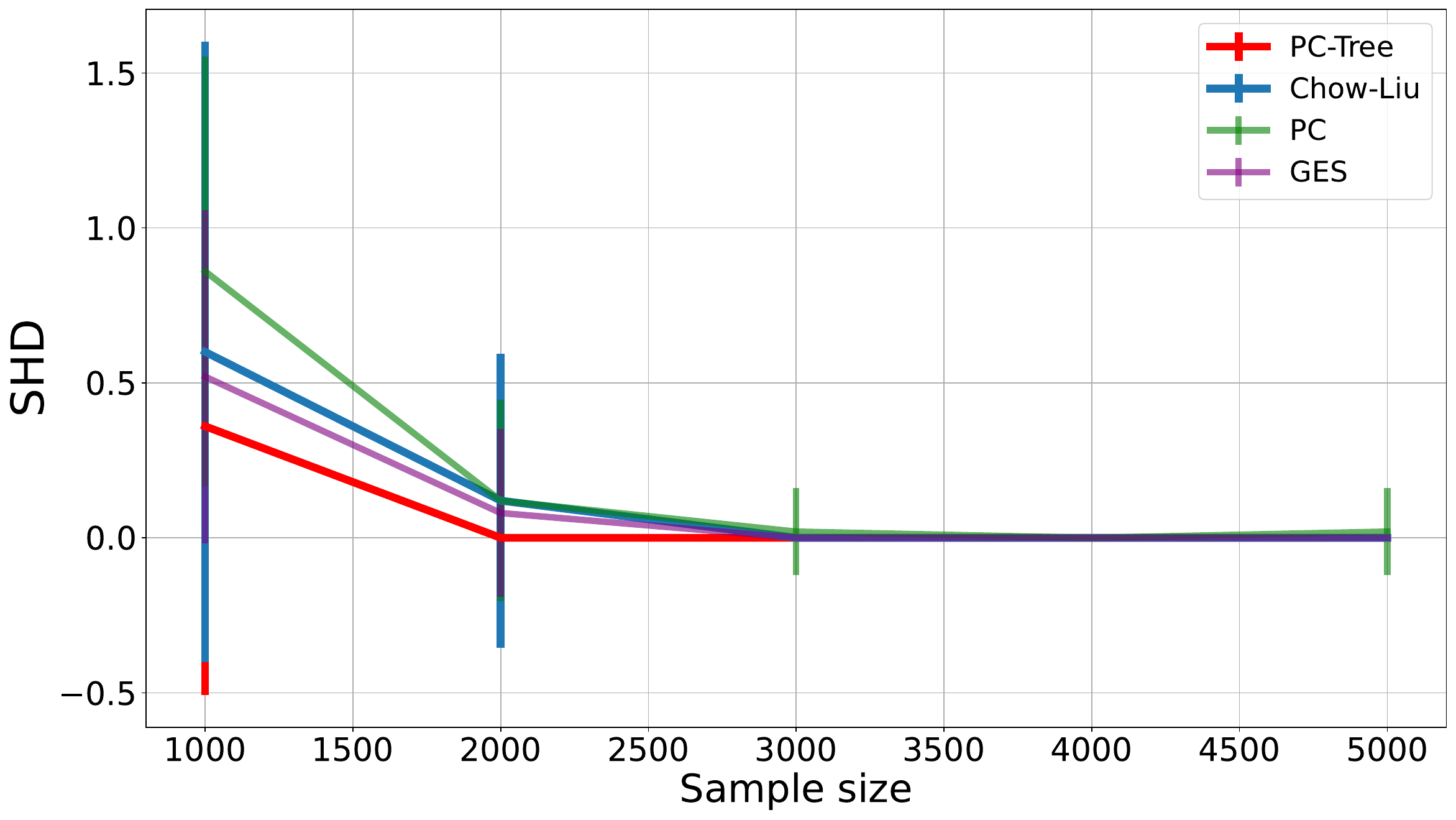}}}
\subfigure[PRR]{\label{fig:gaussian_10_prr}\includegraphics[width=82mm]{\detokenize{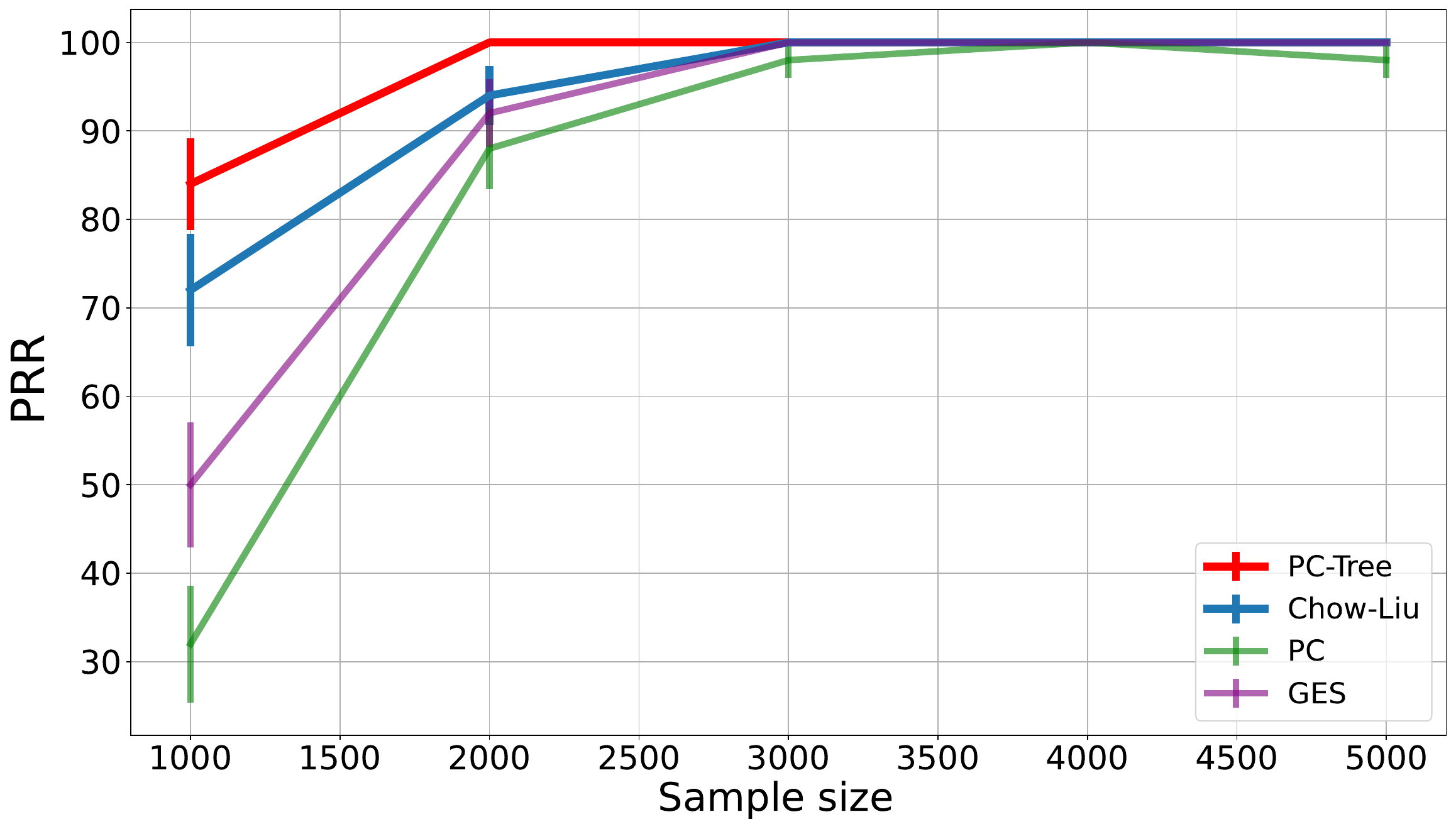}}}
\caption{SHD and PRR for Gaussian $\eta$ and $d=10$.}
\label{fig:gaussian_10}
\end{figure*}

\begin{figure*}
\centering 
\subfigure[SHD]{\label{fig:gaussian_50_shd}\includegraphics[width=82mm]{\detokenize{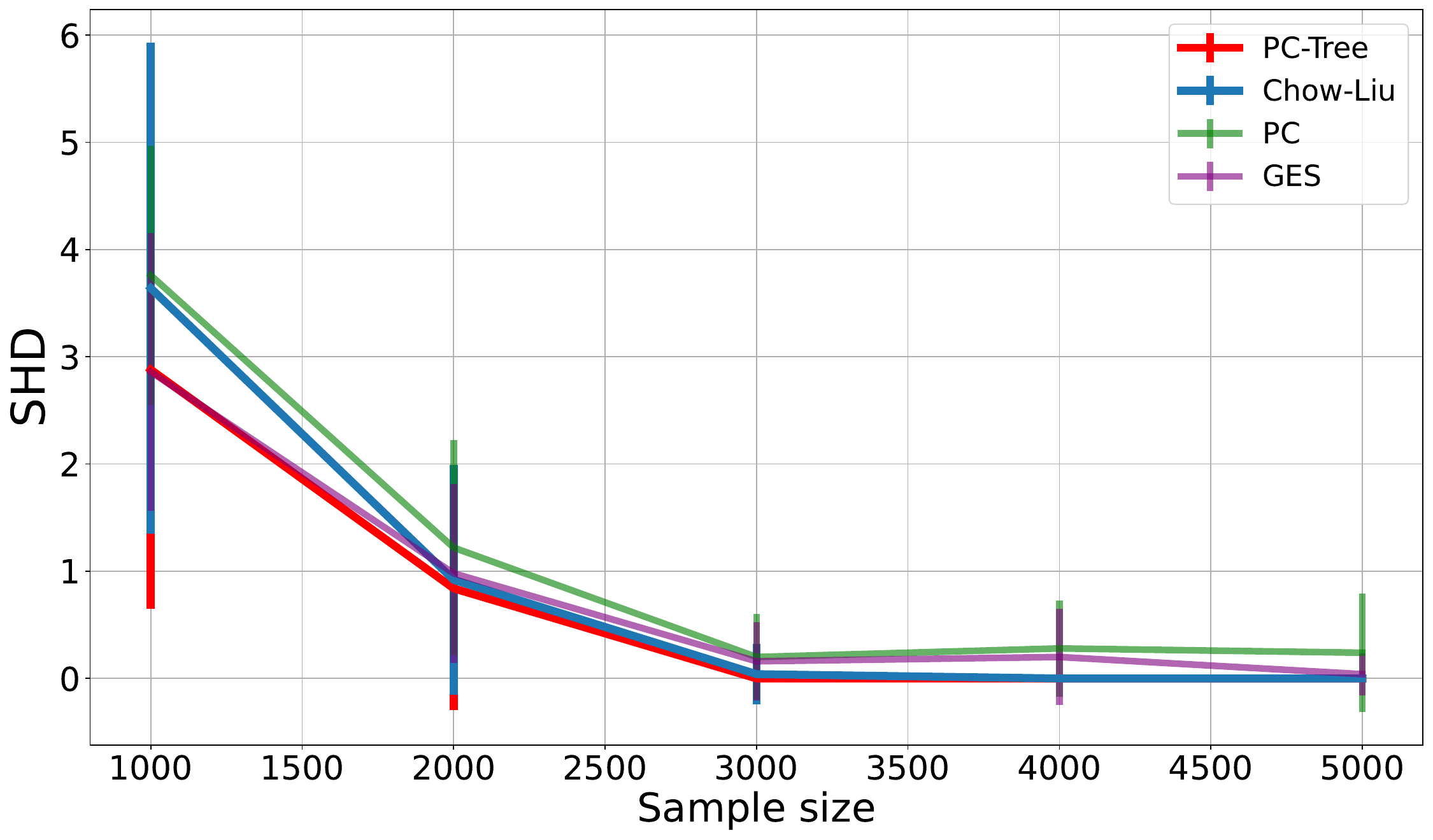}}}
\subfigure[PRR]{\label{fig:gaussian_50_prr}\includegraphics[width=82mm]{\detokenize{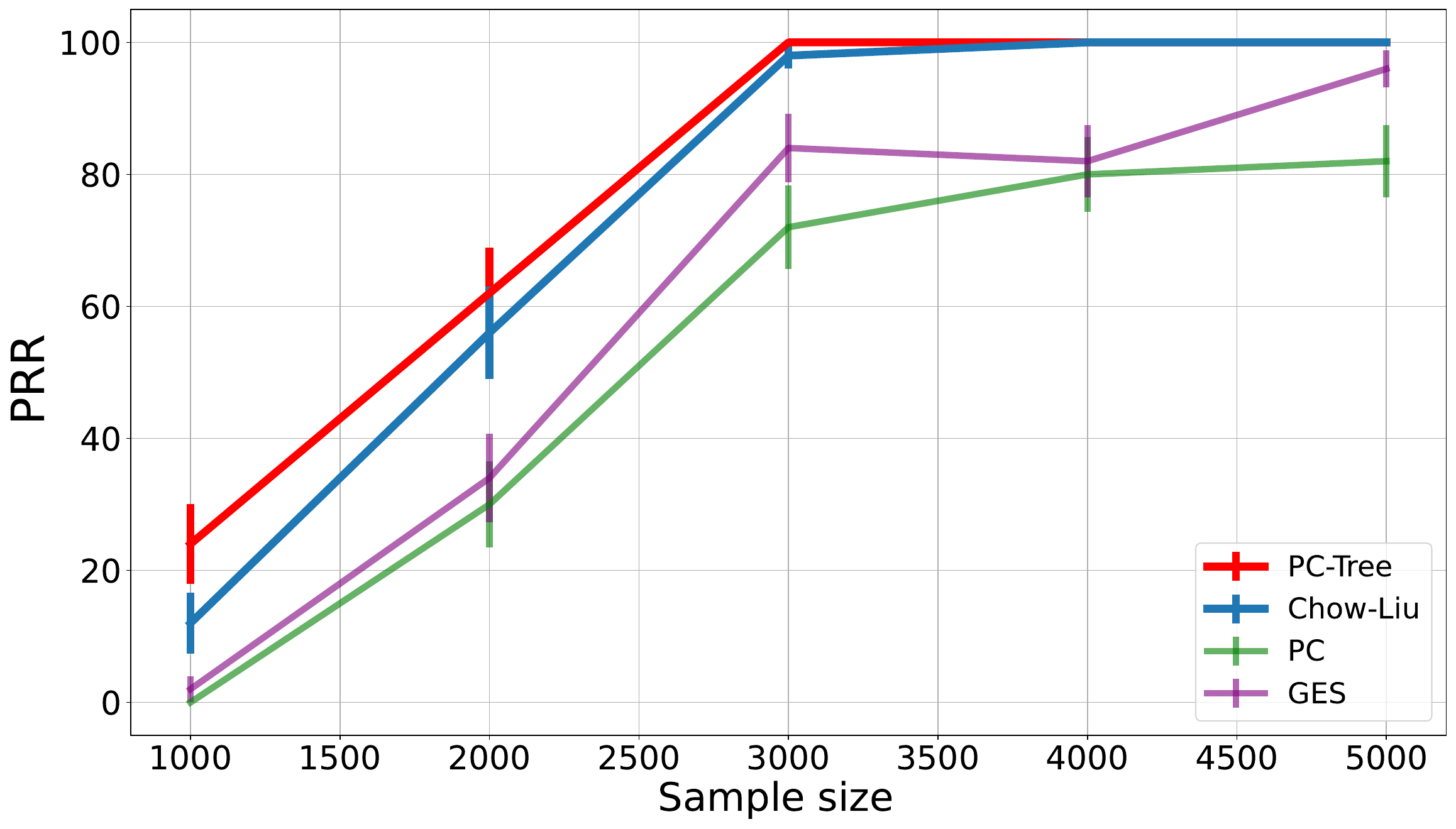}}}
\caption{SHD and PRR for Gaussian $\eta$ and $d=50$.}
\label{fig:gaussian_50}
\end{figure*}

\begin{figure*}
\centering 
\subfigure[SHD]{\label{fig:uniform_10_shd}\includegraphics[width=82mm]{\detokenize{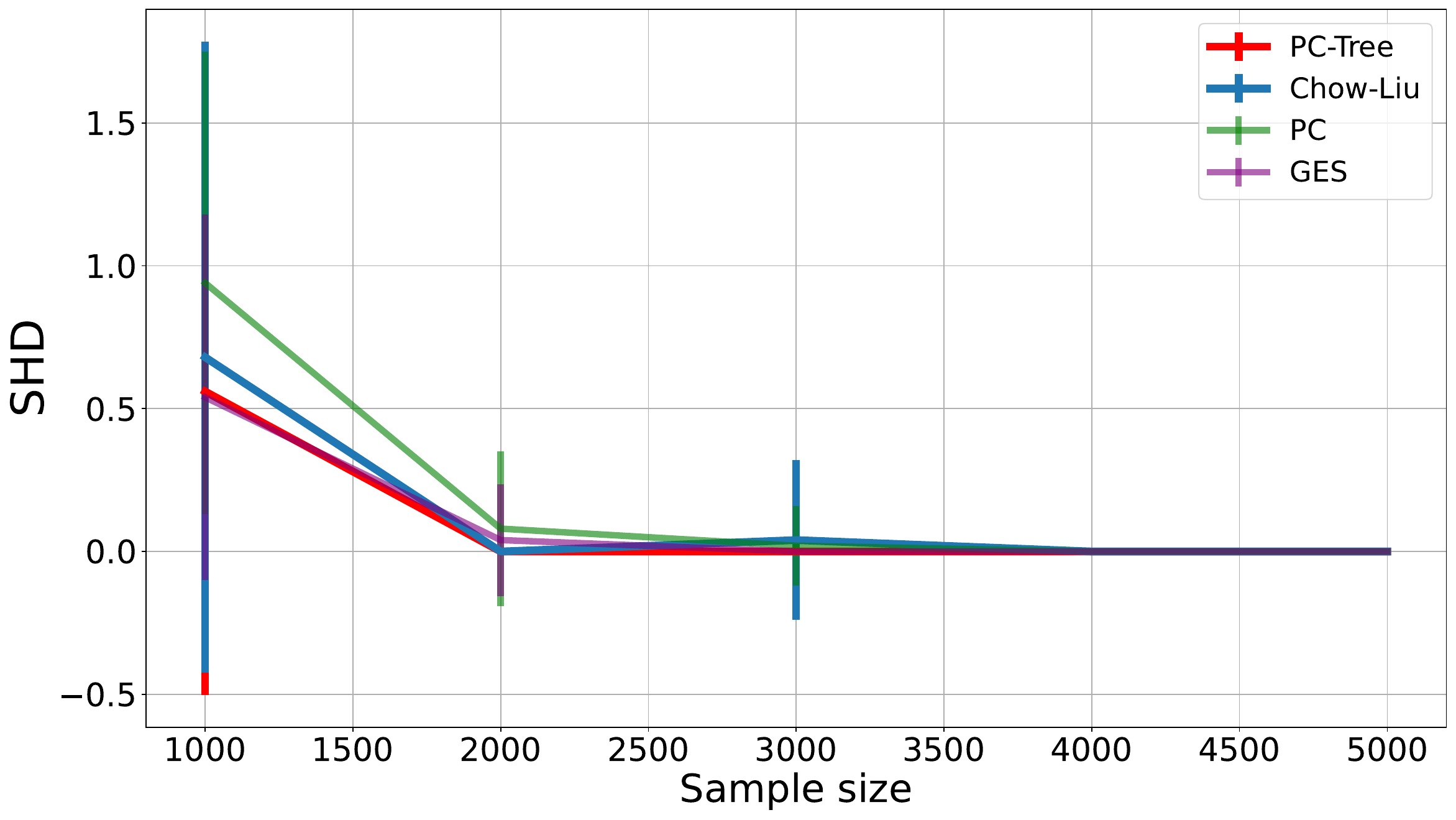}}}
\subfigure[PRR]{\label{fig:uniform_10_prr}\includegraphics[width=82mm]{\detokenize{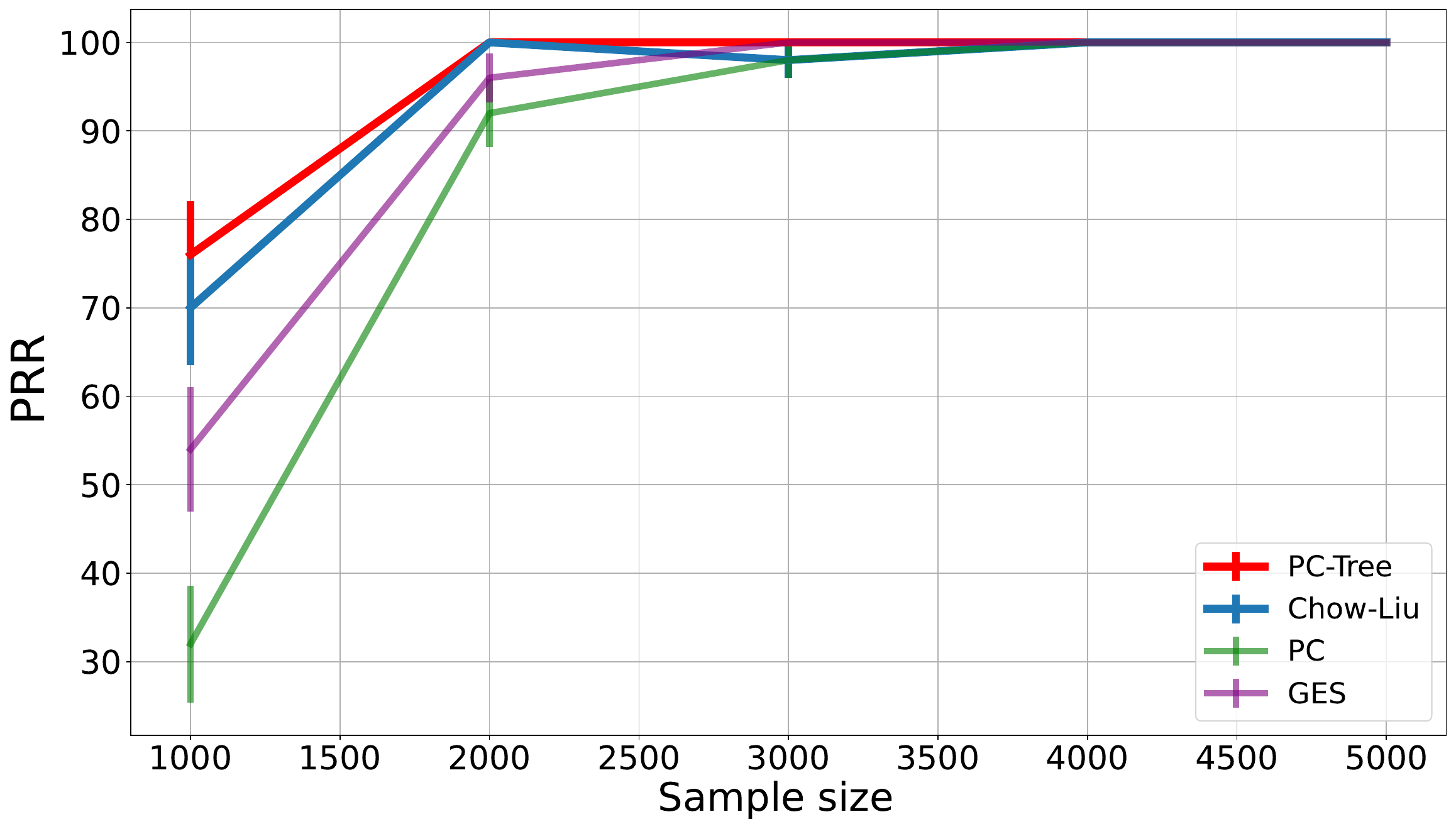}}}
\caption{SHD and PRR for Uniform $\eta$ and $d=10$.}
\label{fig:uniforme_10}
\end{figure*}

\begin{figure*}
\centering 
\subfigure[SHD]{\label{fig:uniform_50_shd}\includegraphics[width=82mm]{\detokenize{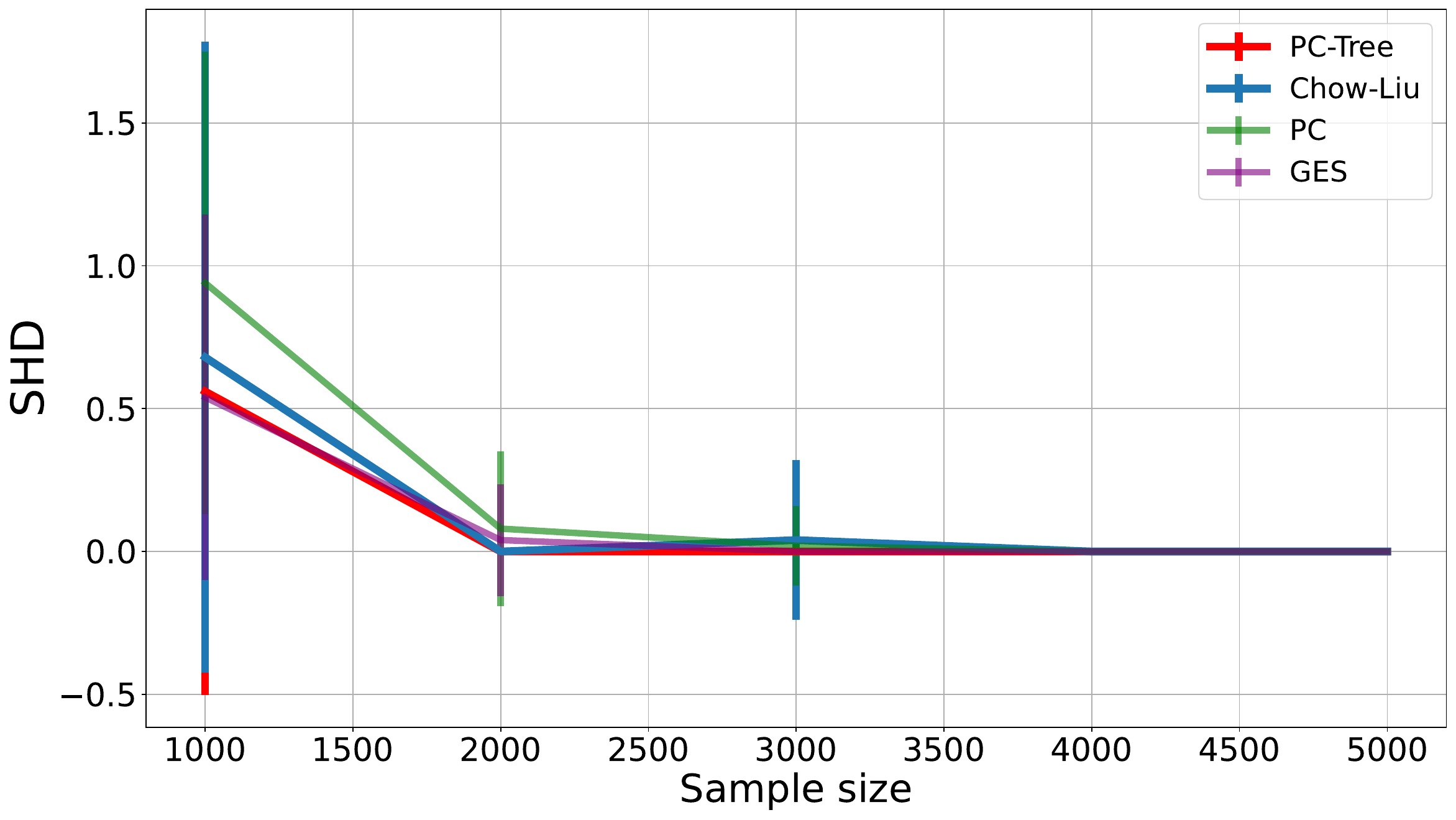}}}
\subfigure[PRR]{\label{fig:uniform_50_prr}\includegraphics[width=82mm]{\detokenize{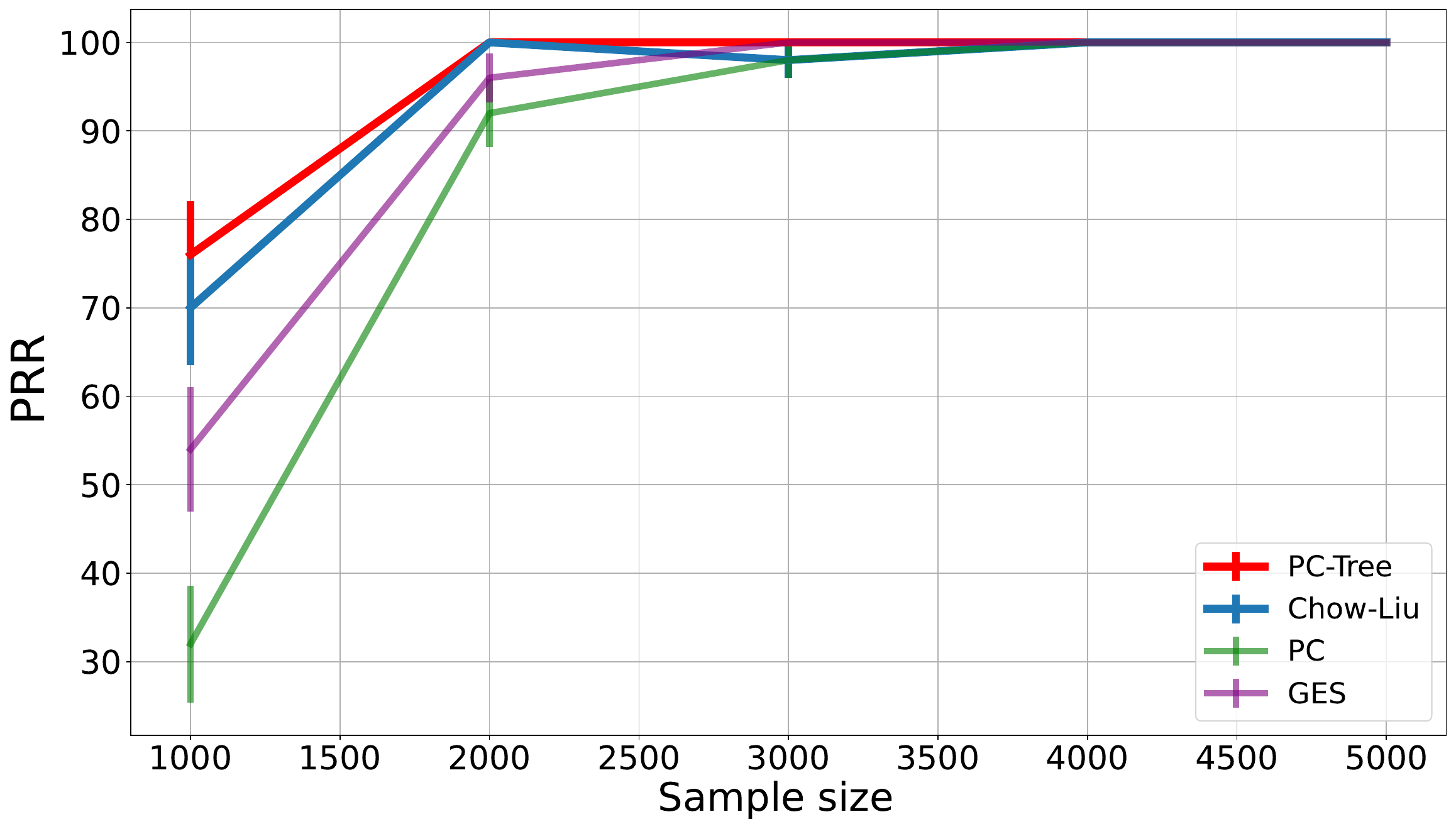}}}
\caption{SHD and PRR for Uniform $\eta$ and $d=50$.}
\label{fig:uniform_50}
\end{figure*}

\begin{figure*}
\centering 
\subfigure[SHD]{\label{fig:uniform_100_shd}\includegraphics[width=82mm]{\detokenize{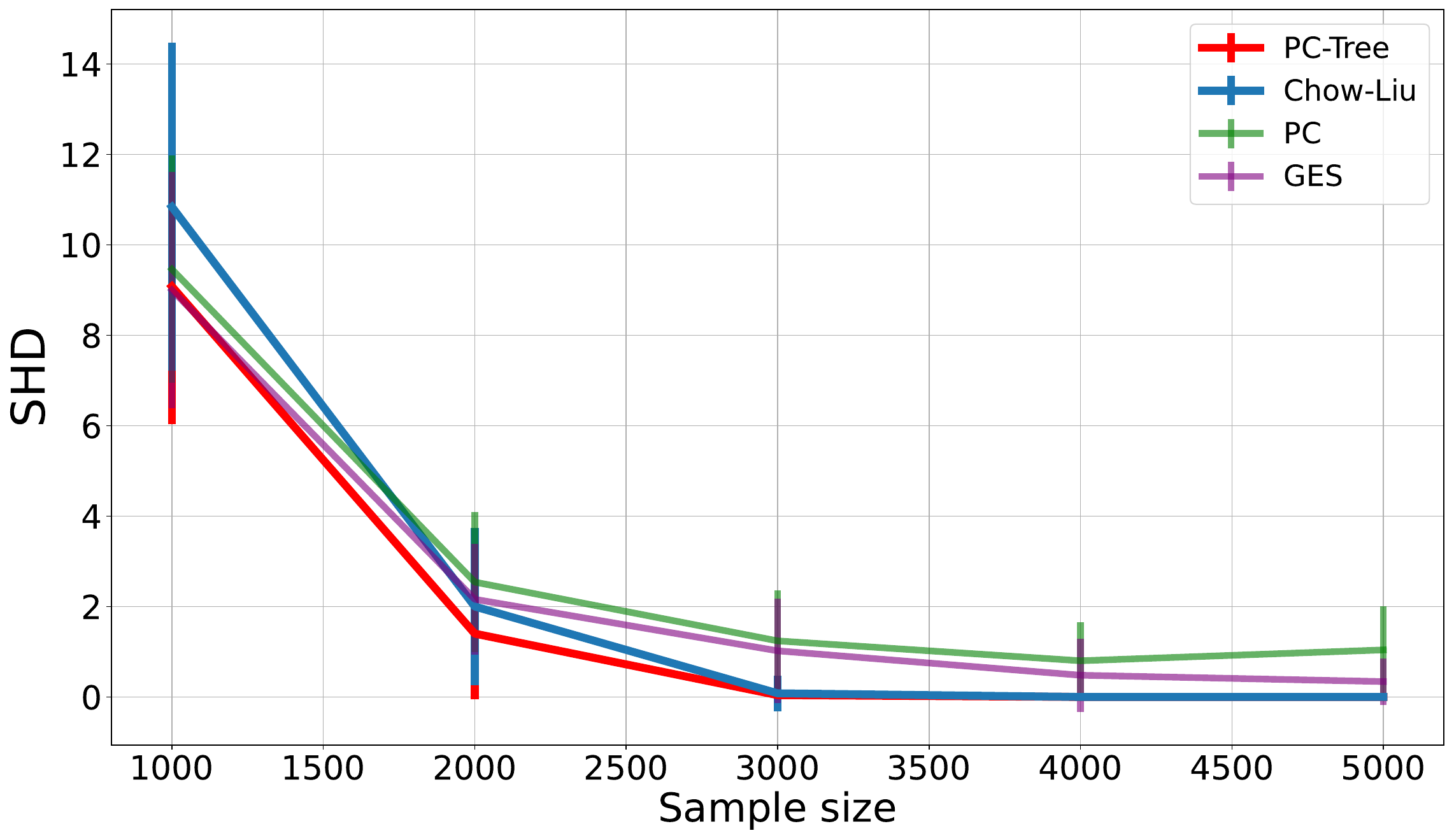}}}
\subfigure[PRR]{\label{fig:uniform_100_prr}\includegraphics[width=82mm]{\detokenize{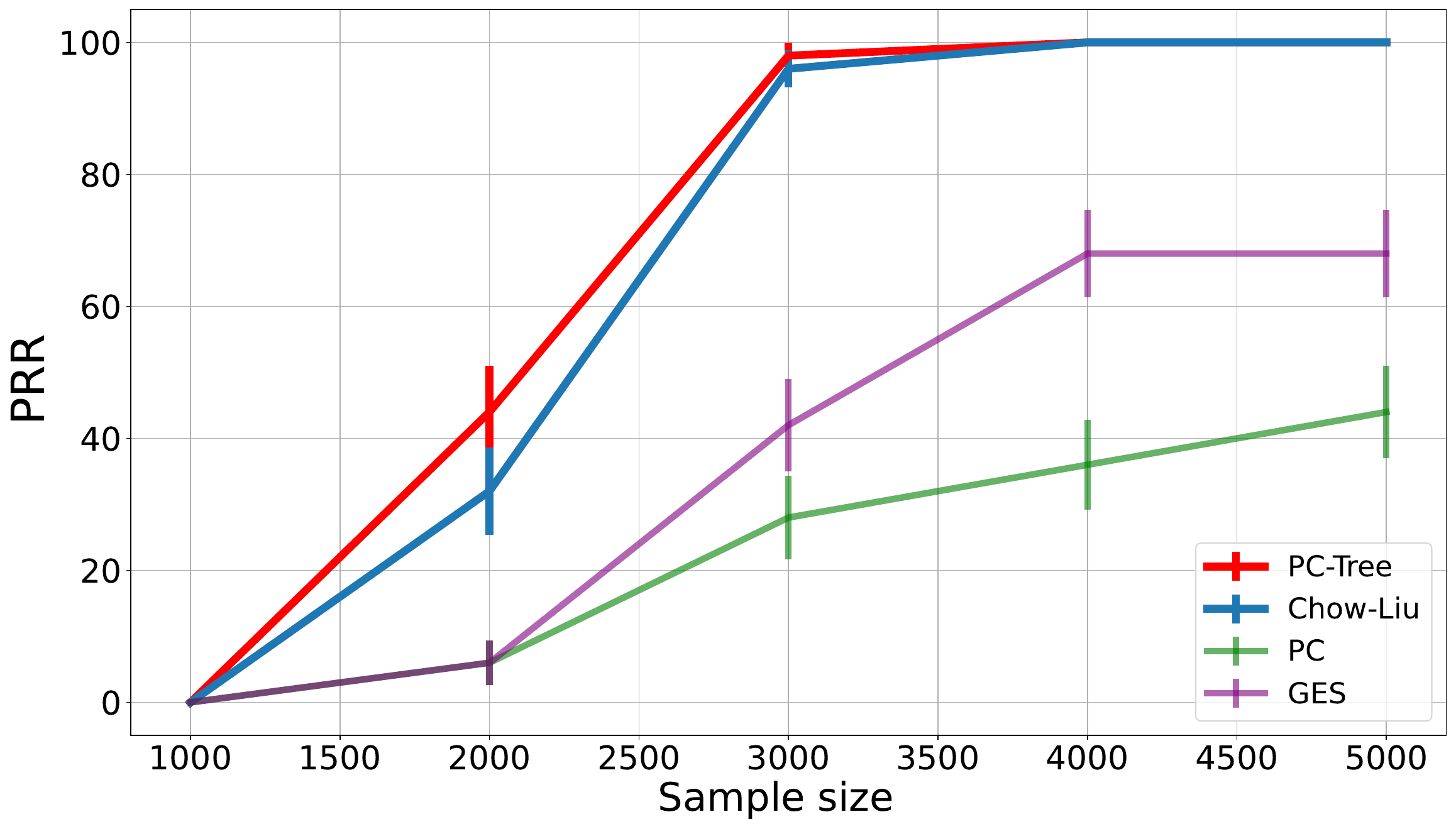}}}
\caption{SHD and PRR for Uniform $\eta$ and $d=100$.}
\label{fig:uniform_100}
\end{figure*}

\begin{figure*}
\centering 
\subfigure[SHD]{\label{fig:laplace_10_shd}\includegraphics[width=82mm]{\detokenize{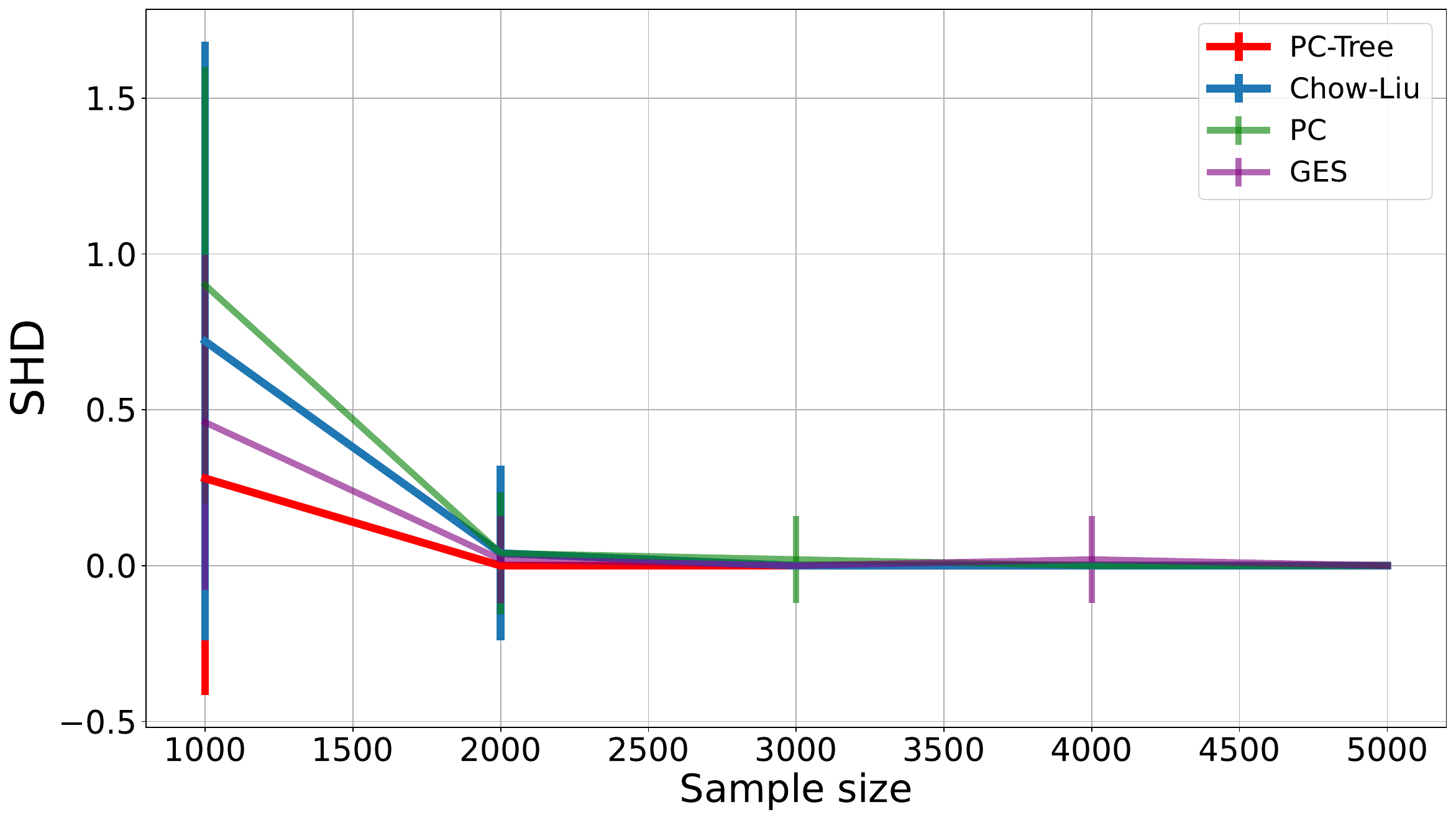}}}
\subfigure[PRR]{\label{fig:laplace_10_prr}\includegraphics[width=82mm]{\detokenize{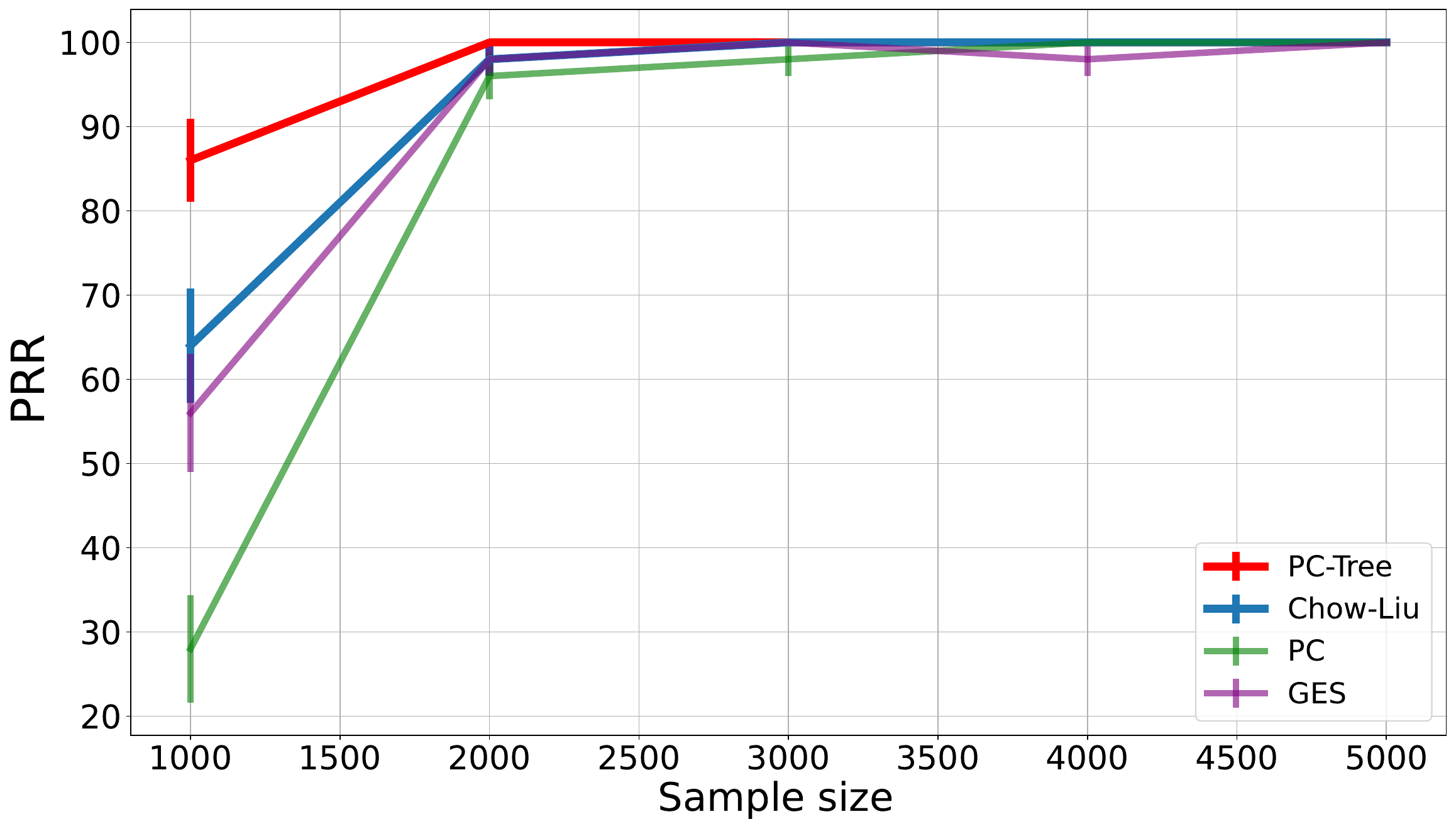}}}
\caption{SHD and PRR for Laplace $\eta$ and $d=10$.}
\label{fig:laplace_10}
\end{figure*}

\begin{figure*}
\centering 
\subfigure[SHD]{\label{fig:laplace_50_shd}\includegraphics[width=82mm]{\detokenize{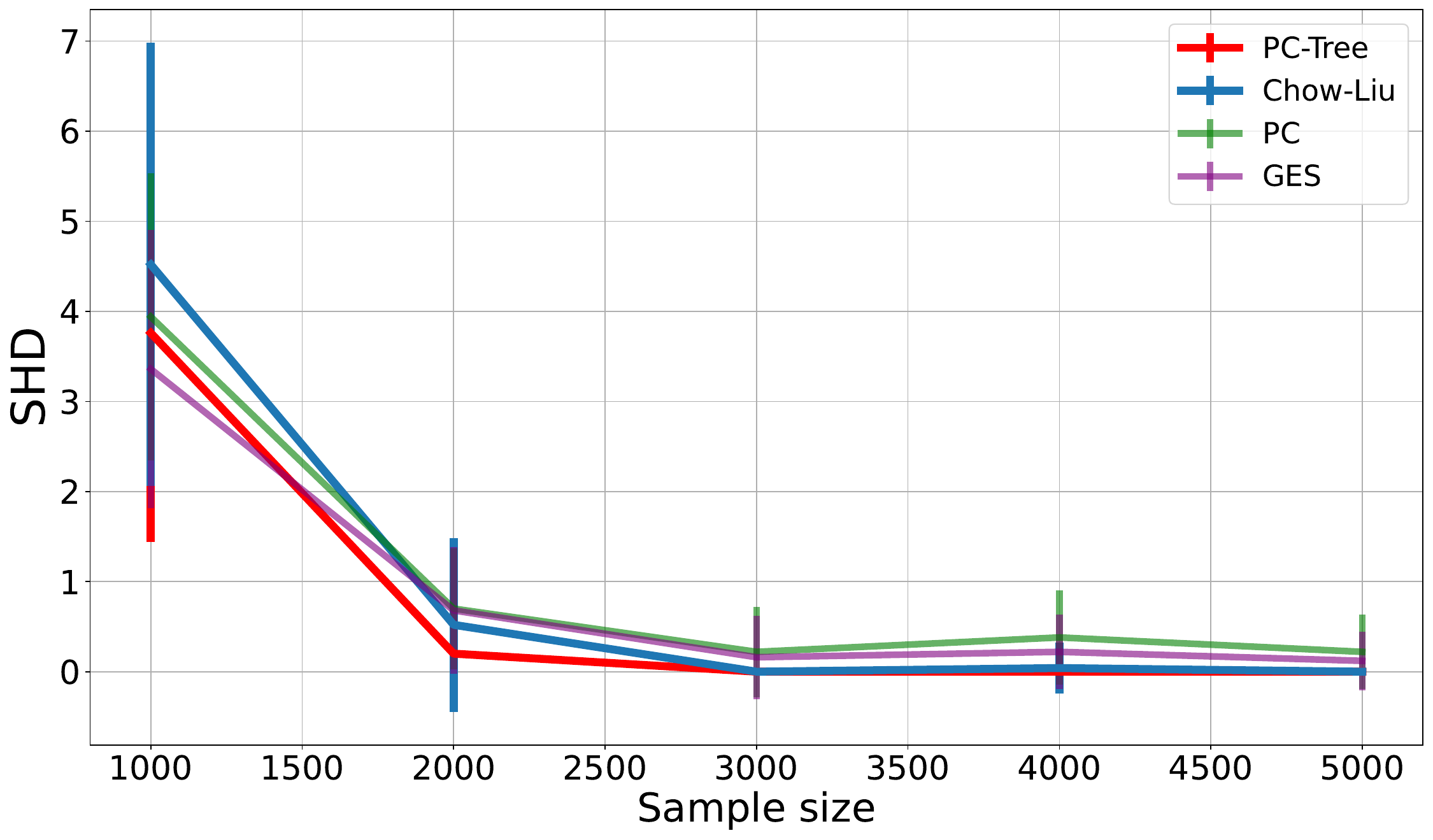}}}
\subfigure[PRR]{\label{fig:laplace_50_prr}\includegraphics[width=82mm]{\detokenize{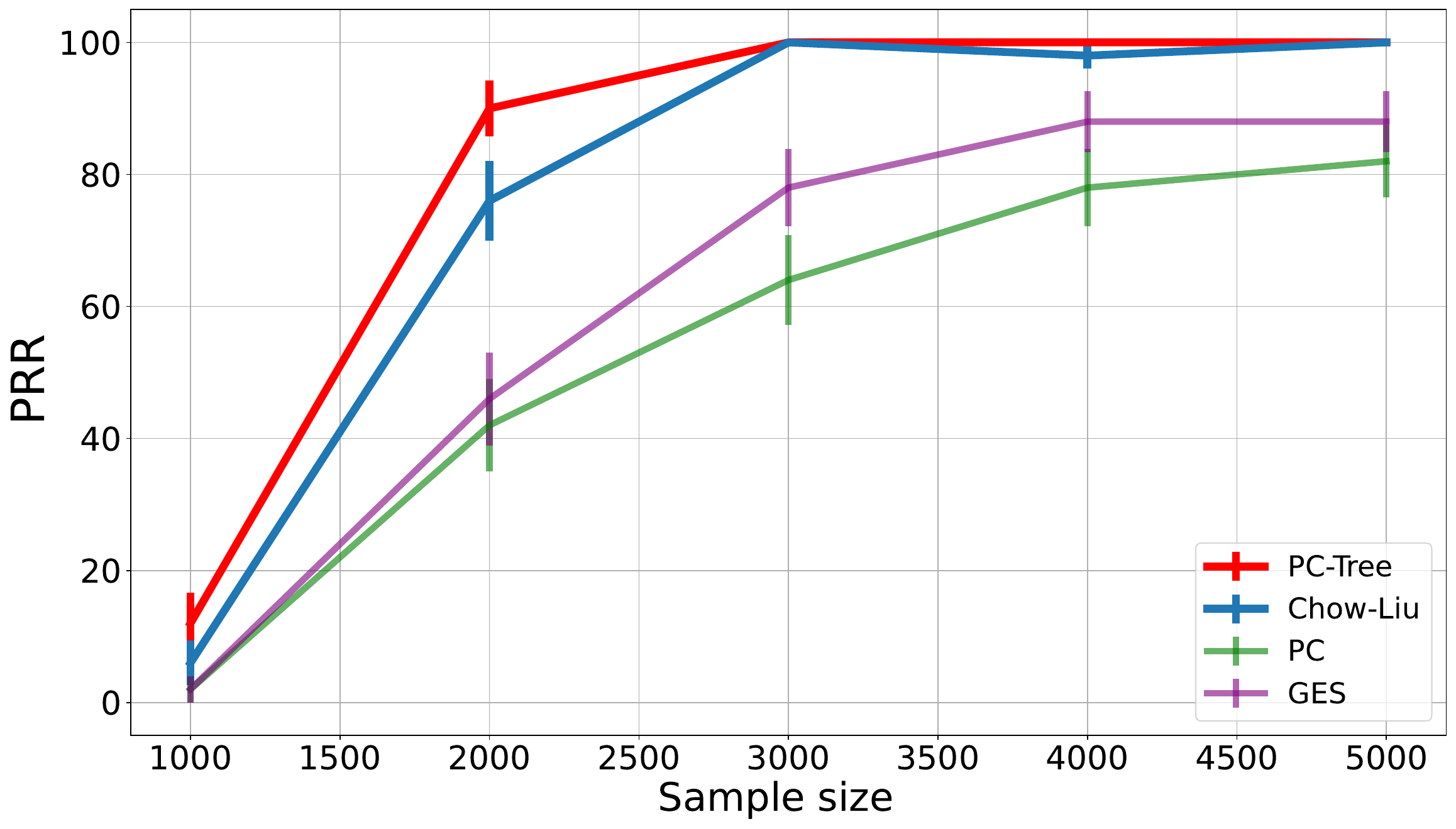}}}
\caption{SHD and PRR for Laplace $\eta$ and $d=50$.}
\label{fig:laplace_50}
\end{figure*}

\begin{figure*}
\centering 
\subfigure[SHD]{\label{fig:laplace_100_shd}\includegraphics[width=82mm]{\detokenize{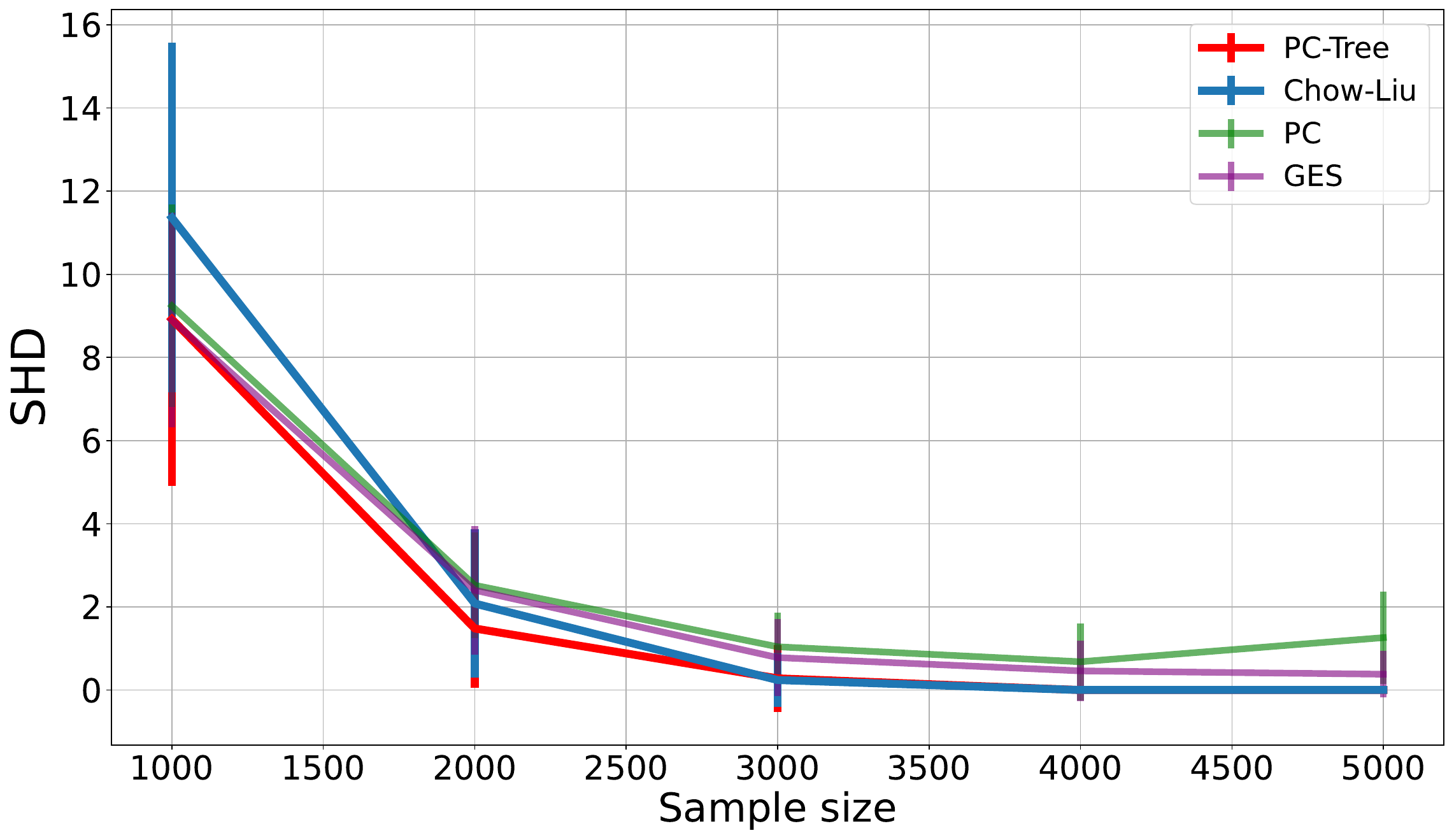}}}
\subfigure[PRR]{\label{fig:laplace_100_prr}\includegraphics[width=82mm]{\detokenize{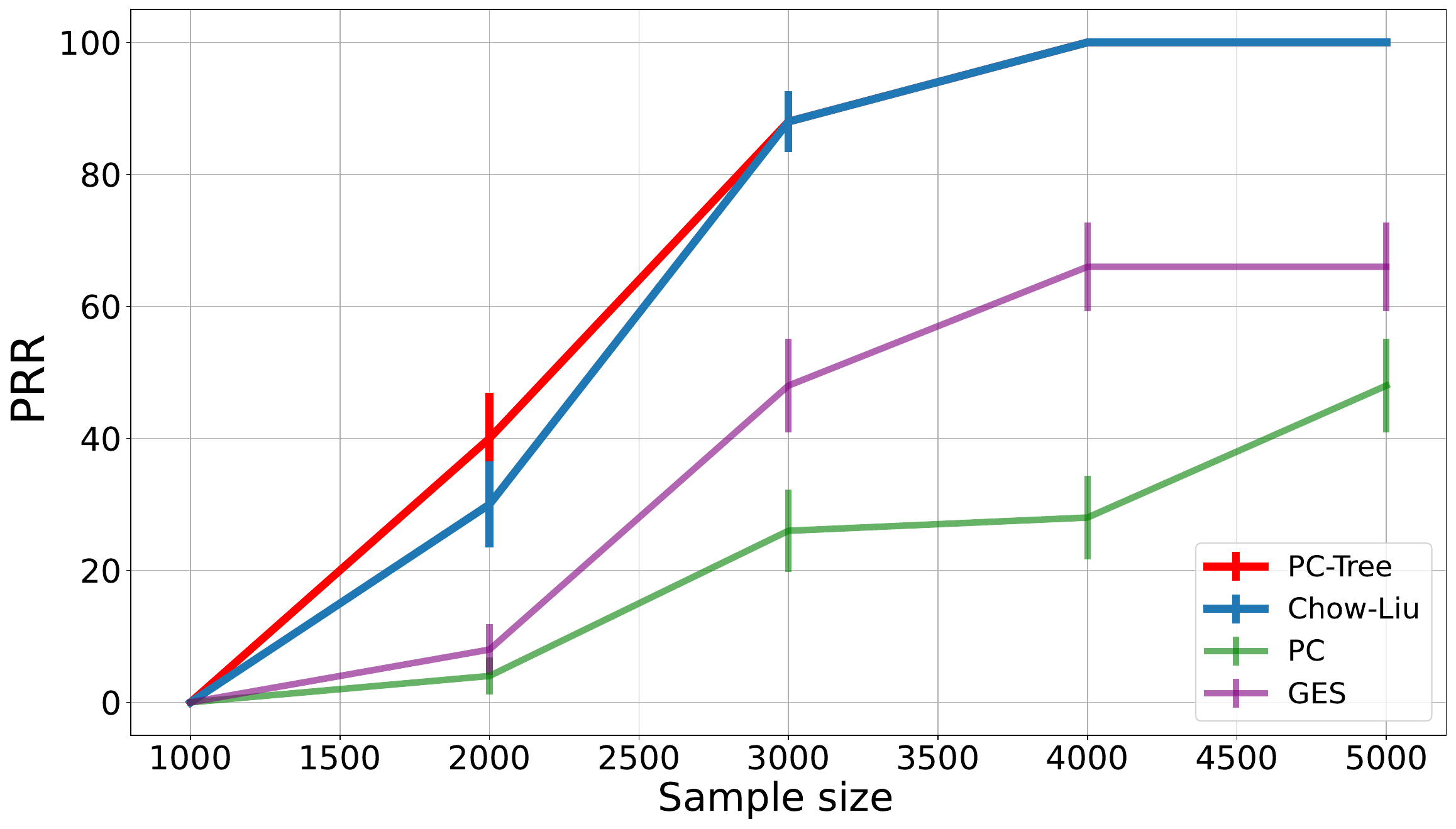}}}
\caption{SHD and PRR for Laplace $\eta$ and $d=100$.}
\label{fig:laplace_100}
\end{figure*}

\paragraph{Agnostic Learning} 
\begin{figure*}[hbt!]
\centering 
\subfigure[SHD comparison (non-i.i.d.)]{\label{fig:gaussian_100_shd_confounder}\includegraphics[width=82mm]{\detokenize{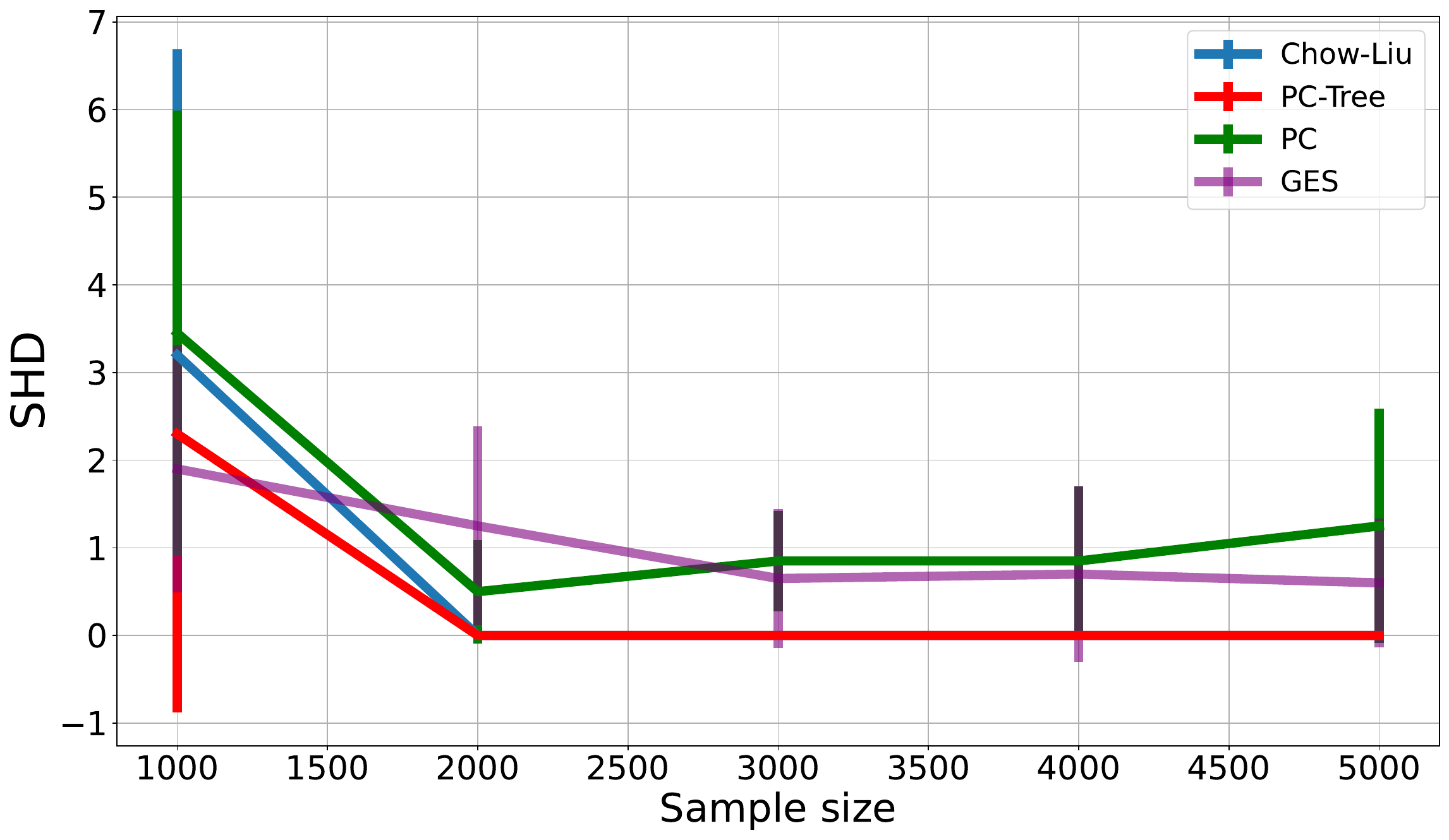}}}
\hspace{-0.1cm}
\subfigure[PRR comparison (non-i.i.d.)]{\label{fig:gaussian_100_prr_confounder}\includegraphics[width=82mm]{\detokenize{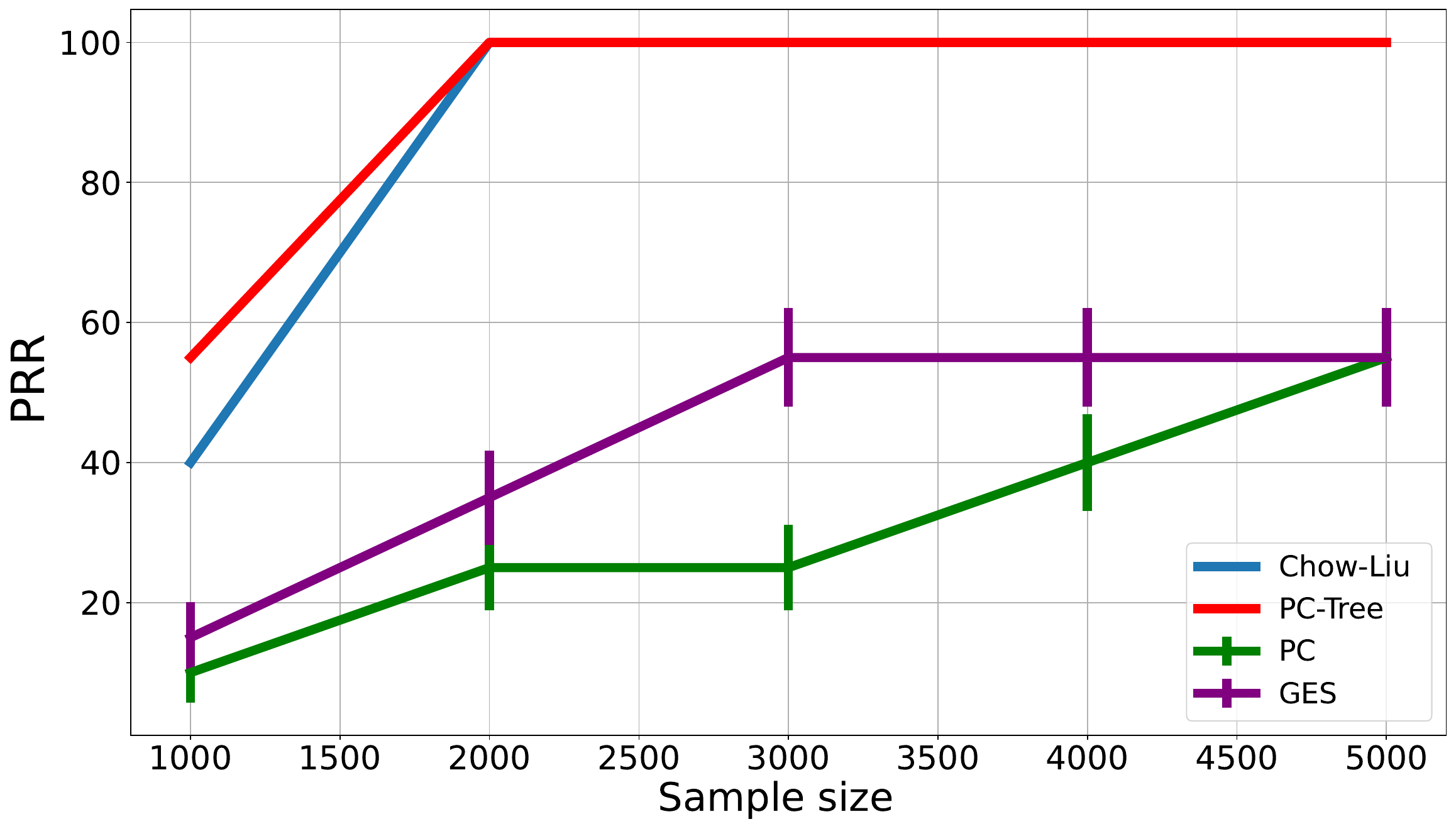}}}
\caption{Performance comparison for PC-Tree, Chow-Liu, PC and GES algorithm evaluated on SHD and PRR in (a) and (b) for non-iid $\beta_k$. The red, blue, green, purple lines are for PC-Tree, Chow-Liu, PC and GES respectively.}
\label{fig:agnoistic_gauss_shd_prr}
\end{figure*}
Additionally, we investigated the algorithm's performance under conditions where the assumption is violated. Specifically, we examined the impact on our algorithm's performance when the coefficients $\beta_k$ in \eqref{eq:sem} are not independently and identically distributed (i.i.d.). To address this question, we conducted agnostic learning experiments and present the corresponding results.

See \cref{fig:agnoistic_gauss_shd_prr} for results with non-iid $\beta_k$. Specifically, $\beta_k = \alpha_k + z$, where we sample $\alpha_k$ iid uniformly and $z$ uniformly, applying the same $z$ to all $\alpha_k$. Here, $z$ introduces dependence among $\beta_k$. When $z=0$, $\beta_k$ is i.i.d., and when $z \neq 0$, $\beta_k$ is non-i.i.d. For brevity, we only report the most relevant setting with $d=100$ nodes and data are Gaussian. 
We simulated random directed trees and synthetic data via equation ~\eqref{eq:sem}. We can see the performance of both PC-tree and Chow-Liu are less affected even when $\beta_k$ are non i.i.d: The Structural Hamming Distance (SHD) becomes 0 in both i.i.d and non i.i.d. setting, and the Precise Recovery Rate (PRR) also outperforms other methods. 
\end{document}